\documentclass{article}
\usepackage{arxiv}

\usepackage{fix-cm}
\usepackage{amsmath}
\usepackage{amsthm}
\usepackage{graphicx}
\usepackage{xcolor}
\usepackage[numbers]{natbib}
\usepackage{amsfonts}
\usepackage{amssymb}
\usepackage{mathrsfs}
\usepackage{algorithm}
\usepackage{algorithmic}
\usepackage[vlined,linesnumbered,ruled,algo2e]{algorithm2e}
\usepackage{mathtools}
\usepackage{booktabs}

\let\oldnl\nl
\newcommand{\nonl}{\renewcommand{\nl}{\let\nl\oldnl}}

\usepackage{graphicx, subfigure}
\usepackage{tikz}

\usetikzlibrary{decorations.pathreplacing,calc,arrows.meta}
\newcommand{\tikzmark}[1]{\tikz[overlay,remember picture] \node (#1) {};}

\newcommand*{\AddNote}[4]{%
    \begin{tikzpicture}[overlay, remember picture]
        \draw [decoration={brace,amplitude=0.5em},decorate,ultra thick,red]
            ($(#3)!(#1.north)!($(#3)-(0,1)$)$) --  
            ($(#3)!(#2.south)!($(#3)-(0,1)$)$)
                node [align=center, text width=2.5cm, pos=0.5, anchor=west] {#4};
    \end{tikzpicture}
}

\newcommand{\norm}[1]{\left\Vert#1\right\Vert}
\newcommand{\abs}[1]{\left\vert#1\right\vert}

\newcommand{\EE}[1]{{\mathbb E}\left[#1\right]}      
\newcommand{\TTh}{\widetilde\theta}
\newcommand{\one}[1]{\mathbf 1 \left\{#1\right\}}           

\DeclareMathOperator*{\argmax}{arg\,max}

\newtheorem{thm}{Theorem}[section]
\newtheorem{cor}[thm]{Corollary}
\newtheorem{lemma}[thm]{Lemma}
\newtheorem{ass}[thm]{Assumption}
\newtheorem{prop}[thm]{Property}

\title{Reinforcement Learning for Strategic Recommendations}

\author{Georgios Theocharous \\
              Adobe Research \\
              \texttt{theochar@adobe.com}\\
           \And
           Yash Chandak \\
           University of Massachusetts Amherst \\
            \texttt{ychandak@cs.umass.edu}\\
            \And
            Philip S. Thomas \\
           University of Massachusetts Amherst \\
            \texttt{pthomas@cs.umass.edu}\\
            \And 
            Frits de Nijs \\
            Monash University \\
            \texttt{frits.nijs@monash.edu}\\
}

\begin{document}
\maketitle

\begin{abstract}
Strategic recommendations (SR) refer to the problem where an intelligent agent observes the sequential behaviors and activities of users and decides when and how to interact with them to optimize some long-term objectives, both for the user and the business. These systems are in their infancy in the industry and in need of practical solutions to some  fundamental research challenges.  At Adobe research, we have been implementing such systems for various use-cases, including points of interest recommendations, tutorial recommendations, next step guidance in multi-media editing software, and ad recommendation for optimizing lifetime value.  There are many research challenges when building these systems, such as modeling the sequential behavior of users, deciding when to intervene and offer recommendations without annoying the user, evaluating policies offline with high confidence, safe deployment, non-stationarity, building systems from passive data that do not contain past recommendations, resource constraint optimization in multi-user systems, scaling to large and dynamic actions spaces,  and handling and incorporating human cognitive biases. In this paper we cover various use-cases and research challenges we solved to make these systems practical.
\end{abstract}


\section{Introduction}
\label{sec:intro}

In strategic recommendation (SR) systems, the goal is to learn a strategy that sequentially selects  recommendations with the highest long-term acceptance by each visiting user of a retail website, a business, or a user  interactive  system  in general.  These systems are in their infancy in the industry and in need of practical solutions to some fundamental research challenges. At Adobe research, we have been implementing such SR systems for various use-cases, including points of interest recommendations, tutorial recommendations, next step guidance in multi-media editing software, and ad recommendation for optimizing lifetime value. 

Most recommendation systems today use supervised learning or contextual bandit algorithms. These algorithms assume that the visits are i.i.d.~and do not discriminate between visit and user, i.e.,~each visit is considered as a new user that has been sampled i.i.d.~from the population of the business's users. As a result, these algorithms are myopic and do not try to optimize the long-term effect of the recommendations on the users. Click through rate (CTR) is a suitable metric to evaluate the performance of such greedy algorithms. Despite their success, these methods are becoming insufficient as users incline to establish longer and longer-term relationship with their websites (by going back to them). This increase in {\em returning users} further violates the main assumption underlying supervised learning and bandit algorithms, i.e.,~there is no difference between a visit and a user. This is the main motivation for SR systems that we examine in this paper. 

Reinforcement learning (RL) algorithms that aim to optimize the long-term performance of the system (often formulated as the expected sum of rewards/costs) seem to be suitable candidates for SR systems. The nature of these algorithms allows them to take into account all the available knowledge about the user in order to select an offer or recommendation that maximizes the total number of times she will click or accept the recommendation over multiple visits, also known as the user's life-time value (LTV). Unlike myopic approaches, RL algorithms differentiate between a visit and a user, and consider all the visits of a user (in chronological order) as a system trajectory. Thus, they model the user, and not their visits, as i.i.d.~samples from the population of the users of the website. This means that although we may evaluate the performance of the RL algorithms using CTR, this is not the metric that they optimize, and thus, it would be more appropriate to evaluate them based on the expected total number of clicks per user (over the user's trajectory), a metric we call LTV. This long-term approach to SR systems allows us to make decisions that are better than the short-sighted decisions made by the greedy algorithms.  Such decisions might  propose an offer that is considered as a loss to the business in the short term, but increases the user loyalty and engagement in the long term. 

Using RL for LTV optimization is still in its infancy. Related work has experimented with toy examples and has appeared mostly in marketing venues \cite{Pfeifer2000,Jonker2004,giuliano@marketing07}. An approach directly related to ours first appeared in \cite{Pednault:2002:SCD:775047.775086}, where the authors used public data of an email charity campaign, batch RL algorithms, and heuristic simulators for evaluation, and showed that RL policies produce better results than myopic’s. Another one is \cite{Silver2013}, which proposed an on-line RL system that learns concurrently from multiple customers. The system was trained and tested on a simulator.  A recent approach uses RL to optimizes LTV for slate recommendations \cite{DBLP:journals/corr/abs-1905-12767}. It addresses the problem of how to decompose the LTV of a slate into a tractable function of its component item-wise LTVs. 
Unlike most of previous previous work, we address many more  challenges that are found when dealing real data. These challenges, which  hinder the widespread application of RL technology to SR systems include:
\begin{itemize}
\item{\bf High confidence off-policy evaluation} refers to the problem of evaluating the performance of an SR system with high confidence before costly A/B testing or deployment.
\item{\bf Safe deployment} refers to the problem of deploying a policy without creating disruption from the previous running policy.  For example, we should never deploy a policy that will have a worse LTV than the previous.
\item{\bf Non-stationarity} refers to the fact that the real world is non-stationary.  In RL and Markov decision processes there is usually the assumption that the transition dynamics and reward are stationary over time.  This is often violated in the marketing world where trends and seasonality are always at play.
\item{\bf Learning from passive data} refers to the fact that there is usually an abundance of sequential data or events that have been collected without a recommendation system in place.  For example,  websites record the sequence of products and pages a user views.  Usually in RL, data is in the form of sequences of states, actions and rewards.  The question is how can we leverage passive data that do not contain actions to create a recommendation systems that recommends the next page or product.
\item{\bf Recommendation acceptance factors} refers to the problem of deeper understanding of recommendation acceptance than simply predicting clicks.  For example, a person might have low propensity to listen due to various reasons of inattentive disposition.  A classic problem is the `recommendation fatigue', where people may quickly stop paying attention to recommendations such as ads and promotional emails, if they are presented too often.
\item{\bf Resource constraints in multi-user systems} refers to the problems of constraints created in multi-user recommendation systems.  For example, if multiple users in a theme park are offered the same strategy for touring the park, it could overcrowd various attractions. Or, if a store offers the same deal to all users, it might deplete a specific resource.
\item{\bf Large action spaces} refers to the problem of having too many recommendations.  Netflix for example employs thousands of movie recommendations. This is particularly challenging for SR systems that make a sequence of decisions, since the search space grows exponentially with the planning horizon (the number of decisions made in a sequence).
\item{\bf Dynamic actions} refers to the problem where the set of recommendations may be changing over time.  This is a classic problem in marketing where the offers made at some retail shop could very well be slowly changing over time.  Another example is movie recommendation, in businesses such as Netflix, where the catalogue of movies evolves over time. 
\end{itemize}

In this paper we address all of the above research challenges.  We summarize in chronological order our work in making SR systems practical for the real-world. In Section~\ref{sec:hope} we present a  method for evaluating SR systems off-line with high confidence.  In Section~\ref{sec:par} we present a practical reinforcement learning (RL) algorithm for implementing an SR system with an application to ad offers. In Section~\ref{sec:safety} we present an algorithm for safely deploying an SR system.  In Section~\ref{sec:nonstationary} we tackle the problem of non-stationarity.  Technologies in sections \ref{sec:hope}, \ref{sec:par}, \ref{sec:safety} and \ref{sec:nonstationary} are build chronologically, where high confidence off-policy evaluation is leveraged across all of them. 

In Section~\ref{sec:passive} we address the problem of bootstrapping an SR system from passive sequential data that do not contain past recommendations. In Section~\ref{sec:acceptance} we examine recommendation acceptance factors, such as the `propensity to listen' and `recommendation fatigue'.  In Section~\ref{sec:constraints} we describe a solution that can optimize for resource constraints in multi-user SR systems. Sections \ref{sec:passive},  \ref{sec:acceptance} and \ref{sec:constraints} are build chronologically, where the bootstrapping  from passive data  is used across.

In Section~\ref{sec:large-actions} we describe a solution to the large action space in SR systems. In Section~\ref{sec:dynaimic-actions} we describe a solution to the dynamic action problem, where the available actions can vary over time.  Sections \ref{sec:large-actions} and \ref{sec:dynaimic-actions} are build chronologically, where they use same action embedding technology.

Finally, in Section~\ref{sec:bias-aware} we argue that the next generation of recommendation systems needs to incorporate human cognitive biases.

\section{Preliminaries}
In this section, we present the general set of notations, which will be useful throughout the paper. In cases where problem specific notations are required, we introduce them in the respective section.

We model SR systems as \textit{Markov decision process} (MDPs) \citep{Sutton1998}.
    	An MDP is represented by a tuple, $\mathcal{M} = (\mathcal{S},\mathcal{A},\mathcal{P},\mathcal{R}, \gamma, d_0)$.
    	$\mathcal{S}$ is the set of all possible states, called the state set, and $\mathcal{A}$ is a finite  set of actions, called the action set. 
    	The random variables, $S_t \in \mathcal S$, $A_t \in \mathcal A$, and $R_t$ denote the state, action, and reward at time $t$.
    	%
        %
    	The first state comes from an initial distribution, $d_0$.
    	%
    	%
    	%
    	%
    	The reward discounting parameter is given by $\gamma \in [0,1)$. 
    	$\mathcal{P}$ is the state transition function.
    	We denote by $s_t$ the feature vector describing a user's $t^\text{th}$ visit with the system and  by $a_t$ the $t^\text{th}$ recommendation shown to the user, and refer to them as a \textit{state} and an \textit{action}.  The rewards are assumed to be non-negative. The reward  $r_t$ is $1$ if the user accepts the recommendation $a_t$ and $0$, otherwise. We assume that the users interact at most $T$ times. We write $\tau\coloneqq \{s_1, a_1, r_1, s_2, a_2, r_2, \hdots, s_{T}, a_{T}, r_{T}\}$ to denote the history of visits with one user, and we call $\tau$ a \textit{trajectory}. The \textit{return} of a trajectory is the discounted sum of rewards, $R(\tau)\coloneqq\sum_{t=1}^T\gamma^{t-1}r_t$, 

A \textit{policy} $\pi$ is used to determine the probability of showing each recommendation. Let $\pi(a \vert s)$ denote the probability of taking action $a$ in state $s$, regardless of the time step $t$. The goal is to find a policy that maximizes the expected total number of recommendation acceptances per user: $\rho(\pi) \coloneqq \mathbb E[R(\tau) \vert \pi].$ Our historical data is a set of trajectories, one per user. Formally, $\mathcal D$ is the historical data containing $n$ trajectories $\{\tau_i\}_{i=1}^n$, each labeled with the \textit{behavior policy} $\pi_i$ that produced it. 

\section{High Confidence Off-policy Evaluation}

\label{sec:hope}
One of the first challenges in building SR systems is evaluating their performance before costly A/B testing and deployment.  Unlike classical machine learning systems, an SR system is more complicated to evaluate because recommendations can affect how a user responds to all future recommendations.  In this section we summarize a  \textit{high confidence off-policy evaluation} (HCOPE) method \cite{DBLP:conf/aaai/ThomasTG15}, which can inform the business manager of the performance of the SR system with some guarantee, before the system is deployed.
We denote the policy to be evaluated as the  \textit{evaluation policy} $\pi_e$.

HCOPE is a family of methods that use the historical data $\mathcal D$ in order to compute a $1-\delta$-confidence lower bound on the expected performance of the evaluation policy $\pi_e$~\cite{DBLP:conf/aaai/ThomasTG15}.
In this section, we explain three different approaches to HCOPE.
All these approaches are based on importance sampling. The {\em importance sampling estimator}
\begin{equation}
\label{eq:ISE}
\hat \rho(\pi_e \vert \tau_i, \pi_i) \coloneqq \underbrace{R(\tau_i)}_{\text{return}} \underbrace{\prod_{t=1}^T \frac{\pi_e(a_t^{\tau_i} \vert s_t^{\tau_i})}{\pi_i(a_t^{\tau_i} \vert s_t^{\tau_i})}}_{\text{importance weight}},
\end{equation}
is an unbiased estimator of $\rho(\pi)$ if $\tau_i$ is generated using policy $\pi_i$ \cite{Precup2000}, the support of $\pi_e$ is a subset of the support or $\pi_i$, and where $a_t^{\tau_i}$ and $s_t^{\tau_i}$ denote the state and action in trajectory $\tau_i$ respectively. 
Although the importance sampling estimator is conceptually easier to understand, in most of our applications we use the {\em per-step importance sampling estimator}
\begin{equation}
\label{eq:PEISE}
\hat \rho(\pi_e \vert \tau_i, \pi_i) \coloneqq \sum_{t=1}^T\gamma^{t-1}r_t\left(\prod_{j=1}^t\frac{\pi_e(a_j^{\tau_i} \vert s_j^{\tau_i})}{\pi_i(a_j^{\tau_i} \vert s_j^{\tau_i})}\right),
\end{equation}
where the term in the parenthesis is the importance weight for the reward generated at time $t$. This estimator has a lower variance than~\eqref{eq:ISE}, and remains unbiased. 

For brevity, we describe the approaches to HCOPE in terms of a set of non-negative independent random variables, $\mathbf X=\{X_i\}_{i=1}^n$ (note that the importance weighted returns are non-negative because the rewards are never negative, since in our applications the reward is  $1$ when the user accepts a recommendation and $0$ otherwise). For our applications, we will use $X_i =  \hat \rho(\pi_e \vert \tau_i, \pi_i)$, where $\hat \rho(\pi_e \vert \tau_i, \pi_i)$ is computed either by~\eqref{eq:ISE} or~\eqref{eq:PEISE}. The three approaches that we will use are:

\noindent{\bf 1. Concentration Inequality:} Here we use the concentration inequality (CI) in~\cite{DBLP:conf/aaai/ThomasTG15} and call it the {\em CI approach}. We write $\rho_-^\text{CI}(\mathbf X, \delta)$ to denote the $1-\delta$ confidence lower-bound produced by their method. The benefit of this method is that it provides a true high-confidence lower-bound, i.e.,~it makes no false assumption or approximation, and so we refer to it as $\textit{safe}$.
However, as it makes no assumptions, bounds obtained using CI happen to be overly conservative, as shown in Figure \ref{fig:boundComparison}.

\noindent{\bf 2. Student's $t$-test:} One way to tighten the lower-bound produced by the CI approach is to introduce a false but reasonable assumption. Specifically, we leverage the central limit theorem, which says that $\hat X \coloneqq \frac{1}{n} \sum_{i=1}^n X_i$ is approximately normally distributed if $n$ is large. Under the assumption that $\hat X$ is normally distributed, we may apply the one-tailed Student's $t$-test to produce $\rho_-^\text{TT}(\mathbf X, \delta)$, a $1-\delta$ confidence lower-bound on $\mathbb{E}[\hat X]$, which in our application is a $1-\delta$ confidence lower-bound on $\rho(\pi_e)$.  Unlike the other two approaches, this approach, which we call {\em TT}, requires little space to be formally defined, and so we present its formal specification:
\begin{gather}
\hat X \coloneqq \frac{1}{n} \sum_{i=1}^n X_i,\quad\quad \hat \sigma \coloneqq \sqrt{\frac{1}{n-1}\sum_{i=1}^n \left (X_i - \hat X \right )^2},\\
\rho_-^\text{TT}(\mathbf X, \delta) \coloneqq \hat X - \frac{\hat \sigma}{\sqrt{n}}t_{1-\delta, n-1},
\end{gather}
where  $t_{1-\delta, \nu}$ denotes the inverse of the cumulative distribution function of the Student's $t$ distribution with $\nu$ degrees of freedom, evaluated at probability $1-\delta$ (i.e.,~function $\text{tinv}(1-\delta,\nu)$ in {\sc Matlab}).

Because $\hat \rho_-^\text{TT}$ is based on a false (albeit reasonable) assumption, we refer to it as \textit{semi-safe}. Although the {\em TT approach} produces tighter lower-bounds than the CI's, it still tends to be overly conservative for our application, as shown in Figure \ref{fig:boundComparison}.
More discussion can be found in the work by \cite{DBLP:conf/aaai/ThomasTG15}.

\noindent{\bf 3. Bias Corrected and Accelerated Bootstrap:} One way to correct for the overly-conservative nature of TT is to use bootstrapping to estimate the true distribution of $\hat X$, and to then assume that this estimate is the true distribution of $\hat X$. The most popular such approach is \textit{Bias Corrected and accelerated} (BCa) bootstrap \cite{Efron1987}. We write $\rho_-^\text{BCa}(\mathbf X, \delta)$ to denote the lower-bound produced by BCa, whose pseudocode can be found in ~\cite{DBLP:conf/icml/ThomasTG15}.
While the bounds produced by BCa are reliable, like t-test it may have error rates larger than $\delta$ and are thus \textit{semi-safe}. 
An illustrative example is provided in Figure \ref{fig:boundComparison}.

\begin{figure}
  \centering
  \includegraphics[width=0.5\columnwidth]{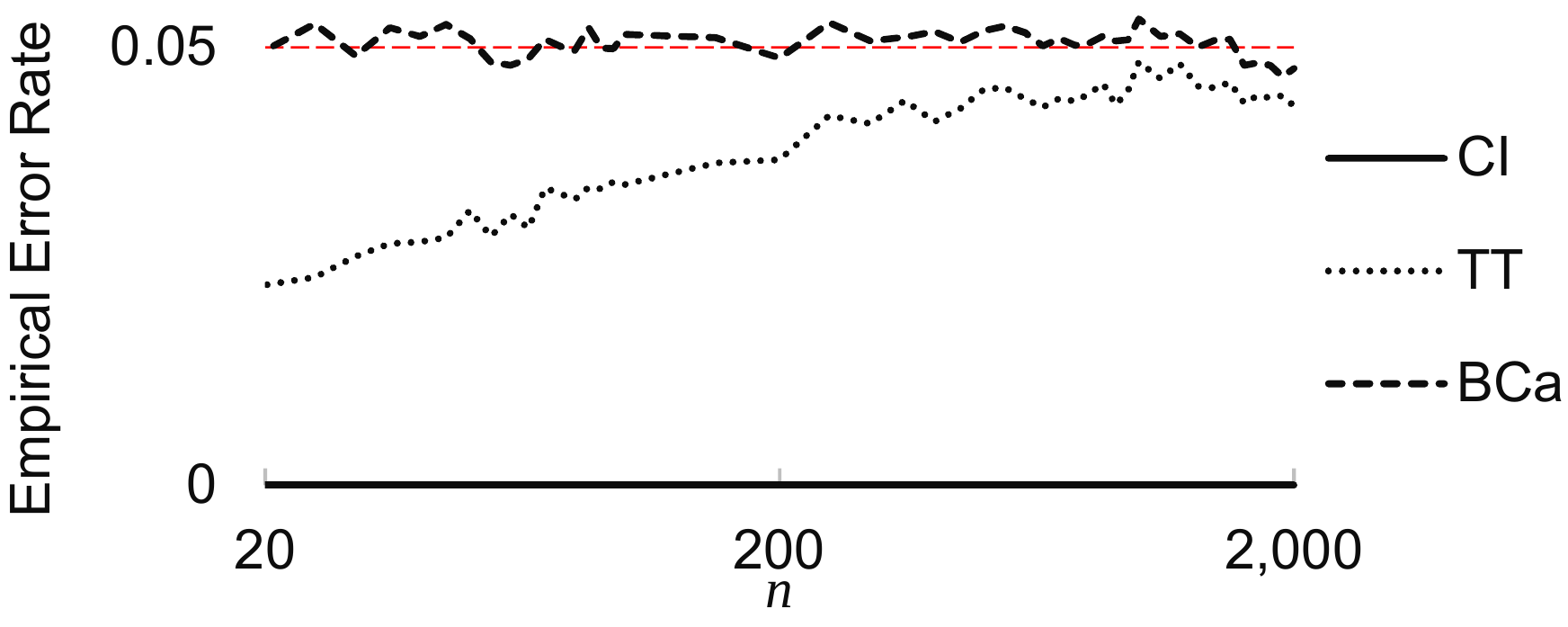}
	\caption{Empirical error rates when estimating a $95\%$ confidence lower-bound on the mean of a gamma distribution (shape parameter $k=2$ and scale parameter $\theta=50$) using $\rho_-^\dagger$, where the legend specifies the value of $\dagger$. This gamma distribution has a heavy upper-tail similar to that of importance weighted returns. The logarithmically scaled horizontal axis is the number of samples used to compute the lower bound (from $20$ to $2000$) and the vertical axis is the mean empirical error rate over $100,\!000$ trials. Note that CI is overly conservative, with zero error in all the trials (it is on the $x$-axis). The $t$-test is initially conservative, but approaches the allowed $5\%$ error rate as the number of samples increases. BCa remains around the correct $5\%$ error rate regardless of the number of samples.
}
	\label{fig:boundComparison}
\end{figure}

For SR systems, where ensuring quality of a system before deployment is critical, these three approaches provide several viable approaches to obtaining performance guarantees using only historical data.
%


\section{Personalized Ad Recommendation Systems for Life-Time Value Optimization with Guarantees}
\label{sec:par}
The next question is how to compute a good SR policy.  In this section we demonstrate how to compute an SR policy for personalized ad recommendation (PAR) systems using reinforcement learning (RL).
RL algorithms take into account the long-term effect of actions, and thus, are more suitable than myopic techniques, such as contextual bandits, for modern PAR systems in which the number of returning visitors is rapidly growing. However, while myopic techniques have been well-studied in PAR systems, the RL approach is still in its infancy, mainly due to two fundamental challenges: how to compute a good RL strategy and how to evaluate a solution using historical data to ensure its `safety' before deployment. In this section, we use  the family  of off-policy evaluation techniques with statistical guarantees presented in Section~\ref{sec:hope} to tackle both of these challenges. We apply these methods to a real PAR problem, both for evaluating the final performance and for optimizing the parameters of the RL algorithm. Our results show that an RL algorithm equipped with these off-policy evaluation techniques outperforms the myopic approaches. Our results give fundamental insights on the difference between the click through rate (CTR) and life-time value (LTV) metrics for evaluating the performance of a PAR algorithm \cite{TheocharousTG15}.

\subsection{CTR versus LTV}

Any personalized ad recommendation (PAR) policy could be evaluated for its greedy/myopic or long-term performance. For greedy performance, click through rate (CTR) is a reasonable metric, while life-time value (LTV) seems to be the right choice for long-term performance. These two metrics are formally defined as 
\begin{equation}
\text{CTR }=\frac{\text{Total $\#$ of Clicks}}{\text{Total $\#$ of {\bf Visits}}} \times 100,
\label{eq:ctr}
\end{equation}
\begin{equation}
\text{LTV}= \frac{\text{Total $\#$ of Clicks}}{\text{Total $\#$ of {\bf Visitors}}} \times 100.
\label{eq:ltv}
\end{equation}

CTR is a well-established metric in digital advertising and can be estimated from historical data (off-policy) in unbiased (inverse propensity scoring;~\cite{lihong@www10}) and biased (see e.g.,~\cite{StrehlLLK10}) ways.  The reason that we use LTV is that CTR is not a good metric for evaluating long-term performance and could lead to misleading conclusions. Imagine a greedy advertising strategy at a website that directly displays an ad related to the final product that a user could buy. For example, it could be the BMW website and an ad that offers a discount to the user if she buys a car. users who are presented such an offer would either take it right away or move away from the website. Now imagine another marketing strategy that aims to transition the user down a sales funnel before presenting her the discount. For example, at the BMW website one could be first presented with an attractive finance offer and a great service department deal before the final discount being presented. Such a long-term strategy would incur more visits with the user and would eventually  produce more clicks per user and more purchases. The crucial insight here is that the policy can change the number of times that a user will be shown an advertisement---the length of a trajectory depends on the actions that are chosen. A visualization of this concept is presented in Figure~\ref{fig:ltv-ctr}.

\begin{figure}[h]
\centering
\includegraphics[height=2.1in]{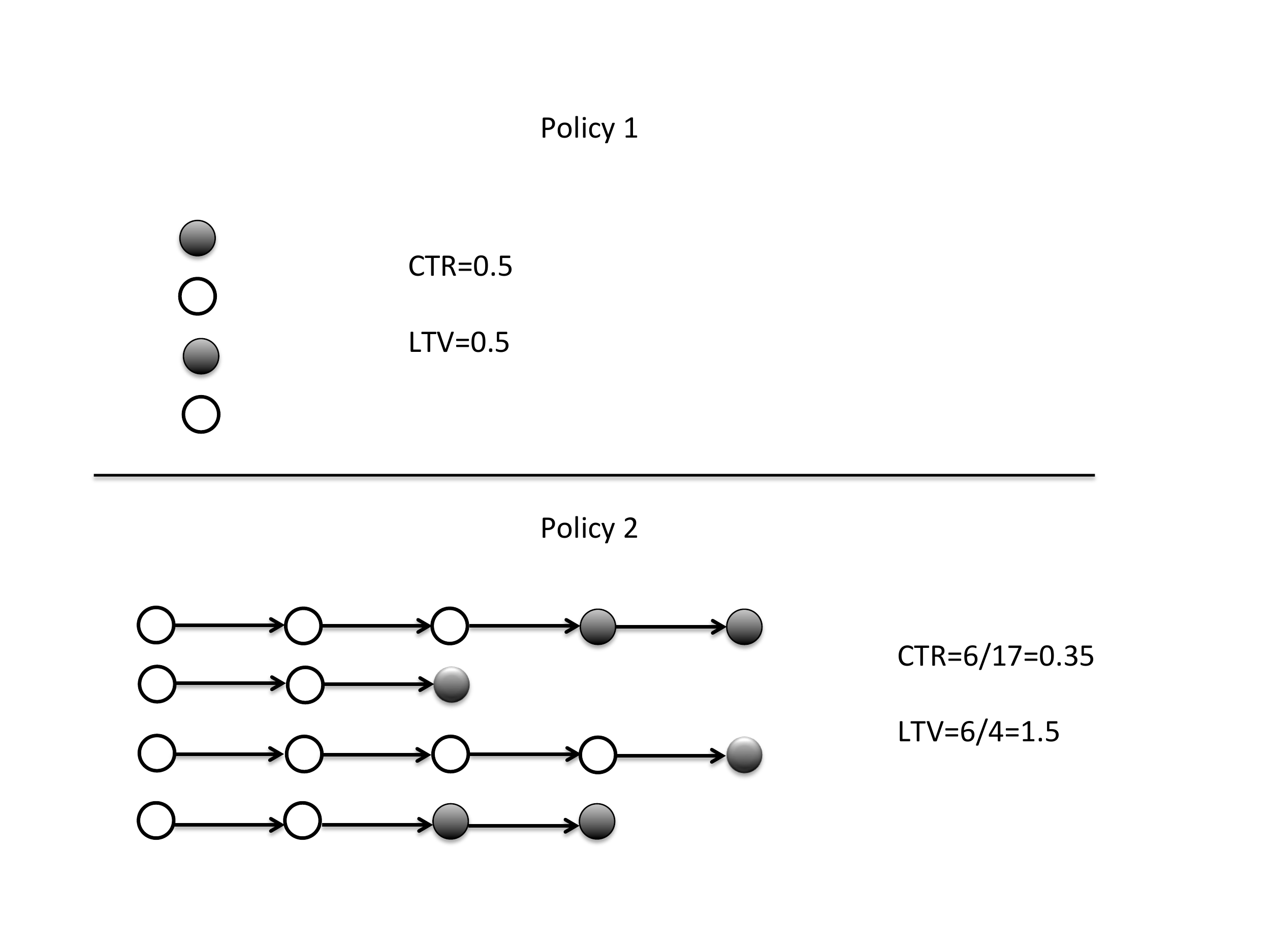}
\caption{The circles indicate user visits. The black circles indicate clicks. Policy~1~is greedy and users do not return. Policy~2~optimizes for the long-run, users come back multiple times, and click towards the end. Even though Policy~2~has a lower CTR than Policy~1, it results in more revenue, as captured by the higher LTV. Hence, LTV is potentially a better metric than CTR for evaluating ad recommendation policies.}
\label{fig:ltv-ctr}
\end{figure}

\subsection{Ad Recommendation Algorithms}

For greedy optimization, we used a random forest (RF) algorithm~\cite{Statistics01randomforests} to learn a mapping from features to actions. RF is a state-of-the-art ensemble learning method for regression and classification, which is relatively robust to overfitting and is often used in industry for big data problems. The system is trained using a RF for each of the offers/actions to predict the immediate reward. During execution, we use an $\epsilon$-greedy strategy, where we choose the offer whose RF has the highest predicted value with probability $1-\epsilon$, and the rest of the offers, each with probability $\epsilon/(|A|-1)$ 




For LTV optimization, we used the Fitted Q Iteration (FQI)~\cite{Ernst05tree-basedbatch} algorithm, with RF function approximator, which allows us to handle high-dimensional continuous and discrete variables. When an arbitrary function approximator is used in the FQI algorithm, it does not converge monotonically, but rather oscillates during training iterations. To alleviate the oscillation problem of FQI and for better feature selection, we used our high confidence off-policy evaluation (HCOPE) framework within the training loop.  The loop keeps track of the best FQI result according to a validation data set (see Algorithm~\ref{alg:ltv}). 

For both algorithms we start with three data sets an $X_\text{train}$, $X_\text{val}$ and  $X_\text{test}$. Each one is made of complete user trajectories.   A user only appears in one of those files.  The $X_\text{val}$ and $X_\text{test}$ contain users that have been served by the random policy.   The greedy approach proceeds by first doing feature selection on the $X_\text{train}$, training a random forest, turning the  policy into $\epsilon$-greedy on the  $X_\text{test}$ and then evaluating that policy using the off-policy evaluation techniques.  The LTV approach  starts from the random forest model of the greedy approach.  It then computes labels as shown is step 6 of the LTV optimization algorithm \ref{alg:ltv}.  It does feature selection, trains a random forest  model, and then turns the policy into $\epsilon$-greedy on the $X_\text{val}$ data set.  The policy is tested using the importance weighted returns according to Equation \ref{eq:PEISE}.  LTV optimization loops over a fixed number of iterations and keeps track of the best performing policy, which is finally evaluated on the $X_\text{test}$.   The final outputs are `risk plots', which are graphs that show the lower-bound of the expected sum of discounted reward of the policy  for different confidence values.

\begin{algorithm}
\begin{algorithmic}[1]
\STATE $\pi_b = \text{randomPolicy}$
\STATE $Q =$ {\sc rf.Greedy}$(\mathbf X_\text{train}, \mathbf X_\text{test}, \delta) $ \COMMENT{start with greedy value function}
\FOR{$i=1$  {\bf to} $K$ }
 	\STATE $r=\mathbf X_\text{train}(\text{reward}) $ \COMMENT{use recurrent visits}
 	\STATE $x=\mathbf X_\text{train}(\text{features})$ 
        \STATE $y= r_{t} + \gamma \max_{a \in A }Q_a(x_{t+1})$ 
	\STATE $\bar{x}= \text{informationGain}(x,y)$ \COMMENT{feature selection}
	\STATE $Q_a = \text{randomForest}(\bar{x}, y)$ \COMMENT{for each action}
	\STATE $\pi_e = \text{epsilonGreedy}( Q, \mathbf X_\text{val} )$ 
	\STATE $W = \hat \rho(\pi_e |\mathbf X_\text{val}, \pi_b)$ \COMMENT{importance weighted returns}	
	\STATE $\text{currBound} = \rho_-^\dagger(W, \delta)$
	\IF{$\text{currBound}  > \text{prevBound}$}	
		\STATE $\text{prevBound}= \text{currBound}$
		\STATE $Q_\text{best} = Q$
	\ENDIF
\ENDFOR
	\STATE $\pi_e = \text{epsilonGreedy}( Q_\text{best}, \mathbf X_\text{test} )$ 
	\STATE $W = \hat \rho(\pi_e |\mathbf X_\text{test}, \pi_b)$
\STATE \Return $ \rho_-^\dagger(W, \delta)$ \COMMENT{lower bound}
\end{algorithmic}
\caption{{\sc LtvOptimization}$(\mathbf X_\text{train}, \mathbf X_\text{val}, \mathbf X_\text{test}, \delta, K, \gamma, \epsilon)$ : compute a LTV  strategy using $\mathbf X_\text{train}$,  and predict the $1-\delta$  lower bound on the test data $\mathbf X_\text{test}$.}
\label{alg:ltv}
\end{algorithm}

\subsection{Experiments}
For our experiments we used 2 data sets from the banking industry. On the bank website when users visit, they are shown one of a finite number of offers.  The reward is 1 when a user clicks on the offer and 0, otherwise. For data set 1, we collected data from a particular campaign of a bank for a month that had 7 offers and approximately $200,\!000$ visits.  About $20,\!000$ of the visits were produced by a random strategy.  For data set 2 we collected data from  a different  bank for a campaign that had 12 offers and  $4,\!000,\!000$ visits, out of which  $250,\!000$ were produced by a random strategy.  When a user visits the bank website for the first time, she is assigned either to a random strategy or a targeting strategy for the rest of the campaign life-time.  We splitted the random strategy  data into a test set and a validation set.  We used the targeting data for training to optimize the greedy and LTV strategies.  We used aggressive feature selection for the greedy strategy and selected $20\%$ of the features.  For LTV, the feature selection had to be even more aggressive due to the fact that the number of recurring visits is approximately $5\%$.  We used information gain for the feature selection module~\cite{Cheng:2012:FSE:2399970.2399981}. With our algorithms we produce performance results both for the CTR and LTV metrics. To produce results for CTR we assumed that each visit is a unique visitor.  We performed various experiments to understand the different elements and parameters of our algorithms.   For all experiments we set $\gamma=0.9$ and $\epsilon=0.1$.

\paragraph{Experiment 1: How do LTV and CTR compare?} For this experiment we show that every strategy has both a CTR and LTV metric as shown in Figure~\ref{fig:ctr-ltv} (Left).  In general the LTV metric gives higher numbers than the CTR metric. Estimating the LTV metric however gets harder as the trajectories get longer and as the mismatch with the behavior policy gets larger.  In this experiment the policy we evaluated was the random policy which is the same as the behavior policy, and in effect we eliminated the importance weighted factor.

\begin{figure}[ht!]
\centering
\includegraphics[width=0.45\textwidth]{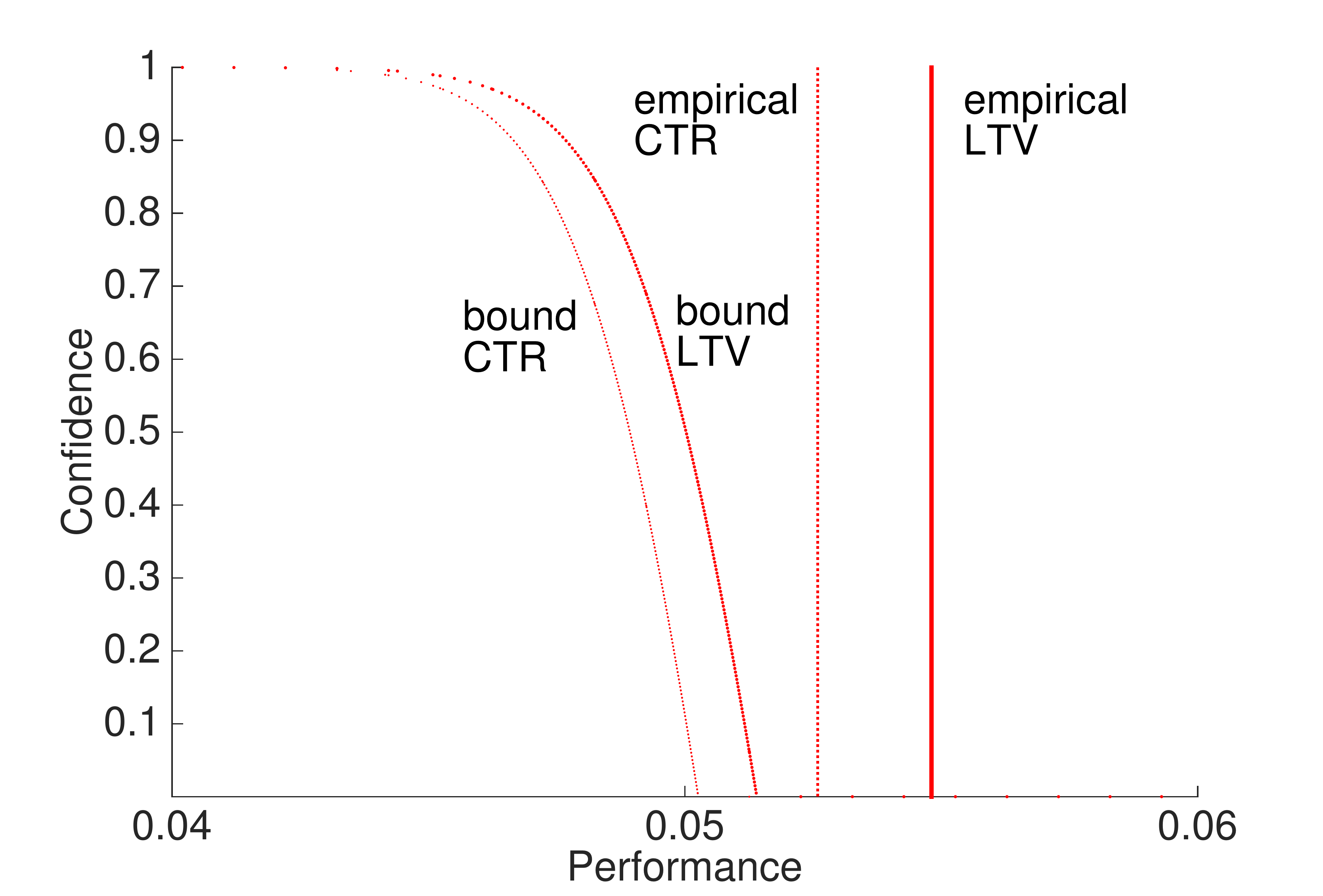}\includegraphics[width=0.45\textwidth]{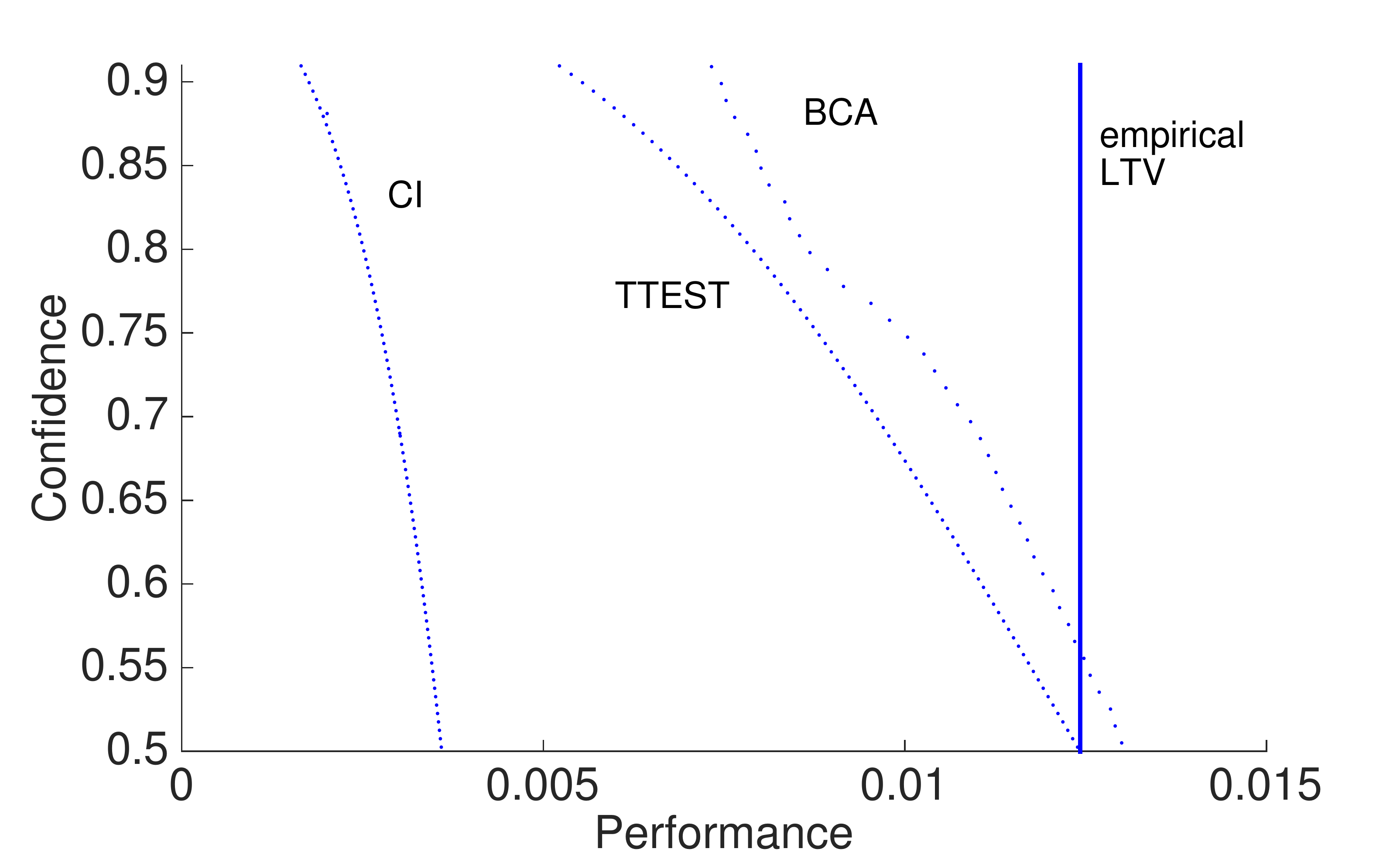}
\caption{ (Left) This figure shows the bounds and  empirical importance weighted returns for the random strategy.  It shows that every strategy has both a CTR and LTV metric.  This was done for data set 1.
(Right) This figure shows comparison between the 3 different bounds.  It was done for data set 2.}
\label{fig:ctr-ltv}
\end{figure}

\paragraph{Experiment 2: How do the three bounds differ?}  In this experiment we compared the 3 different lower-bound estimation methods, as shown in Figure~\ref{fig:ctr-ltv} (Right).  We observed that the bound for the $t$-test  is tighter than that for CI, but it makes the  false assumption that importance weighted returns are normally distributed.  We observed that the bound  for BCa  has higher confidence than the $t$-test approach for the same performance.  The BCa bound does not make  a Gaussian assumption, but still makes the false assumption that the distribution of future empirical returns will be the same as what has been observed in the past.


\paragraph{Experiment 3: When should each of the two optimization algorithms be used?}
In this experiment  we observed that the {\sc GreedyOptimization} algorithm performs the best under the CTR metric and the {\sc LtvOptimization} algorithm performs the best under the LTV metric as expected, see Figure \ref{fig:ctr-tt}.  The same claim holds for data set 2.

\begin{figure}[ht!]
\centering
\includegraphics[width=0.45\textwidth]{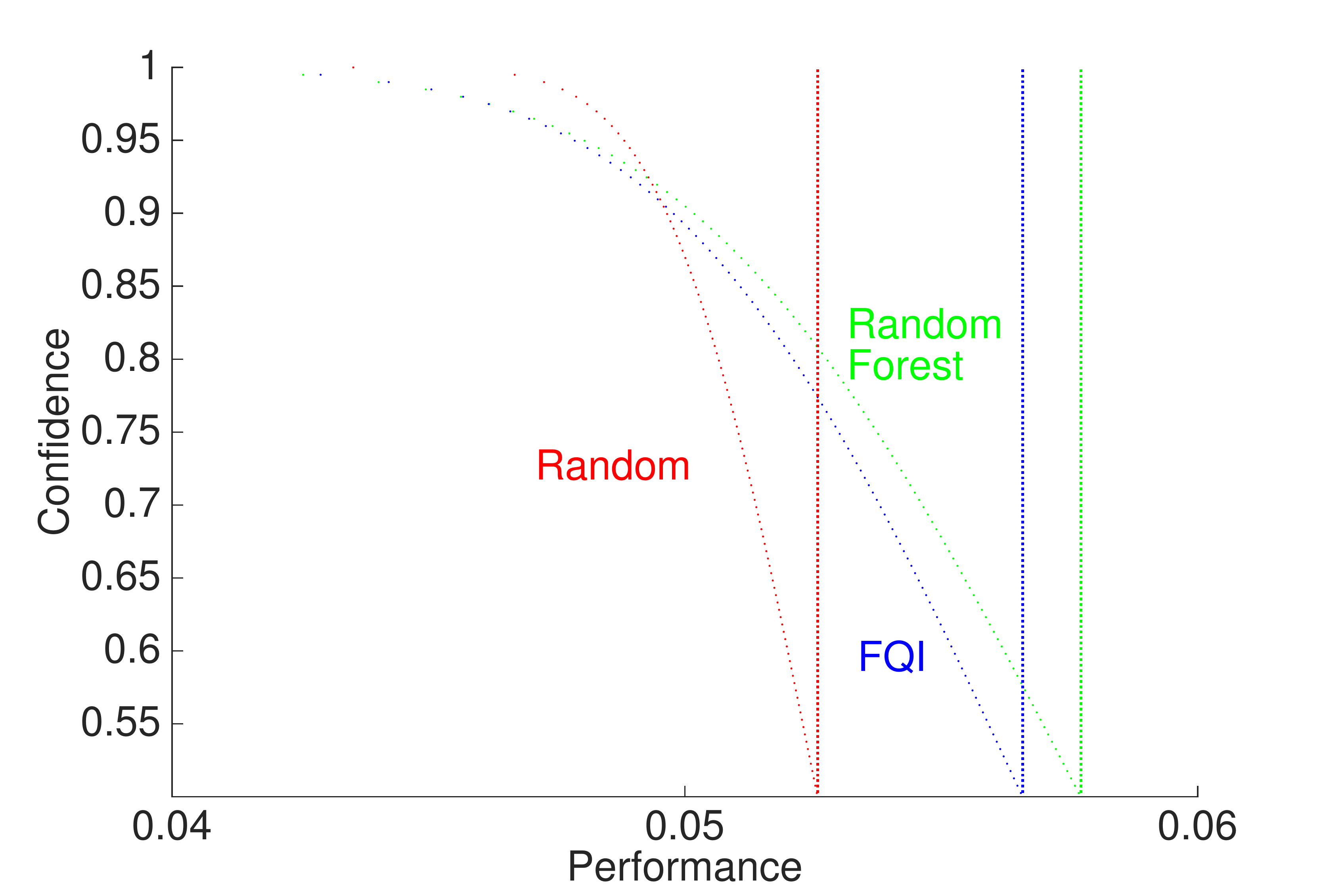}
\includegraphics[width=0.45\textwidth]{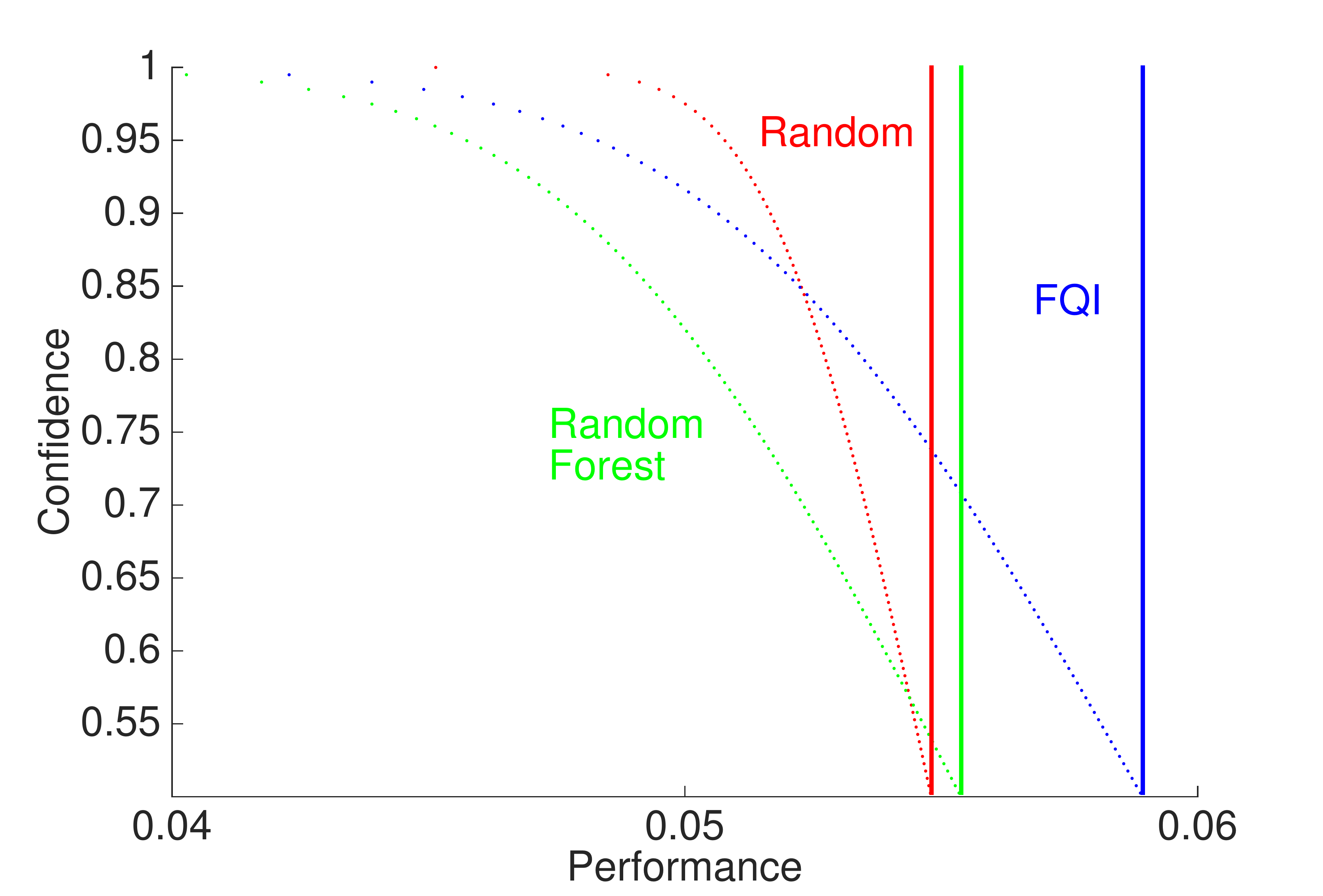}
\caption{(Left) This figure compares the CTR bounds of the Greedy versus the LTV optimization It was done for data set 1, but similar graphs exist for data set 2.
(Right) This figure compare the LTV bounds of the Greedy versus the LTV optimization It was done for data set 1, but similar graphs exist for data set 2.}
\label{fig:ctr-tt}
\end{figure}


\paragraph{Experiment 4: What is the effect of $\epsilon$?}  
One of the limitations of out algorithm is that it requires stochastic policies.  The closer the new policy is to the behavior policy the easier to estimate the performance. Therefore, we approximate our policies with $\epsilon$-greedy and use the random data for the behavior policy.  The larger the $\epsilon$, the easier is to get a more accurate performance of a new policy, but at the same time we would be estimating the performance of a sub-optimal policy, which  has moved closer to the random policy, see Figure~\ref{fig:epsilon}.  Therefore, when using the bounds to compare two policies, such as Greedy vs.~LTV, one should use the same $\epsilon$.

\begin{figure}[h]
\centering
\includegraphics[height=1.9in]{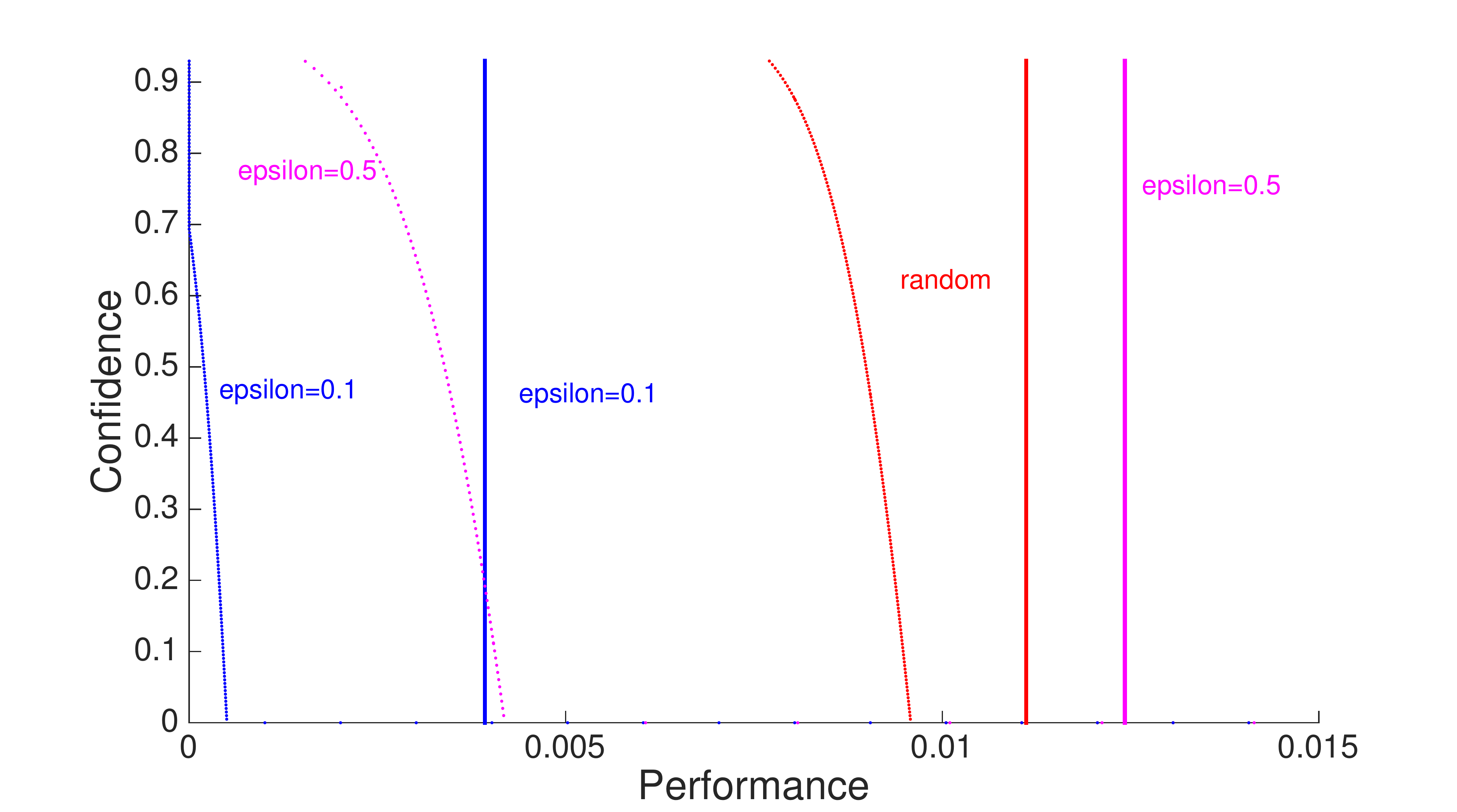}
\caption{The figure shows that as epsilon gets larger the policy moves towards the random policy.  Random polices are easy to estimate their performance since they match the behavior policy exactly. Thus epsilon should be kept same when comparing two policies.  This experiment was done on data set 2 and shows the bounds and empirical mean importance weighted returns (vertical line) for the LTV policy.  The bound used here was the CI.}
\label{fig:epsilon}
\end{figure}


\section{Safe Deployment}
\label{sec:safety}
In the previous sections we described how to compute an SR in combination with high confidence off-policy evaluation for deployment with some guarantees.  In the real world the deployment may need to happen incrementally, where at fixed intervals of time we would like to update the current SR policy in a safe manner.  In this section we present a batch reinforcement learning (RL) algorithm that provides probabilistic guarantees about the quality of each policy that it proposes, and which has no hyper-parameter that requires expert tuning. Specifically, the user may select any performance lower-bound, $\rho_-$, and confidence level, $\delta$, and our algorithm will ensure that the probability that it returns a policy with performance below $\rho_-$ is at most $\delta$. We then propose an incremental algorithm that executes our policy improvement algorithm repeatedly to generate multiple policy improvements. We show the viability of our approach with a digital marketing application that uses real world data
\cite{DBLP:conf/icml/ThomasTG15}.

\subsection{Problem Formulation}
\label{problem}

Given a (user specified) lower-bound, $\rho_-$, on the performance and a confidence level, $\delta$, we call an RL algorithm \textit{safe} if it ensures that the probability that a policy with performance less than $\rho_-$ will be proposed is at most $\delta$. 
The only assumption that a safe algorithm may make is that the underlying environment is a POMDP. Moreover, we require that the safety guarantee must hold regardless of how any hyperparameters are tuned.

We call an RL algorithm \textit{semi-safe} if it would be safe, except that it makes a false but reasonable assumption. Semi-safe algorithms are of particular interest when the assumption that the environment is a POMDP is significantly stronger than any (other) false assumption made by the algorithm, e.g.,~that the sample mean of the importance weighted returns is normally distributed when using many trajectories.


We call a policy, $\pi$, (as opposed to an algorithm) safe if we can ensure that $\rho(\pi) \geq \rho_-$ with confidence $1-\delta$. Note that ``a policy is safe" is a statement about our belief concerning that policy given the observed data, and not a statement about the policy itself. 

If there are many policies that might be deemed safe, then the policy improvement mechanism should return the one that is expected to perform the best, i.e.,
\begin{equation}
\pi' \in \arg \max_{\text{safe }\pi} g(\pi \vert \mathcal D),
\label{eq:g}
\end{equation}
where $g(\pi \vert \mathcal D)\in \mathbb R$ is a prediction of $\rho(\pi)$ computed from $\mathcal D$. We use a lower-variance, but biased, alternative to ordinary importance sampling, called \textit{weighted} importance sampling~\cite{Precup2000}, for $g$, i.e.,~
$$
g(\pi \vert \mathcal D) \coloneqq  \frac{\sum_{i=1}^{\lvert \mathcal D \rvert} \hat \rho(\pi \vert \tau_i^\mathcal D, \pi_i^\mathcal D) }{ \sum_{i=1}^{\lvert \mathcal D \rvert} \hat w(\tau_i^\mathcal D, \pi, \pi_i^\mathcal D)}.
$$
Note that even though Eq.~\eqref{eq:g} uses $g$, our safety guarantee is uncompromising---it uses the true (unknown and often unknowable) expected return, $\rho(\pi)$.

In the following sections, we present batch and incremental policy improvement algorithms that are safe when they use the CI approach to HCOPE and semi-safe when they use the $t$-test or BCa approaches. Our algorithms have no hyperparameters that require expert tuning.


In the following, we use the $\dagger$ symbol as a placeholder for either CI, TT, or BCa. We also overload the symbol $\rho_-^\dagger$ so that it can take as input a policy, $\pi$, and a set of trajectories, $\mathcal D$, in place of $\mathbf X$, as follows:  
\begin{equation}
\rho_-^\dagger (\pi, \mathcal D, \delta, m) \coloneqq \rho_-^\dagger \Big ( \underbrace{\bigcup_{i=1}^{\lvert \mathcal D \rvert} \left \{ \hat \rho \left (\pi \vert \tau_i^{\mathcal D}, \pi_i^{\mathcal D} \right ) \right \}}_{\mathbf X},\delta, m \Big ).
\end{equation}
For example, $\rho_-^\text{BCa}(\pi, \mathcal D, \delta, m)$ is a prediction made using the data set $\mathcal D$ of what the $1-\delta$ confidence lower-bound on $\rho(\pi)$ would be, if computed from $m$ trajectories by BCa.


\subsection{Safe Policy Improvement}
\label{policyImprovement}

Our proposed batch (semi-)safe policy improvement algorithm, {\sc PolicyImprovement}$^\dagger_\ddagger$, takes as input a set of trajectories labeled with the policies that generated them, $\mathcal D$, a performance lower bound, $\rho_-$, and a confidence level, $\delta$, and outputs either a new policy or {\sc No Solution Found (NSF)}. The meaning of the $\ddagger$ subscript will be described later.

When we use $\mathcal D$ to both search the space of policies and perform safety tests, we must be careful to avoid the \textit{multiple comparisons problem} \cite{Benjamin1995}. 
To make this important problem clear, consider what would happen if our search of policy space included only two policies, and used all of $\mathcal D$ to test both of them for safety. If at least one is deemed safe, then we return it. HCOPE methods can incorrectly label a policy as safe with probability at most $\delta$. However, the system we have described will make an error whenever either policy is incorrectly labeled as safe, which means its error rate can be as large as $2\delta$. In practice the search of policy space should include many more than just two policies, which would further increase the error rate.
%

%
We avoid the multiple comparisons problem by setting aside data that is only used for a \textit{single} safety test that determines whether or not a policy will be returned. Specifically, we first partition the data into a small training set, $\mathcal D_\text{train}$, and a larger test set, $\mathcal D_\text{test}$. The training set is used to search for which single policy, called the \textit{candidate policy}, $\pi_c$, should be tested for safety using the test set. This policy improvement method, {\sc PolicyImprovement}$^\dagger_\ddagger$, is reported in Algorithm~\ref{alg:PolicyImprovement}. To simplify later pseudocode, {\sc PolicyImprovement}$^\dagger_\ddagger$ assumes that the trajectories have already been partitioned into $\mathcal D_\text{train}$ and $\mathcal D_\text{test}$. In practice, we place $1/5$ of the trajectories in the training set and the remainder in the test set. Also, note that {\sc PolicyImprovement}$^\dagger_\ddagger$ can use the safe concentration inequality approach, $\dagger=$ CI, or the semi-safe $t$-test or BCa approaches, $\dagger \in \{$ TT, BCa$\}$.

{\sc PolicyImprovement}$^\dagger_\ddagger$ is presented in a top-down manner in Algorithm \ref{alg:PolicyImprovement}, and makes use of the {\sc GetCandidatePolicy}$^\dagger_\ddagger(\mathcal D, \delta, \rho_-, m)$ method, which searches for a candidate policy. The input $m$ specifies the number of trajectories that will be used during the subsequent safety test. Although {\sc GetCandidatePolicy}$^\dagger_\ddagger$ could be any batch RL algorithm, like LSPI or FQI \citep{Lagoudakis2001,Ernst05tree-basedbatch}, we propose an approach that leverages our knowledge that the candidate policy must pass a safety test. We will present two versions of {\sc GetCandidatePolicy}$^\dagger_\ddagger$, which we differentiate between using the subscript $\ddagger$, which may stand for None or $k$-fold. 

Before presenting these two methods, we define an objective function $f^\dagger$ as:
$$
	f^\dagger(\pi, \mathcal D,\delta,\rho_-,m)\coloneqq
			\begin{cases}
g(\pi \vert \mathcal D)&\hspace{-1.25cm}\mbox{if } \;\rho_-^\dagger \left (\pi, \mathcal D, \delta, m \right ) \geq \rho_-, \\
				\rho_-^\dagger \left (\pi, \mathcal D, \delta, m \right ) &\mbox{\hspace{0.8cm}otherwise.} 		
			\end{cases}
$$
Intuitively, $f^\dagger$ returns the predicted performance of $\pi$ if the predicted lower-bound on $\rho(\pi)$ is at least $\rho_-$, and the predicted lower-bound on $\rho(\pi)$, otherwise.

Consider {\sc GetCandidatePolicy}$^\dagger_\text{None}$, which is presented in Algorithm \ref{alg:CandidateNone}. This method uses all of the available training data to search for the policy that is predicted to perform the best, subject to it also being predicted to pass the safety test. That is, if no policy is found that is predicted to pass the safety test, it returns the policy, $\pi$, that it predicts will have the highest lower bound on $\rho(\pi)$. If policies are found that are predicted to pass the safety test, it returns one that is predicted to perform the best (according to $g$). 

The benefits of this approach are its simplicity and that it works well when there is an abundance of data. However, when there are few trajectories in $\mathcal D$ (e.g.,~cold start), this approach has a tendency to overfit---it finds a policy that it predicts will perform exceptionally well and which will easily pass the safety test, but actually fails the subsequent safety test in {\sc PolicyImprovement}$^\dagger_\text{None}$. We call this method $\ddagger$ = None because it does not use any methods to avoid overfitting.

\begin{algorithm}
\small
\begin{algorithmic}[1]
	\STATE $\pi_c \gets $ {\sc GetCandidatePolicy}$^\dagger_\ddagger(\mathcal D_\text{train}, \delta, \rho_-, \lvert \mathcal D_\text{test} \rvert )$
	\STATE {\bf if }$\rho_-^\dagger \left (\pi_c, \mathcal D_\text{test}, \delta, \lvert \mathcal D_\text{test} \rvert \right ) \geq \rho_-$ {\bf then return }$\pi_c$
	\STATE {\bf return }{\sc NSF}
\end{algorithmic}
\caption{\small {\sc PolicyImprovement}
$^\dagger_\ddagger(\mathcal D_\text{train}, \mathcal D_\text{test}, \delta, \rho_-)$
Either returns {\sc No Solution Found (NSF)} or a \mbox{\text{(semi-)}}safe policy. Here $\dagger$ can denote either CI, TT, or BCa.}
\label{alg:PolicyImprovement}
\end{algorithm}

\begin{algorithm}
\small
\begin{algorithmic}[1]
	\STATE {\bf return }$
\arg \max_\pi f^\dagger(\pi, \mathcal D, \delta, \rho_-, m)
$
\end{algorithmic}
\caption{\small {\sc GetCandidatePolicy}$_\text{None}^\dagger(\mathcal D, \delta, \rho_-, m)$Searches for the candidate policy, but does nothing to mitigate overfitting.}
\label{alg:CandidateNone}
\end{algorithm}

In machine learning, it is common to introduce a regularization term, $\alpha \lVert w \rVert$, into the objective function in order to prevent overfitting. Here $w$ is the model's weight vector and $\lVert \cdot \rVert$ is some measure of the complexity of the model (often $L_1$ or squared $L_2$-norm), and $\alpha$ is a parameter that is tuned using a model selection method like cross-validation. This term penalizes solutions that are too complex, since they are likely to be overfitting the training data.

Here we can use the same intuition, where we control for the complexity of the solution policy using a regularization parameter, $\alpha$, that is optimized using $k$-fold cross-validation. Just as the squared $L_2$-norm relates the complexity of a weight vector to its squared distance from the zero vector, we define the complexity of a policy to be some notion of its distance from the initial policy, $\pi_0$. In order to allow for an intuitive meaning of $\alpha$, rather than adding a regularization term to our objective function, $f^\dagger(\cdot, \mathcal D_\text{train}, \delta, \rho_-, \vert \mathcal D_\text{test} \vert )$, we directly constrain the set of policies that we search over to have limited complexity. 

We achieve this by only searching the space of mixed policies $\mu_{\alpha, \pi_0, \pi}$, where $\mu_{\alpha, \pi_0, \pi}(a|s) \coloneqq \alpha \pi(a|s) + (1 - \alpha)\pi_0$. Here, $\alpha$ is the fixed regularization parameter, $\pi_0(a|s)$ is the fixed initial policy, and we search the space of all possible $\pi$.
Consider, for example what happens to the probability of action $a$ in state $s$ when $\alpha=0.5$. If $\pi_0(a|s)=0.4$, then for any $\pi$, we have that $\mu_{\alpha, \pi_0, \pi}(a\vert s) \in [0.2,0.7]$. That is, the mixed policy can only move $50\%$ of the way towards being deterministic (in either direction). In general, $\alpha$ denotes that the mixed policy can change the probability of an action no more than $100\alpha\%$ towards being deterministic. So, using mixed policies results in our searches of policy space being constrained to some \textit{feasible set} centered around the initial policy, and where $\alpha$ scales the size of this feasible set.

While small values of $\alpha$ can effectively eliminate overfitting by precluding the mixed policy from moving far away from the initial policy, they also limit the quality of the best mixed policy in the feasible set. 
It is therefore important that $\alpha$ is chosen to balance the tradeoff between overfitting and limiting the quality of solutions that remain in the feasible set. Just as in machine learning, we use $k$-fold cross-validation to automatically select $\alpha$.


This approach is provided in Algorithm~\ref{alg:CandidateKFold}, where {\sc CrossValidate}$^\dagger(\alpha, \mathcal D, \delta, \rho_-, m)$ uses $k$-fold cross-validation to predict the value of $f^\dagger(\pi, \mathcal D_\text{test}, \delta, \rho_-, \lvert \mathcal D_\text{test} \rvert)$ if $\pi$ were to be optimized using $\mathcal D_\text{train}$ and regularization parameter $\alpha$. {\sc CrossValidate}$^\dagger$ is reported in Algorithm~\ref{alg:CrossValidate}. In our implementations we use $k=\min\{20, \frac{1}{2}\lvert \mathcal D \rvert \}$ folds.


\begin{algorithm}
\small
\begin{algorithmic}[1]
	\STATE {\small $\alpha^\star \gets \arg \max_{\alpha\in[0,1]}\text{\sc CrossValidate}^\dagger(\alpha, \mathcal D, \delta, \rho_-, m)$}
	\STATE $	\pi^\star \gets \arg \max_\pi f^\dagger(\mu_{\alpha^\star, \pi_0,\pi}, \mathcal D, \delta, \rho_-, m)$
	\STATE {\bf return }$\mu_{\alpha^\star,\pi_0,\pi^\star}$
\end{algorithmic}
\caption{\small{\sc GetCandidatePolicy}$_\text{$k$-fold}^\dagger(\mathcal D, \delta, \rho_-, m)$Searches for the candidate policy using $k$-fold cross-validation to avoid overfitting.}
\label{alg:CandidateKFold}
\end{algorithm}

\begin{algorithm}
\small
\begin{algorithmic}[1]
	\STATE Partition $\mathcal D$ into $k$ subsets, $\mathcal D_1, \hdots, \mathcal D_k$, of approximately the same size.
	\STATE result $\gets  0$
	\FOR{$i=1$ {\bf to} $k$}
			\STATE $\widehat {\mathcal D} \gets \bigcup_{j \neq i} \mathcal D_j$
			\STATE $\pi^\star \gets \arg \max_\pi f^\dagger(\mu_{\alpha, \pi_0,\pi}, \widehat{\mathcal D}, \delta, \rho_-, m)$
		\STATE result $\gets$ result $ + f^\dagger(\mu_{\alpha, \pi_0,\pi^\star}, \mathcal D_i, \delta, \rho_-, m)$
	\ENDFOR
	\STATE {\bf return }result$/k$
\end{algorithmic}
\caption{\small{\sc CrossValidate}$^\dagger(\alpha, \mathcal D, \delta, \rho_-, m)$}
\label{alg:CrossValidate}
\end{algorithm}

\subsection{Daedalus}
\label{Daedalus}

The {\sc PolicyImprovement}$^\dagger_\ddagger$ algorithm is a batch method that can be applied to an existing data set, $\mathcal D$. However, it can also be used in an incremental manner by executing new safe policies whenever they are found. The user might choose to change $\rho_-$ at each iteration, e.g., to reflect an estimate of the performance of the best policy found so far or the most recently proposed policy. However, for simplicity in our pseudocode and experiments, we assume that the user fixes $\rho_-$ as an estimate of the performance of the initial policy. This scheme for selecting $\rho_-$ is appropriate when trying to convince a user to deploy an RL algorithm to tune a currently fixed initial policy, since it guarantees with high confidence that it will not decrease performance.

Our algorithm maintains a list, $\mathcal C$, of the policies that it has deemed safe. When generating new trajectories, it always uses the policy in $\mathcal C$ that is expected to perform best. $\mathcal C$ is initialized to include a single initial policy, $\pi_0$, which is the same as the baseline policy used by {\sc GetCandidatePolicy}$_\text{$k$-fold}^\dagger$. This online safe learning algorithm is presented in Algorithm \ref{alg:Daedalus}.\footnote{If trajectories are available \textit{a priori}, then $\mathcal D_\text{train}, \mathcal D_\text{test},$ and $\mathcal C$ can be initialized accordingly.} It takes as input an additional constant, $\beta$, which denotes the number of trajectories to be generated by each policy. If $\beta$ is not already specified by the application, it should be selected to be as small as possible, while allowing {\sc Daedalus}$^\dagger_\ddagger$ to execute within the available time. We name this algorithm {\sc Daedalus}$^\dagger_\ddagger$ after the mythological character who promoted safety when he encouraged Icarus to use caution.

\begin{algorithm}
\small
\begin{algorithmic}[1]
	\STATE $\mathcal C \gets \{\pi_0\}$
	\STATE $\mathcal D_\text{train} \gets \mathcal D_\text{test} \gets \emptyset$
	\WHILE {{\bf true}}
		\STATE $\widehat{\mathcal D} \gets \mathcal D_\text{train}$
		\STATE $\pi_\star \gets \arg \max_{\pi \in \mathcal C} g(\pi \vert \widehat{\mathcal D})$
		\STATE Generate $\beta$ trajectories using $\pi^\star$ and append $\lceil \beta/5 \rceil$ to $\mathcal D_\text{train}$ and the rest to $\mathcal D_\text{test}$
		\STATE $\pi_c \gets ${\sc PolicyImprovement}$^\dagger_\ddagger(\mathcal D_\text{train}, \mathcal D_\text{test}, \delta, \rho_-)$
		\STATE $\widehat{\mathcal D} \gets \mathcal D_\text{train}$
		\IF{$\pi_c \neq$ {\sc NSF} {\bf and } $g(\pi_c \vert \widehat{\mathcal D}) > \max_{\pi \in \mathcal C} g(\pi \vert \widehat{\mathcal D})$}
			\STATE $\mathcal C \gets \mathcal C \cup \pi_c$
			\STATE $\mathcal D_\text{test} \gets \emptyset$
		\ENDIF
	\ENDWHILE
\end{algorithmic}
\caption{\small{\sc Daedalus}$^\dagger_\ddagger(\pi_0, \delta, \rho_-, \beta)$Incremental policy improvement algorithm.}
\label{alg:Daedalus}
\end{algorithm}

The benefits of $\ddagger=k$-fold are biggest when only a few trajectories are available, since then {\sc GetCandidatePolicy}$^\dagger_\text{None}$ is prone to overfitting. When there is a lot of data, overfitting is not a big problem, and so the additional computational complexity of $k$-fold cross-validation is not justified. In our implementations of {\sc Daedalus}$^\dagger_{k\text{-fold}}$, we therefore only use $\ddagger = {k\text{-fold}}$ until the first policy is successfully added to $\mathcal C$, and $\ddagger$ = None thereafter. This provides the early benefits of $k$-fold cross-validation without incurring its full computational complexity.

The {\sc Daedalus}$^\dagger_\ddagger$ algorithm ensures safety with each newly proposed policy. That is, during each iteration of the while-loop, the probability that a new policy, $\pi$, where $\rho(\pi) < \rho_-$, is added to $\mathcal C$ is at most $\delta$. The multiple comparison problem is not relevant here because this guarantee is per-iteration. However, if we consider the safety guarantee over multiple iterations of the while-loop, it applies and means that the probability that at least one policy, $\pi$, where $\rho(\pi) < \rho_-$, is added to $\mathcal C$ over $k$ iterations is at most $\min\{1,k\delta\}$.

We define {\sc Daedalus2}$^\dagger_\ddagger$ to be {\sc Daedalus}$^\dagger_\ddagger$ but with line 11 removed. The multiple hypothesis testing problem does \emph{not} affect {\sc Daedalus2$^\dagger_\ddagger$} more than {\sc Daedalus}$^\dagger_\ddagger$, since the safety guarantee is per-iteration. However, a more subtle problem is introduced: the importance weighted returns from the trajectories in the testing set, $\hat \rho(\pi_c \vert \tau_i^{\mathcal D_\text{test}}, \pi_i^{\mathcal D_\text{test}})$, are not necessarily unbiased estimates of $\rho(\pi_c)$.
This happens because the policy, $\pi_c$, is computed in part from the trajectories in $\mathcal D_\text{test}$ that are used to test it for safety. This dependence is depicted in Figure \ref{fig:badInfluence}. We also modify {\sc Daedalus2}$^\dagger_\ddagger$ by changing lines 4 and 8 to $\widehat{\mathcal D} \gets \mathcal D_\text{train} \cup \mathcal D_\text{test}$, which introduces an additional minor dependence of $\pi_c$ on the trajectories in $\mathcal D_\text{test}^j$.

\begin{figure}[ht]
  \centering
  \includegraphics[width=0.8\columnwidth]{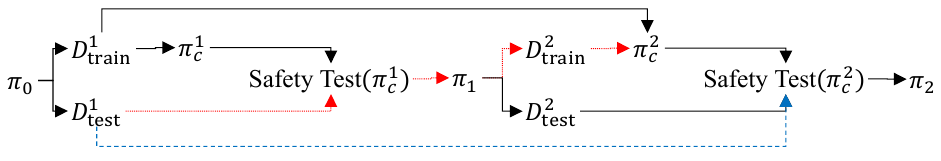}
  \caption{This diagram depicts influences as {\sc Daedalus2$^\dagger_\ddagger$} runs. First, $\pi_0$ is used to generate sets of trajectories, $\mathcal D_\text{train}^1$ and $\mathcal D_\text{test}^1$, where superscripts denote the iteration. Next $\mathcal D_\text{train}^1$ is used to select the candidate policy, $\pi_c^1$. Next, $\pi_c^1$ is tested for safety using the trajectories in $\mathcal D_\text{test}^1$ (this safety test occurs on line 2 of {\sc PolicyImprovement}$^\dagger_\ddagger$). The result of the safety test influences which policy, $\pi_1$, will be executed next. These policies are then used to produce $\mathcal D_\text{train}^2$ and $\mathcal D_\text{test}^2$ as before. Next, both $\mathcal D_\text{train}^1$ and $\mathcal D_\text{train}^2$ are used to select the candidate policy, $\pi_c^2$. This policy is then tested for safety using the trajectories in $\mathcal D_\text{test}^1$ and $\mathcal D_\text{test}^2$. The result of this test influences which policy, $\pi_2$, will be executed next, and the process continues. Notice that $\mathcal D_\text{test}^1$ is used when testing $\pi_c^2$ for safety (as indicated by the dashed blue line) even though it also influences $\pi_c^2$ (as indicated by the dotted red path). This is akin to performing an experiment, using the collected data ($\mathcal D_\text{test}^1$) to select a hypothesis ($\pi_c^2$ is safe), and then using that same data to test the hypothesis. {\sc Daedalus}$^\dagger_\ddagger$ does not have this problem because the dashed blue line is not present.}
\label{fig:badInfluence}
\end{figure}

Although our theoretical analysis applies to {\sc Daedalus}$^\dagger_\ddagger$, we propose the use of {\sc Daedalus2}$^\dagger_\ddagger$ because the ability of the trajectories, $\mathcal D_\text{test}^i$, to bias the choice of which policy to test for safety in the future ($\pi_c^j$, where $j > i$) towards a policy that $\mathcal D_\text{test}^i$ will deem safe, is small. However, the benefits of {\sc Daedalus2}$^\dagger_\ddagger$ over {\sc Daedalus}$^\dagger_\ddagger$ are significant---the set of trajectories used in the safety tests increases in size with each iteration, as opposed to always being of size $\beta$. So, in practice, we expect the over-conservativeness of $\rho_-^\text{CI}$ to far outweigh the error introduced by {\sc Daedalus2}$^\dagger_\ddagger$. Notice that {\sc Daedalus2}$^\text{CI}_\ddagger$ is safe (not just semi-safe) if we consider its execution up until the first change of the policy, since then the trajectories are always generated by $\pi_0$, which is not influenced by any of the testing data.

\subsection{Empirical Analysis}
\label{caseStudies}

\paragraph{Case Study:}

For our case study we used real data, captured with permission from the website of a Fortune 50 company that receives hundreds of thousands of visitors per day and which uses Adobe Target, to train a simulator using a proprietary in-house system identification tool at Adobe. The simulator produces a vector of $31$ real-valued features that provide a compressed representation of all of the available information about a user. The advertisements are clustered into two high-level classes that the agent must select between. After the agent selects an advertisement, the user either clicks (reward of $+1$) or does not click (reward of $0$) and the feature vector describing the user is updated.  Although this greedy approach has been successful, as we discussed in Section \ref{sec:par}, it does not necessarily also maximize the  total  number  of  clicks  from  each  user  over  his  or  her lifetime. Therefore, we consider a full reinforcement learning solution for this problem. We selected $T= 20$ and $\gamma =1$. This is a particularly challenging problem because the reward signal is sparse. If each action is selected with probability $0.5$ always, only about $0.38\%$ of the transitions are rewarding, since users usually do not click on the advertisements. This means that most trajectories provide no feed-back. Also, whether a user clicks or not is close to random, so returns have relatively high variance. We  generated  data  using  an  initial  baseline  policy  and then evaluated a new policy proposed by an in-house reinforcement learning algorithm. 

In order to avoid the large costs associated with deployment  of  a  bad  policy,  in  this  application  it  is  imperative that new policies proposed by RL algorithms are ensured to be safe before deployment.

\paragraph{Results:}
In our experiments, we selected $\rho_-$ to be an empirical estimate of the performance of the initial policy and $\delta = 0.05$. We used CMA-ES \cite{Hansen2006} to solve all $\arg \max_\pi$, where $\pi$ was parameterized by a vector of policy parameters using linear softmax action selection \cite{Sutton1998} with the Fourier basis \cite{Konidaris2011}.


\begin{figure}
\centering
                \includegraphics[width=0.5\textwidth]{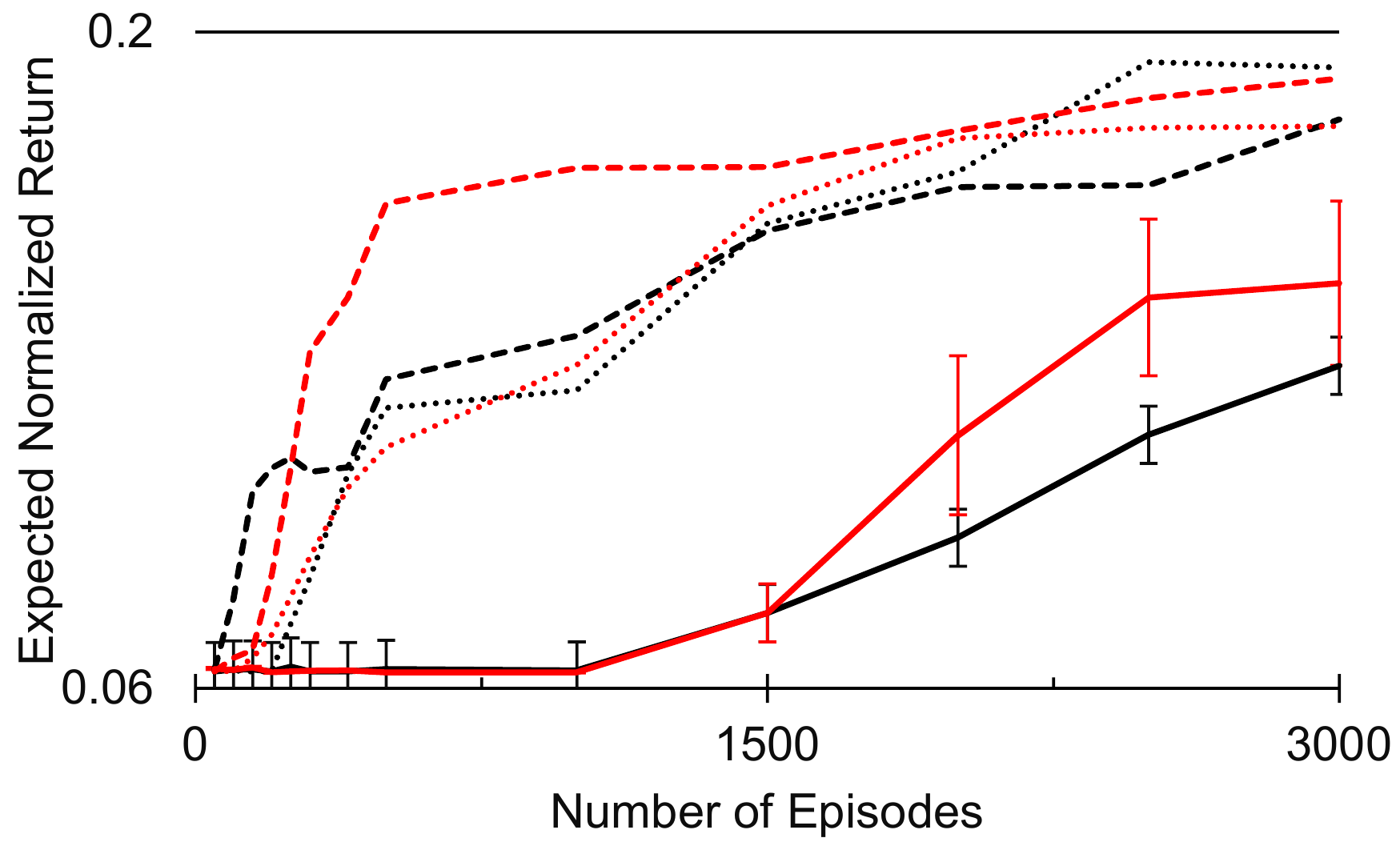}
     \includegraphics[width=0.15\textwidth]{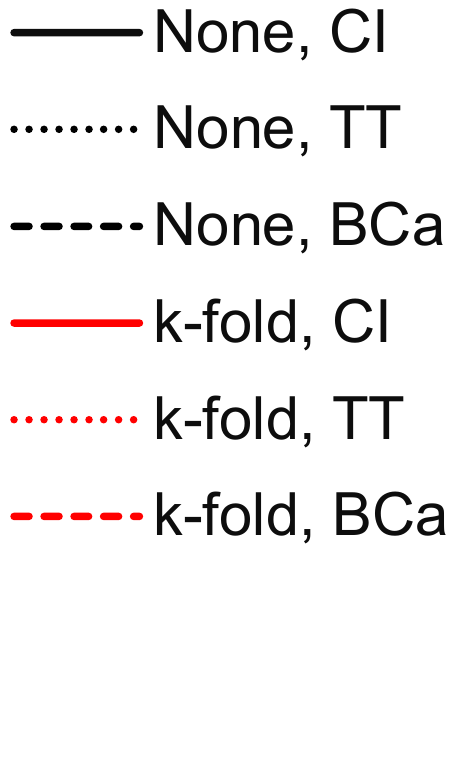}
                
	\caption{
	%
	Performance of {\sc Daedalus2}$^\dagger_\ddagger$ on the digital marketing domain. The legend specifies $\ddagger, \dagger$.
	}   \label{fig:digitalMarketing}
\end{figure}



For our problem domain, we executed {\sc Daedalus2}$^\dagger_\ddagger$ with $\dagger \in \{$CI, TT, BCa$\}$ and $\ddagger \in \{$None, $k$-fold$\}$. Ideally, we would use $\beta=1$ for all domains. However, as $\beta$ decreases, the runtime increases. We selected $\beta \in [50, 100,500]$ for the digital marketing domain. $\beta$ increases with the number of trajectories in the digital marketing domain so that the plot can span the number of trajectories required by the CI approach without requiring too many calls to the computationally demanding {\sc PolicyImprovement}$^\text{BCa}_{k\text{-fold}}$ method. We did not tune $\beta$ for these experiments---it was set solely to limit the runtime. 

The performance of {\sc Daedalus2}$^\dagger_\ddagger$ on the digital marketing domain is provided in Figure \ref{fig:digitalMarketing}. The expected normalized returns in Figure \ref{fig:digitalMarketing} are computed using $20,\!000$ Monte Carlo rollouts, respectively. The curves are also averaged over $10$ trials, respectively, with standard error bars provided when they do not cause too much clutter.

First, consider the different values for $\dagger$. As expected, the CI approaches (solid curves) are the most conservative, and therefore require the most trajectories in order to guarantee improvement. The BCa approaches (dashed lines) perform the best, and are able to provide high-confidence guarantees of improvement with as few as $50$ trajectories. The TT approach (dotted lines) perform in-between the CI and BCa approaches, as expected (since the $t$-test tends to produce overly conservative lower bounds for distributions with heavy upper tails).

Next, consider the different values of $\ddagger$. Using $k$-fold cross-validation provides an early boost in performance by limiting overfitting when there are few trajectories in the training set. Although the results are not shown, we experimented with using $\ddagger=k$-fold for the entire runtime (rather than just until the first policy improvement), but found that while it did increase the runtime significantly, it did not produce much improvement.

\section{Non-Stationarity}
\label{sec:nonstationary}

In the previous sections we made a critical assumption that the domain can be modeled as a POMDP.
However, real world problems are often non-stationary.  In this section we consider the problem of evaluating an SR policy off-line without assuming stationary transition and rewards. We argue that off-policy policy evaluation for non-stationary MDPs can be phrased as a time series prediction problem, which results in predictive methods that can anticipate changes before they happen. We
therefore propose a synthesis of existing off-policy policy evaluation methods with existing time series prediction methods, which we show results in a drastic reduction of mean squared error when evaluating policies using real digital marketing data set~\cite{DBLP:conf/aaai/ThomasTGDB17}.

\subsection{Motivating Example}
\label{subsec:motivatehotel}
In digital marketing applications,  when a person visits the website of a company, she is often shown a list of current promotions. In order for the display of these promotions to be effective, it must be properly targeted based on the known information about the person (e.g., her interests, past travel behavior, or income). The problem now reduces to automatically deciding which promotion (sometimes called a \textit{campaign}) to show to the visitor of a website.


As we have described  in Section \ref{sec:par} the system's goal is to determine how to select actions (select promotions to display) based on the available observations (the known information of the visitor) such that the reward is maximized (the number of clicks is maximized). Let $\rho(\pi_e,\iota)$ be the performance of the policy $\pi_e$ in episode $\iota$. In the bandit setting $\rho(\pi_e,\iota)$ is the expected number of clicks \emph{per visit}, called the \textit{click through rate} (CTR), while in the reinforcement learning setting it is the expected number of clicks \emph{per user}, called the \textit{life-time value} (LTV).

In order to determine how much of a problem non-stationarity really is, we collected data from the website of one of Adobe's Test and Target customers: the website of a large company in the hotel and entertainment industry. We then used a proprietary policy search algorithm custom designed for digital marketing to generate a new policy for the customer. We then collected $n\approx 300,\!000$ new episodes of data, which we used as $D$, to compute $\operatorname{OPE}(\pi_e,\iota|D)$ for all $\iota \in \{0,\dotsc,n-1\}$ using ordinary importance sampling. Figure \ref{fig:Hotel} summarizes the resulting data.

\begin{figure}%
\centering
\includegraphics[width=0.35\columnwidth]{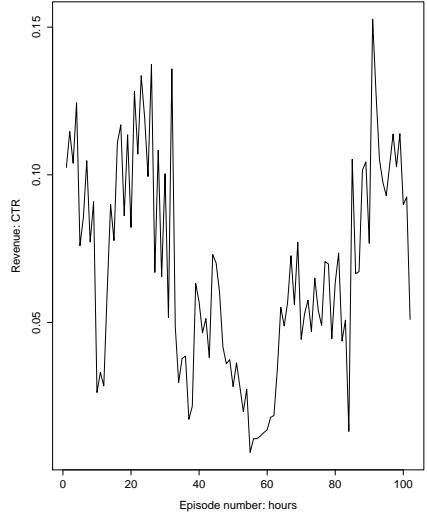}%
\caption{Plot of $\operatorname{OPE}(\pi_e, \iota | D)$ for various $\iota$, on the real-world digital marketing data from a large company in the hotel and entertainment industry. This data spans several days. Since the raw data has high variance, we bin the data into bins that each span one hour. Notice that the performance of the policy drops from an initial CTR of $0.1$ down to a near-zero CTR near the middle of the data set.}%
\label{fig:Hotel}%
\end{figure}

In this data it is evident that there is significant non-stationarity---the CTR varied drastically over the span of the plot. This is also not just an artifact of high variance: using Student's $t$-test we can conclude that the expected return during the first $100,\!000$ and subsequent $60,\!000$ episodes was different with $p=1.6\times 10^{-33}$. This is compelling evidence that we cannot ignore non-stationarity in our users' data when providing predictions of the expected future performance of our digital marketing algorithms, and is compelling real-world motivation for developing non-stationary off-policy policy evaluation algorithms.

\subsection{Nonstationary Off-Policy Policy Evaluation (NOPE)}

\begin{figure}[h]%
\centering
\includegraphics[width=0.35\columnwidth]{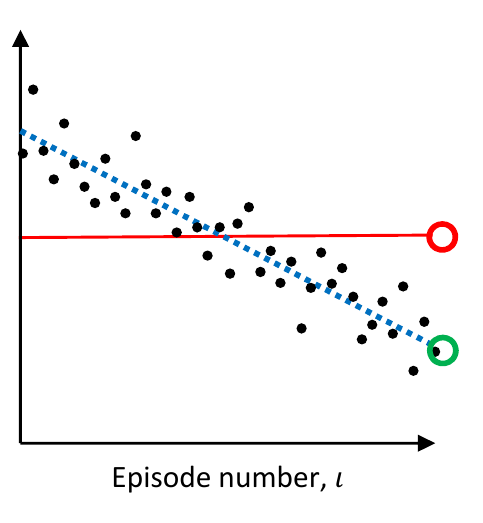}%
\caption{This illustration depicts an example of how the existing standard OPE methods produce \textit{reactive} behavior, and is hand-drawn to provide intuition. Here the dotted blue line depicts $\rho(\pi_e, \iota)$ for various $\iota$. The black dots denote $\operatorname{OPE}(\pi_e,\iota|D)$ for various $\iota$. Notice that each $\operatorname{OPE}(\pi_e,\iota|D)$ is a decent estimate of $\rho(\pi_e,\iota)$, which changes with $\iota$. Our goal is to estimate $\rho(\pi_e, n)$---the performance of the policy during the \textit{next} episode. That is, our goal is to predict the vertical position of the green circle. However, by averaging the OPE estimates, we get the red circle, which is a reasonable prediction of performance in the past. As more data arrives ($n$ increases) the predictions will decrease, but will always remain behind the target value of $\rho(\pi_e,n)$.}%
\label{fig:example}%
\end{figure}

\textit{Non-stationary Off-Policy Policy Evaluation} (NOPE) is simply OPE for non-stationary MDPs. In this setting, the goal is to use the available data $D$ to estimate $\rho(\pi_e, n)$---the performance of $\pi_e$ during the \textit{next episode} (the $n^\text{th}$ episode).

Notice that we have not made assumptions about how the transition and reward functions of the non-stationary MDP change. For some applications, they may drift slowly, making $\rho(\pi_e,\iota)$ change slowly with $\iota$. For example, this sort of drift may occur due to mechanical wear in a robot. For other applications, $\rho(\pi_e,\iota)$ may be fixed for some number of episodes, and then make a large jump. For example, this sort of jump may occur in digital marketing applications \citep{TheocharousTG15} due to media coverage of a relevant topic rapidly changing public opinion of a product. In yet other applications, the environment may include both large jumps and smooth drift.

Notice that NOPE can range from trivial to completely intractable. If the MDP has few states and actions, changes slowly between episodes, and the evaluation policy is similar to the behavior policy, then we should be able to get accurate off-policy estimates. On the other extreme, if for each episode the MDP's transition and reward functions are drawn randomly (or adversarially) from a wide distribution, then producing accurate estimates of $\rho(\pi_e,n)$ may be intractable.

\subsection{Predictive Off-Policy Evaluation using Time Series Methods}

The primary insight in this section, in retrospect, is obvious:
\textbf{NOPE is a time series prediction problem.}
Figure \ref{fig:example} provides an illustration of the idea.
Let $x_\iota = \iota$ and $Y_\iota = \operatorname{OPE}(\pi_e, \iota | D)$ for $\iota \in \{1,\dotsc,n-1\}$. This makes $x$ an array of $n$ times (each episode corresponds to one unit of time) and $y$ an array of the corresponding $n$ observations. Our goal is to predict the expected value of the next point in this time series, which will occur at $x_n=n$. Pseudocode for this \textit{time series prediction} (TSP) approach is given in Algorithm \ref{alg:TSP_POPE}.

\begin{algorithm}[]
   \caption{Time Series Prediction (TSP)}
   \label{alg:TSP_POPE}
\begin{algorithmic}[1]
	\STATE {\bfseries Input:} Evaluation policy, $\pi_e$, historical data, $D\coloneqq (H^\iota, \pi^\iota)_{\iota=0}^{n-1}$, and a time-series prediction algorithm (and its hyper-parameters).
\vspace{.1cm}
	\STATE Create arrays $x$ and $y$, both of length $n$.
	\FOR{$\iota=0$ {\bfseries to} $n-1$}
		\STATE $x_\iota \gets \iota$
		\STATE $y_\iota \gets \operatorname{OPE}(\pi_e, \iota | D)$
	\ENDFOR
	\STATE Train a time-series prediction algorithm on $x,y$.
	\STATE {\bfseries return} the time-series prediction algorithm's prediction for time $n$.
\end{algorithmic}
\end{algorithm}


When considering using time-series prediction methods for off-policy policy evaluation, it is important that we establish that the underlying process is actually nonstationary. One popular method for determining whether a process is stationary or nonstationary is to report the sample \emph{autocorrelation function} (ACF):
$$
\operatorname{ACF}_h \coloneqq \frac{\mathbf{E}[(X_{t+h}-\mu)(X_t- \mu)]}{\mathbf{E}[(X_t-\mu)^2]},
$$
where $h$ is a parameter called the \textit{lag} (which is selected by the researcher), $X_t$ is the time series, and $\mu$ is the mean of the time series.  For a stationary time series, the ACF will drop to zero relatively quickly, while the ACF of nonstationary data decreases slowly.

ARIMA models are models of time series data that can capture many different sources of non-stationarity. The time series prediction algorithm that we use in our experiments is the $R$ forecast package for fitting ARIMA models \cite{Hyndman2008}.

\subsection{Empirical Studies}
 
In this section we show that, despite the lack of theoretical results about using TSP for NOPE, it performs remarkably well on real data. Because our experiments use real-world data, we do not know ground truth---we have $\operatorname{OPE}(\pi_e,\iota|D)$ for a series of $\iota$, but we do not know $\rho(\pi_e,\iota)$ for any $\iota$. This makes evaluating our methods challenging---we cannot, for example, compute the true error or mean squared error of estimates. We therefore estimate the mean error and mean squared error directly from the data as follows. 

For each $\iota \in \{1,\dotsc,n-1\}$ we compute each method's output, $\hat y_\iota$, given all of the previous data, $D_{\iota-1}\coloneqq (H^{\hat \iota}, \pi^{\hat \iota})_{\hat \iota=0}^{\iota-1}$. We then compute the observed next value, $y_\iota = \operatorname{OPE}(\pi_e,\iota | D_\iota)$. From these, we compute the squared error, $(\hat y_\iota - y_\iota)^2$, and we report the mean squared error over all $\iota$. We perform this experiment using both the current standard OPE approach, which computes sample mean of performance over all the available data, and using our new time series prediction approach.

Notice that this scheme is not perfect. Even if an estimator perfectly predicts $J(\pi_e,\iota)$ for every $\iota$, it will be reported as having non-zero mean squared error. This is due to the high variance of $\operatorname{OPE}$, which gets conflated with the variance of $\hat y$ in our estimate of mean squared error. Although this means that the mean squared errors that we report are not good estimates of the mean squared error of the estimators, $\hat y$, the variance-conflation problem impacts all methods nearly equally. So, in the absence of ground truth knowledge, the reported mean squared error values are a reasonable measure of how accurate the methods are relative to each other.

The domain we consider is digital marketing using the data from the large companies in the hotel and entertainment industry as described in Section \ref{subsec:motivatehotel}. We refer to this domain as the \textit{Hotel} domain. For this domain, and all others, we used ordinary importance sampling for $\operatorname{OPE}$. Recall that the performance of the evaluation policy appears to drop initially---the probability of a user clicking decays from a remarkably high $10\%$ down to a near-zero probability---before it rises back to close to its starting level. Recall also that using a two-sided Student's $t$-test we found that the true mean during the first $100,\!000$ trajectories was different from the true mean during the subsequent $60,\!000$ trajectories with $p=1.6\times 10^{-33}$, so the non-stationarity that we see is likely not noise.

We collected additional data from the website of a large company in the financial industry, and used the same proprietary policy improvement algorithm to find a new policy that we might consider deploying for the user. There appears to be less long-term non-stationarity in this data, and a two-sided Student's $t$-test did not detect a difference between the early and late performance of the evaluation policy. We refer to this data as the \textit{Bank} domain.

\subsection{Results}
\begin{figure}%
\centering
\includegraphics[width=0.35\columnwidth]{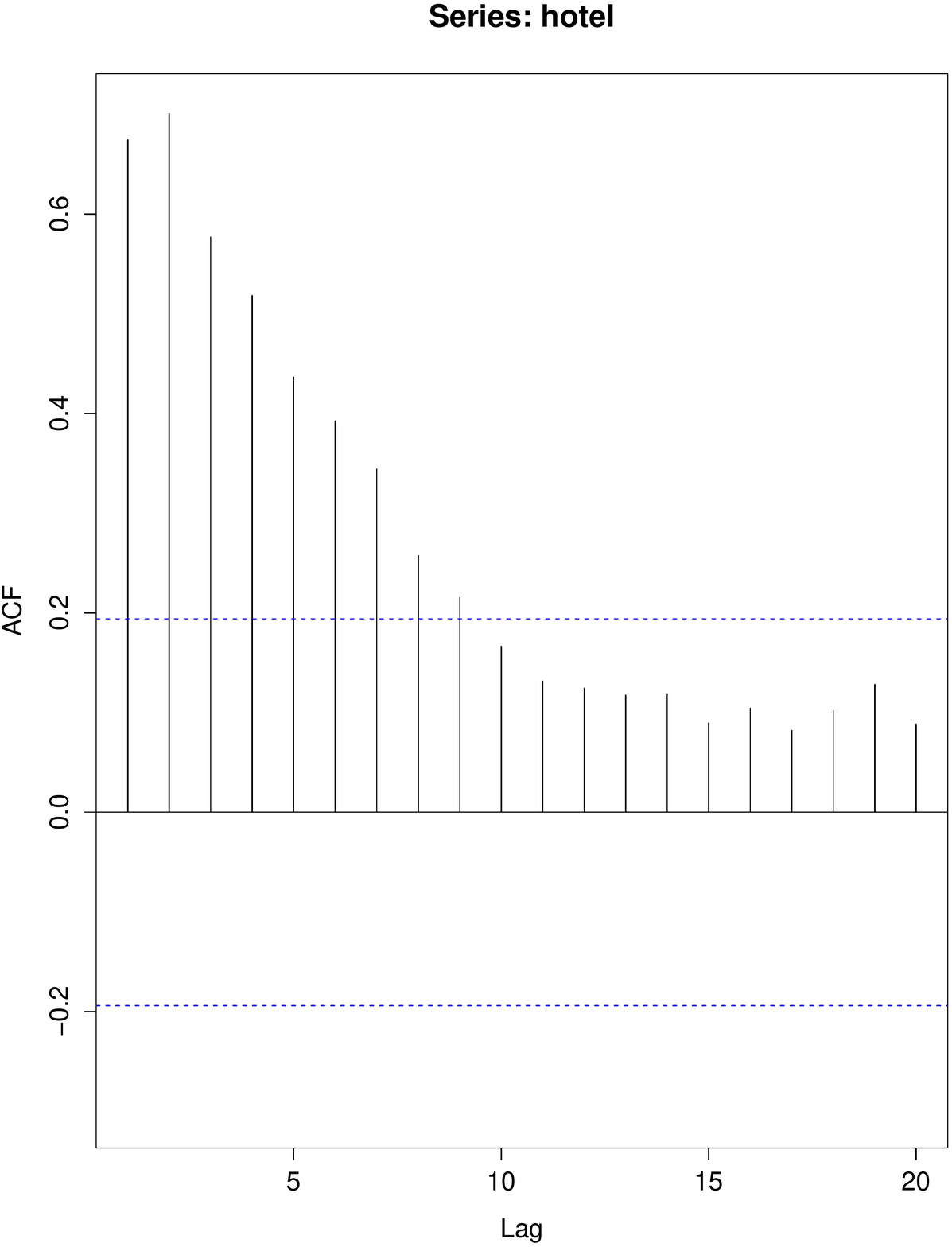}%
\includegraphics[width=0.35\columnwidth]{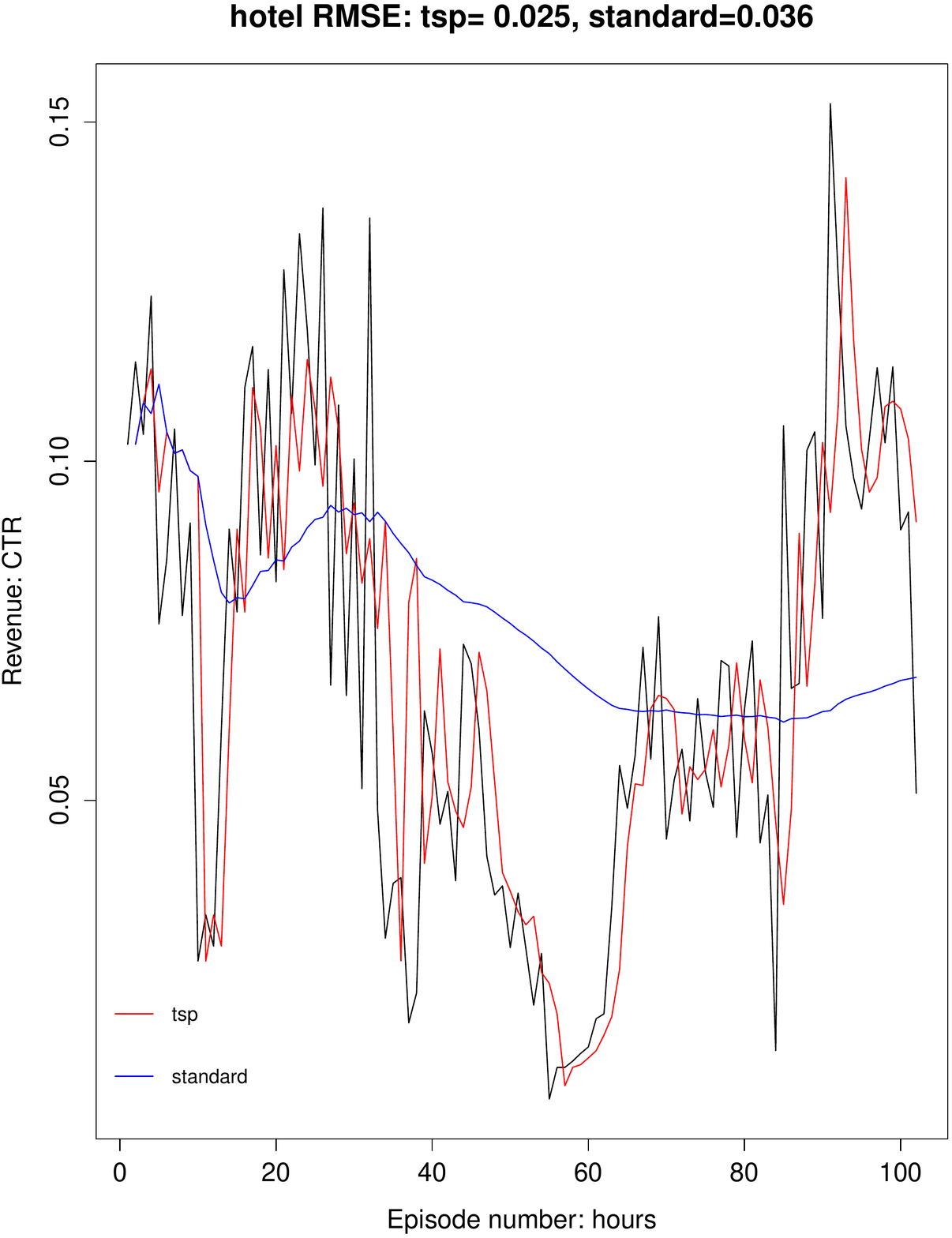}%
\\
\includegraphics[width=0.35\columnwidth]{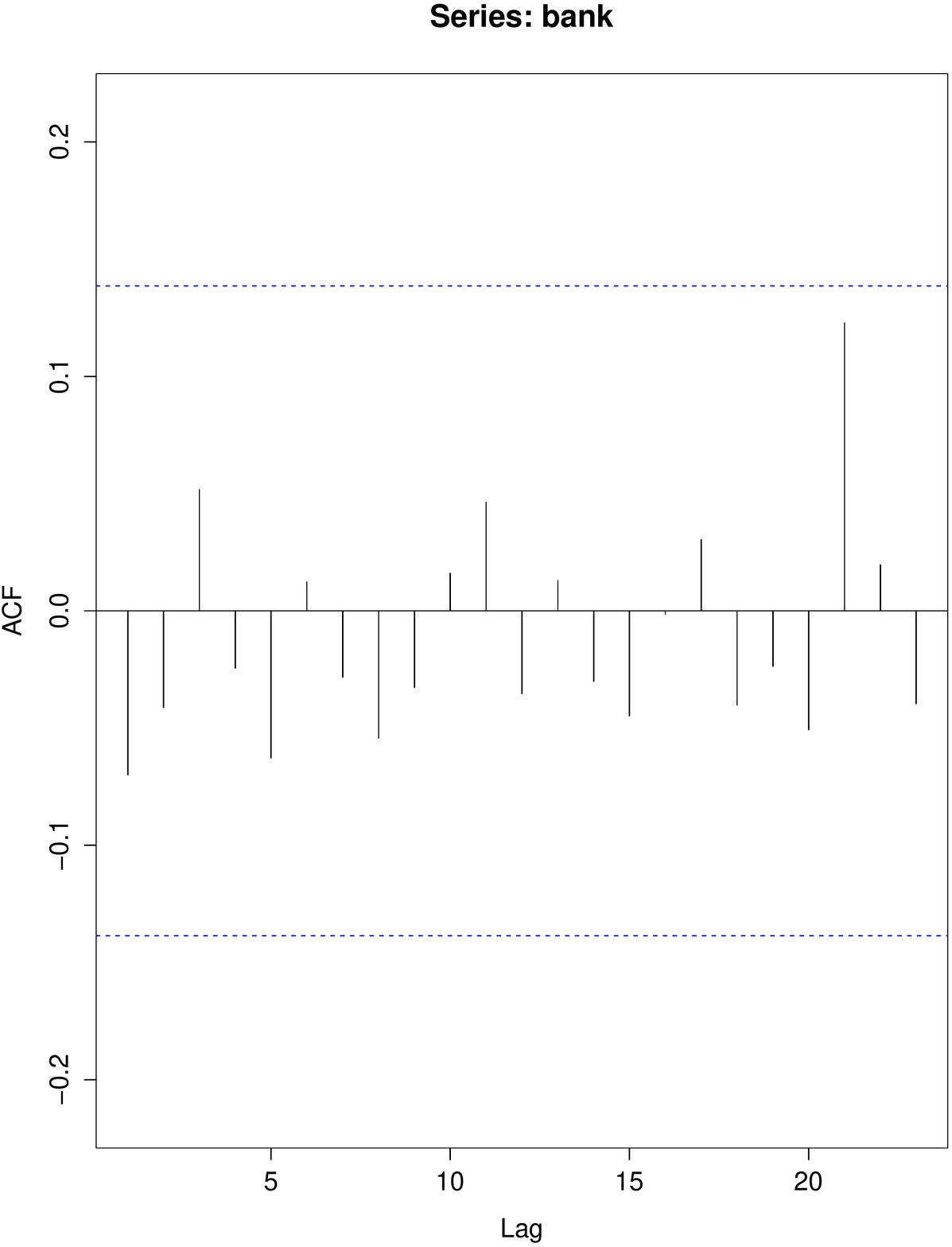}%
\includegraphics[width=0.35\columnwidth]{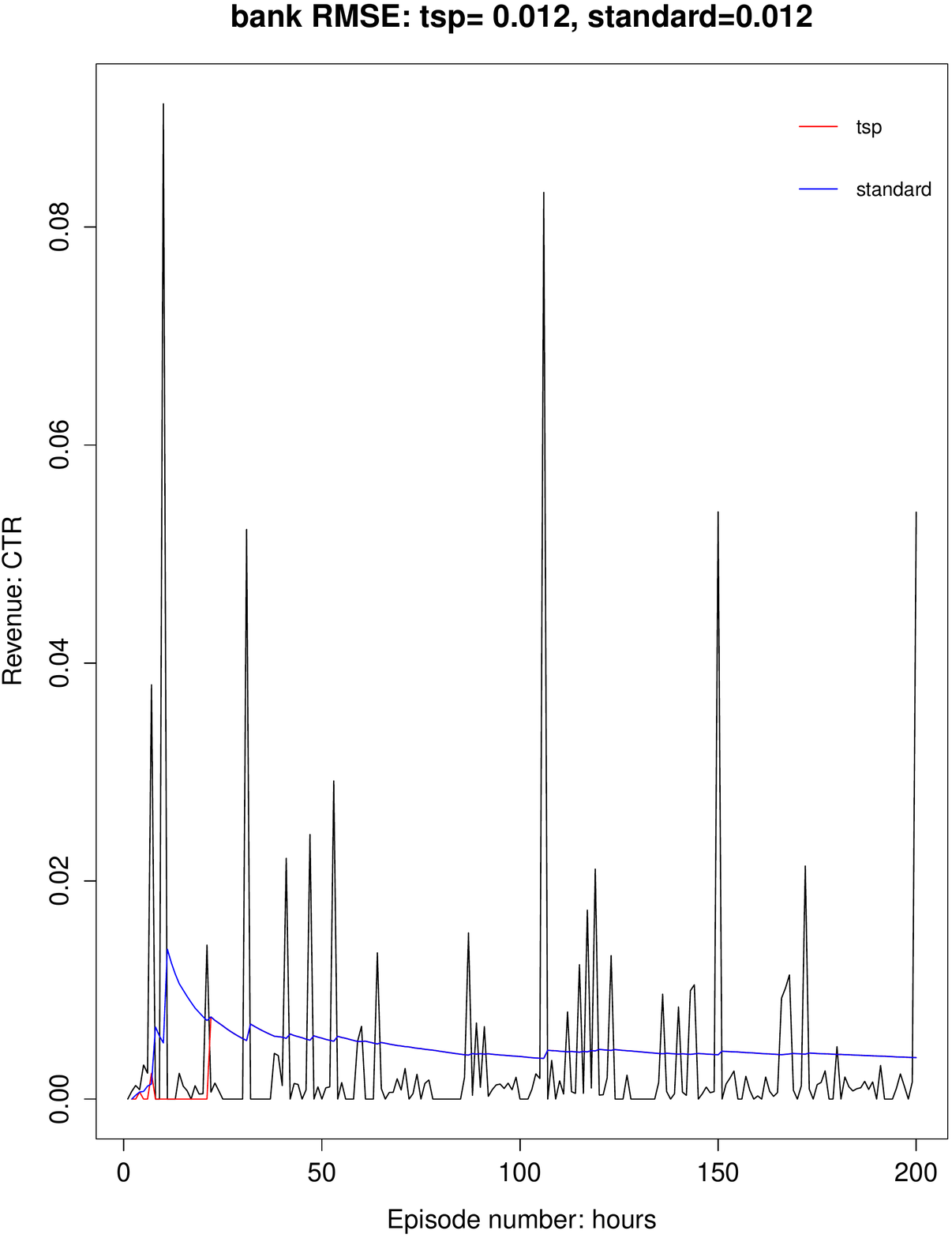}%
\caption{(Top) Hotel domain.  The left plot shows the auto-correlation for the time series, where it is obvious the signal nonstationary.  The right plot compares the tsp approach with the standard.  The tsp outperforms the standard approach, since the series is nonstationary. The time series was aggregated at the hour level.
(Bottom) Bank domain.  The left plot shows the autocorellation for the time series, where it is obvious the signal stationary.  The right plot compares the tsp approach with the standard.  They both perform the same, since the series is stationary. The time series was aggregated at the hour level.}%
\label{fig:hotel-ctr}%
\end{figure}

We applied our TSP algorithm for NOPE, described in Algorithm 
\ref{alg:TSP_POPE}, to the nonstationary hotel and bank data sets. The plots in Figure \ref{fig:hotel-ctr} all take the same form: the plots on the left are autocorrelation plots that show whether or not there appears to be non-stationarity in the data. As a rule of thumb, if the ACF values are within the dotted blue lines, then there is not sufficient evidence to conclude that there is non-stationarity. However, if the ACF values lie outside the dotted blue lines, it suggests that there is non-stationarity. 

The plots on the right depict the expected return (which is the expected CTR for the hotel and bank data sets) as predicted by several different methods. The black curves are the target values---the observed mean OPE estimate over a small time interval. For each episode number, our goal is to compute the value of the black curve given all of the previous values of the black curve. The blue curve does this using the standard method, which simply averages the previous black points. The red curve is our newly proposed method, which uses ARIMA to predict the next point on the black curve---to predict the performance of the evaluation policy during the next episode. Above the plots we report the sample \textit{root mean squared error} (RMSE) for our method, \textit{tsp}, and the standard method, \textit{standard}.


Consider the results on the hotel data set, which are depicted in Figure \ref{fig:hotel-ctr} (Top). The red curve (our method) tracks the binned values (black curve) much better than the blue curve (standard method). Also, the sample RMSE of our method is $0.025$, which is lower than the standard method's RMSE of $0.036$. This suggests that treating the problem as a time series prediction problem results in more accurate estimates.

Finally, consider the results on the bank data set, which are depicted in Figure \ref{fig:hotel-ctr} (Bottom). The auto-correlation plot suggests that there is not much non-stationarity in this data set. This validates another interesting use case for our method: does it break down when the environment happens to be (approximately) stationary? The results suggest that it does not---our method achieves the same RMSE as the standard method, and the blue and red curves are visually quite similar.

An interesting research question is whether our high-confidence policy evaluation and improvement algorithms can be extended to  non-stationary MDPs.
However, following TSP algorithm, it can be noticed that estimating performance of a policy with high-confidence in non-stationary MDP can be reduced to time-series forecasting with high-confidence, which in complete generality is infeasible.
An open research direction is to leverage domain specific structure of the problem and identify conditions under which this problem can be made feasible.

\section{Learning from Passive Data}
\label{sec:passive}
Constructing SR systems is particularly challenging due to the cold start problem. Fortunately, in many real world problems, there is an abundance of sequential data which are usually `passive' in that they do not include past recommendations. In this section we propose a practical approach that learns from passive data. We use scalar parameterization that turns a passive model into active, and posterior sampling for Reinforcement learning (PSRL) to learn the correct parameter value. In this section we summarize our work from \cite{DBLP:conf/iui/TheocharousVW17,DBLP:conf/nips/TheocharousWAV18}.

The idea is to first learn a model from passive data that predicts the next activity given the history of activities. This can be thought of as the ‘no-recommendation’ or passive model. To create actions for recommending the various activities, we can perturb the passive model. Each perturbed model increases the probability of following the recommendations, by a different amount. This leads to a set of models, each one with a different `propensity to listen’. In effect, the single `propensity to listen’ parameter is used to turn a passive model into a set of active models. When there are multiple models one can use online algorithms, such as posterior sampling for Reinforcement learning (PSRL) to identify the best model for a new user \cite{Strens:2000:BFR:645529.658114,NIPS2013_5185}. In fact, we used a deterministic schedule PSRL (DS-PSRL) algorithm, for which we have shown how it satisfies the assumptions of our parameterization in \cite{DBLP:conf/nips/TheocharousWAV18}.   The overall solution is shown in Figure \ref{fig:passive-data-solution}.

\begin{figure*}[ht]
\centering
\includegraphics[width=1.0\textwidth]{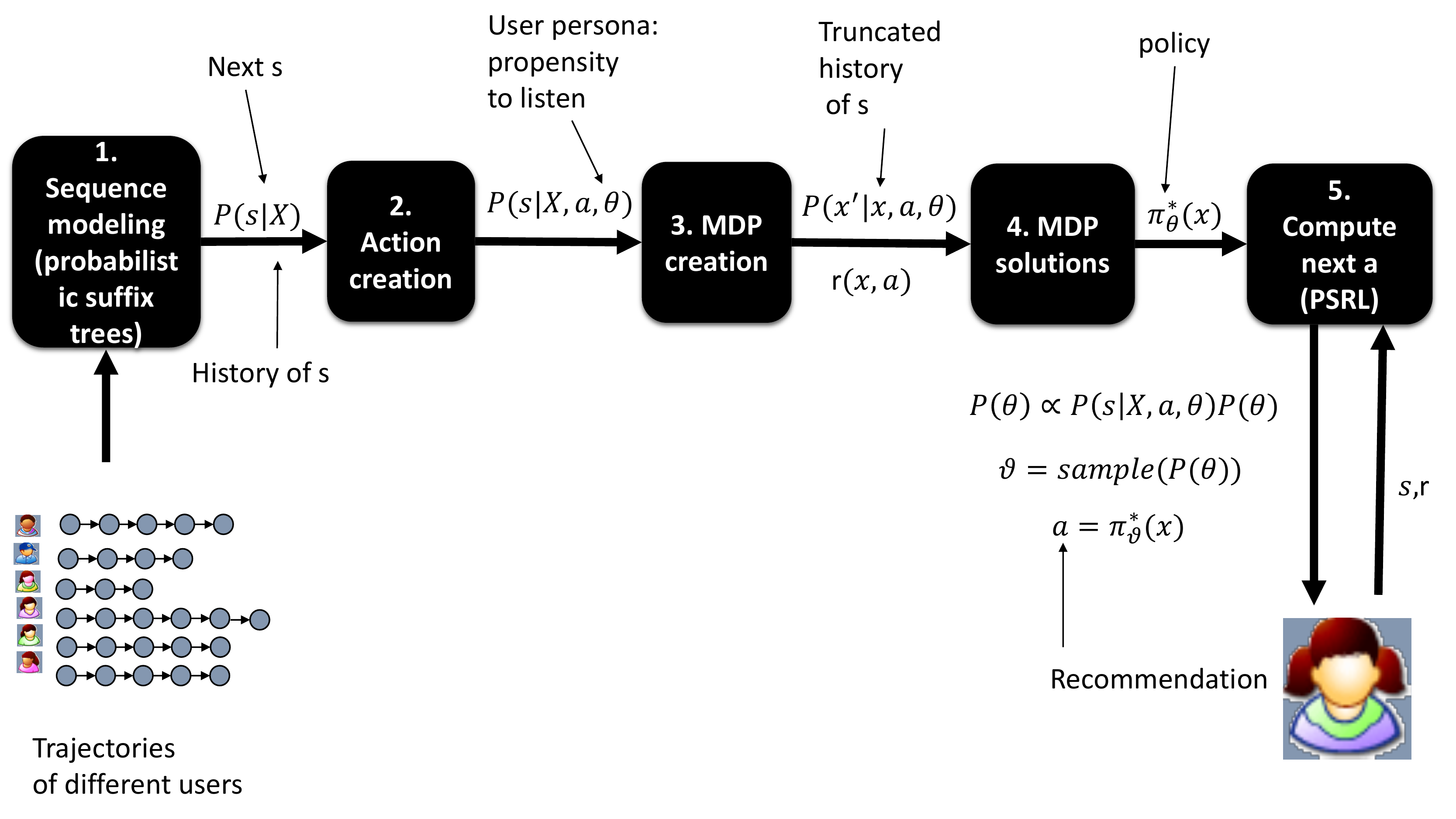}
\caption[ ]{The first 4 steps are done  offline and are used to create and solve a  discrete set of MDPs, for each value of $\theta$.  Step 5 implements the DS-PSRL algorithm of this paper.}
\label{fig:passive-data-solution}
\end{figure*}

\subsection{Sequence Modeling}  The first step in the solution is to model past sequences of activities. Due to the fact that the number of activities is usually finite and discrete and the fact that what activity  a person may do next next depends on the history of activities done, we chose to model activity sequences using probabilistic suffix trees (PST).  PSTs are a compact way of modeling the most frequent suffixes, or history of a discrete alphabet  $S$ (e.g., a set of activities).  The nodes of a PST represent suffixes (or histories).  Each node is associated with a probability distribution for observing every symbol in the alphabet, given the node suffix \citep{pst-R-package}.  Given a PST model one could easily estimate the probability of the next symbol $s=s_{t+1}$ given the history of symbols $X=(s_1, s_2 \dots s_t)$  as $P(s | X)$.    An example PST is shown in Figure \ref{fig:pst}.

\begin{figure}[h]
 \centering
  \includegraphics[width=0.5\textwidth]{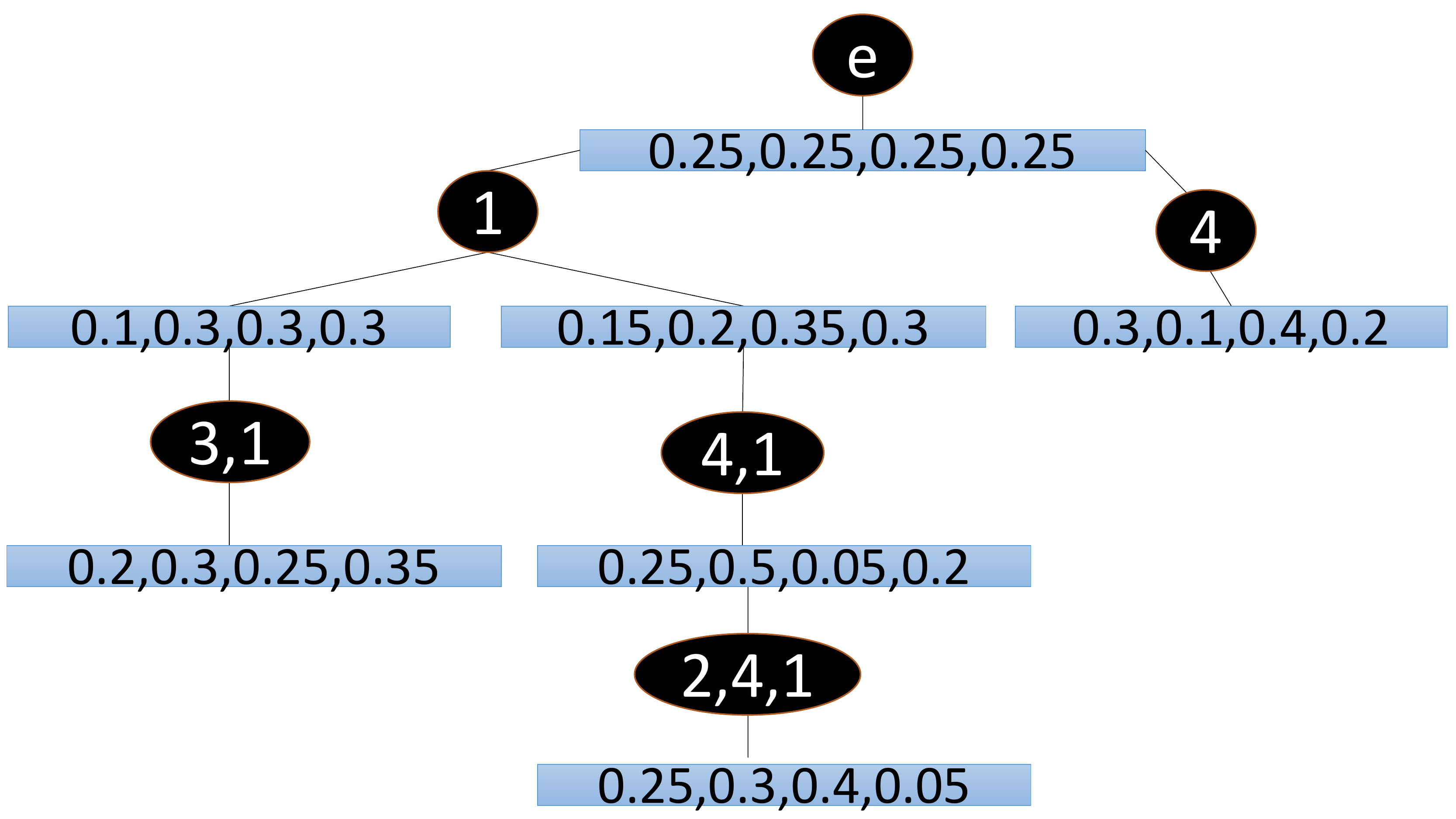}
\caption{The figure describes an example probabilistic suffix tree.  The circles represent suffixes.  In this tree, the suffixes are \{(1), (3,1), (4,1), (2,4,1), (4)\}.  The rectangles show the probability of observing the next symbol given the suffix.  The alphabet, or total number of symbols in the example are \{1,2,3,4\}}
\label{fig:pst}
\end{figure}

The log  likelihood of a set of sequences can easily be computed as $\log(\mathcal{L})=\sum_{s \in S} \log(P(s|X))$, where $s$ are all the symbols appearing in the data and $X$ maps  the longest suffixes (nodes) available in the tree for each symbol.   For our implementation we learned PSTs  using  the {\it pstree} algorithm from  \citep{pst-R-package}.  The  {\it pstree}  algorithm can take as input multiple parameters, such as the depth of the tree, the number of minimum occurrence of a suffix,  and parametrized tree pruning methods.  To select the best set of parameters we perform model selection using the Modified Akaike Information Criterion (AICc) ${\mathrm  {AICc}}=2k-2\log(\mathcal{L})+{\frac  {2k(k+1)}{n-k-1}}$, where  $\log(\mathcal{L})$ is the log likelihood  as defined earlier and $k$ is the number of  parameters \citep{akaike-1974} .

\subsection{Action Creation}  The second step involves the creation of action models for representing various personas.  An easy way to create such parameterization, is to perturb the passive dynamics  of the crowd PST (a global PST learned from all the data).  Each  perturbed model increases the probability of listening to  the recommendation by a different amount.   While there could be many functions to increase the transition probabilities, in our implementation we did it as follows: 

\begin{equation}
\label{eq:poi-dynamics}
P(s|X,a,\theta) =
\begin{cases}
    P(s|X) ^{1/ \theta},       & \text{if } a = s\\
    P(s|X)/z(\theta),                     & \text{otherwise}
\end{cases}
 \end{equation}
 where  $s$  is an activity, $X=(s_1, s_2 \dots s_t)$  a history of activities, and  $z(\theta)=\frac{\sum_{s \neq a} P(s|X)} {1-P(s=a|X) ^{1/ \theta}}$ is a normalizing factor.

\subsection{Markov Decision Processes}\label{sec:psrl-mdp}  The third step is to create MDPs from $P(s|X,a,\theta)$ and compute their policies in the fourth step.  It is straightforward to use the PST to compute an MDP model, where  the states/contexts are all the nodes of the PST. If we denote  $x$ to be a  suffix available in the tree, then we can compute the probability of transitioning from  every node to every other node by finding resulting suffixes in the tree for every additional symbol that an action can produce:
$$p(x'|x,a,\theta)=\sum_{s \in S}  \one{x'=\text{pst.suffix}(x,s)} p(s|x,a,\theta),$$ where $\text{pst.suffix}(x,s)$, is the  longest suffix in the PST of suffix $x$ concatenated with symbol $s$.
 We set the reward $r(x)=f(x,a)$, where $f$ is a function of the suffix history and the recommendation.  This gives us a finite and practically small state space.  We can use the classic {\it policy iteration} algorithm to compute the optimal policies and value functions $V_\theta^*(x) $.

\subsection{Posterior Sampling for Reinforcement Learning}\label{sec:psrl}

The fifth step is to use on-line learning to compute the true user parameters.  For this we used a posterior sampling for reinforcement learning (PSRL) algorithm called deterministic schedule PSRL (DS-PSRL) \cite{,DBLP:conf/nips/TheocharousWAV18}. The DS-PSRL algorithm shown in  Figure~\ref{alg:lazy} changes the policy in an exponentially rare fashion; if the length of the current episode is $L$, the next episode would be $2L$. This switching policy ensures that the total number of switches is $O(\log T)$.

\begin{figure}[ht]
\begin{center}
\framebox{\parbox{8cm}{
\begin{algorithmic}
\STATE {\bf Inputs}: $P_1$, the prior distribution of $\theta_*$.
\STATE $L \leftarrow 1$.

\FOR{$t\gets 1,2,\dots$}
\IF{$t  = L $}
\STATE Sample $\TTh_{t}\sim P_t$.

\STATE $L \leftarrow 2L$.

\ELSE
\STATE $\TTh_{t} \leftarrow \TTh_{t-1}$.
\ENDIF
\STATE Calculate near-optimal action $a_t \leftarrow \pi^*(x_t, \TTh_t)$.
\STATE Execute action $a_t$ and observe the new state $x_{t+1}$.
\STATE Update $P_t$ with $(x_t,a_t,x_{t+1})$ to obtain $P_{t+1}$.
\ENDFOR
\end{algorithmic}
}}
\end{center}
\caption{The DS-PSRL algorithm with deterministic schedule  of policy updates.}
\label{alg:lazy}
\end{figure}

The algorithm makes three  assumptions. First is assumes assume that MDP is weakly communicating. This is a standard assumption and under this assumption, the optimal average loss satisfies the Bellman equation. Second, it assumes that the dynamics are parametrized by a scalar parameter and satisfy a smoothness condition.

\begin{ass}[Lipschitz Dynamics]
\label{ass:lipschitz}
There exist a constant $C$ such that for any state $x$ and action $a$ and parameters $\theta,\theta'\in \Theta \subseteq \Re$,
\[
\norm{P(.|x,a,\theta) - P(.|x,a,\theta')}_1 \le C \abs{\theta-\theta'} \;.
\]
\end{ass}
Third, it makes a concentrating posterior assumption, which states that the variance of the difference between the true parameter and the sampled parameter gets smaller as more samples are gathered.
\begin{ass}[Concentrating Posterior]
\label{ass:concentrating}
Let $N_{j}$ be one plus the number of steps  in the first $j$ episodes. 
Let $\TTh_{j}$ be sampled from the posterior at the current episode $j$. Then there exists a constant $C'$ such that 
\[
\max_{j} \EE{ N_{j-1} \abs{\theta_{*} - \TTh_{j}}^2 } \le C' \log T \;.
\]
\end{ass}
The  \ref{ass:concentrating} assumption simply says the variance of  posterior decreases given more data.  In other words, we assume that the problem is learnable and not a degenerate case.    \ref{ass:concentrating}  was  actually shown to hold for two general categories of problems, finite MDPs and linearly parametrized problems with Gaussian noise \cite{Abbasi-Yadkori-Szepesvari-2015}. Under these assumptions we the following theorem can be proven~\cite{DBLP:conf/nips/TheocharousWAV18}.
\begin{thm}
\label{thm:main}
Under Assumption~\ref{ass:lipschitz} and \ref{ass:concentrating},
the regret of the DS-PSRL algorithm is bounded as 
\[
R_T = \widetilde{O}(C \sqrt{C' T}),
\]
where the $\widetilde{O}$ notation hides logarithmic factors.
\end{thm}
Notice that the regret bound in Theorem~\ref{thm:main} does not directly depend on $S$ or $A$.
Moreover, notice that the regret bound is smaller if the Lipschitz constant $C$ is smaller or the posterior concentrates faster (i.e. $C'$ is smaller).

\subsection{Satisfying the Assumptions}
Here we summarize how the parameterization assumption in Equation \ref{eq:poi-dynamics} satisfies assumptions \ref{ass:lipschitz} and \ref{ass:concentrating}.

\paragraph{Lipschitz Dynamics} We can show that the dynamics are Lipschitz continuous:
\begin{lemma}
\label{lemma:poi-lipschitz}
(Lipschitz Continuity)  Assume the dynamics are given by Equation \ref{eq:poi-dynamics}.  Then for all
 $\theta, \theta' \geq 1$ and all $X$ and $a$, we have 

 \[
 \| P(\cdot|X,a,\theta) - P(\cdot|X,a,\theta') \|_1 \leq \frac{2}{e} |\theta -\theta'|.
 \]
\end{lemma}

\paragraph{Concentrating Posterior} we can also show that Assumption~\ref{ass:concentrating} holds. Specifically, we can show that under mild technical conditions, we have
\[
\max_j \EE{ N_{j-1} \abs{\theta_{*} - \TTh_{j}}^2 } =O(1)
\]
Please refer to~\cite{DBLP:conf/nips/TheocharousWAV18} for the proofs.

\section{Optimizing for Recommendation Acceptance}
\label{sec:acceptance}
Accepting recommendations needs deeper consideration than simply predicting click through probability of an offer.  In this section we examine two acceptance factors, the `propensity to listen' and `recommendations fatigue'.   The `propensity to listen' is a byproduct of the passive data solution shown in Section~\ref{sec:passive}.  The `recommendation fatigue' is a problem where people may quickly stop paying attention to recommendations such as ads, if they are presented too often.  The property of RL algorithms for solving delayed reward problems gives a natural solution to this fundamental marketing problem.  For example, if the decision was to recommend, or not some product every day, where the final goal would be to buy at some point in time, then RL would naturally optimize the right sending schedule and thus avoid fatigue.
In this section we present experimental results for a Point-of-Interest (POI) recommendation system that solves both, the `propensity to listen' as well as `recommendation fatigue' problems \cite{DBLP:conf/iui/TheocharousVW17}.

We experimented with a points of interest domain.  For  experiments  we used the Yahoo! Flicker Creative Commons 100M (YFCC100M)~\cite{Thomee:2016:YND:2886013.2812802}, which consists of 100M Flickr photos and videos. This dataset also comprises the meta information regarding the photos, such as the date/time taken, geo-location coordinates and accuracy of these geo-location coordinates. The geo-location accuracy range from world level (least accurate) to street level (most accurate).  We used  location sequences that were mapped to POIs  near Melbourne Australia\footnote{The data and pre-processing algorithms are publicly available on https://github.com/arongdari/flickr-photo}.  After preprocessing, and removing loops, we had  7246 trajectories and 88 POIs.   

%
%
%

We trained a PST using the data and performed various experiments to test the ability of our algorithm to quickly optimize the cumulative  reward for a given user.  We used $\theta=\{1,10,20\}$ and did experiments  assuming the true user to  be any of those $\theta$.  For reward, we used a signal between $[0,1]$ indicating the frequency/desirability  of the POIs.  We computed the frequency from the data.  The action space was  a recommendation for each POI (88 POIs), plus a null action.  All actions but the null action had a cost $0.2$ of the reward.  Recommending a POI that was already seen (e.g in the current suffix) had an additional cost of $0.4$  of the reward.   This was done in order to reduce the number of recommendations otherwise called the fatigue factor.  We compared DS-PSRL  with greedy policies.  Greedy policies do not solve the underlying MDP but rather choose the action with maximum immediate reward, which is equivalent to the classic Thompson sampling for contextual bandits.  PSRL could also be thought of as Thompson sampling for MDPs.  We also compared with the optimal solution, which is the one that knows the true model from the beginning.  Our experiments are shown in Tables \ref{tab:reward} and \ref{tab:fatigue} and Figure \ref{fig:exp}.  DS-PSRL  quickly learns to optimize the average reward. At the same time it produces more reward than the greedy approach, while minimizing the fatigue factor.

\begin{table}[!htb]
\vspace{0pt}\centering%
\begin{tabular}{r r l}
& {\small \textbf{GREEDY}}
& {\small \textbf{MDP}} \\
\toprule
$*$ & 0.45 & 0.5  \\
TS  & 0.32 & 0.42 \\
\bottomrule
\\
\end{tabular}
\caption{Average reward comparisons between Thomson sampling and the optimal policy in hindsight denoted by $*$. The columns label indicate the type of policies being used.}
\label{tab:reward}
\end{table}

\begin{table}[!htb]
\vspace{0pt}\centering%
\begin{tabular}{r r l}
Time & {\small \textbf{GREEDY}}
& {\small \textbf{MDP}} \\
\toprule
1 & 72 & 72 \\
2 & 71 & 51 \\
3 & 72 &  0 \\
4 & 71 & 72 \\
5 & 72 & 51 \\
6 & 71 &  0 \\
\bottomrule
\\
\end{tabular}
\caption{The table shows the actions taken by each algorithm for the first 6 recommendations.  For cost of recommendation of 0.3.  Thompson sampling for MDPs with deterministic switching schedule (or DS-PSRL) does not give recommendation at every  step, and yet achieves higher reward.  In a way it solves  the recommendation fatigue problem.}\label{tab:fatigue}
\end{table}

\begin{figure}
\centering
\includegraphics[width=0.6\textwidth]{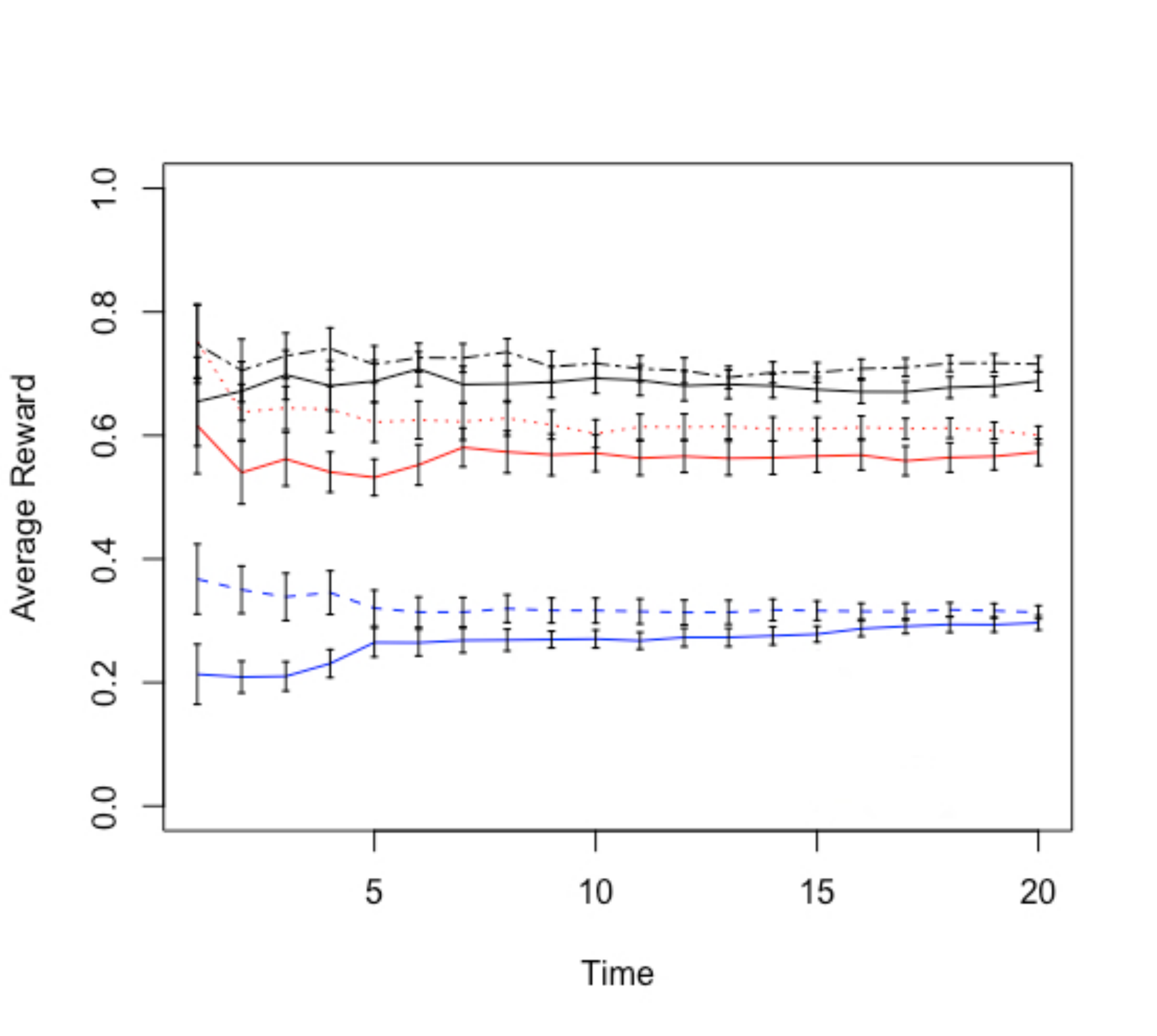}
\caption[ ]{The dotted lines is the performance assuming we knew the true different $\theta$.  DS-PSRL denoted with solid lines learns quickly for different true $\theta$.}
\label{fig:exp}
\end{figure}

\section{Capacity-aware Sequential Recommendations}
\label{sec:constraints}
So far we have considered recommendation systems that consider each user individually, ignoring the collective effects of recommendations. However, ignoring the collective effects could result in diminished utility for the user, for example through overcrowding at high-value points of interest~(POI) considered in Section~\ref{sec:acceptance}. In this section we summarize a solution that can optimize for both latent factors and resource constraints in SR systems~\cite{DBLP:conf/atal/NijsTVWS18}.

To incorporate collective effects in recommendation systems, we extend our model to a multi-agent system with global \emph{capacity constraints}, representing for example, the maximum capacity for visitors at a POI. Furthermore, to handle latent factors, we model each user as a \emph{partially observable} decision problem, where the hidden state factor represents the user's latent interests.

An optimal decision policy for this partially observable problem chooses recommendations that find the best possible trade-off between exploration and exploitation. Unfortunately, both global constraints and partial observability make finding the optimal policy intractable in general. However, we show that the structure of this problem can be exploited, through a novel belief-space sampling algorithm which bounds the size of the state space by a limit on regret incurred from switching from the partially observable model to the most likely fully observable model. We show how to decouple constraint satisfaction from sequential recommendation policies, resulting in algorithms which issue recommendations to thousands of agents while respecting constraints.

\subsection{Model of capacity-aware sequential recommendation problem}
\label{sec:constraints:model}
While PSRL~(Section~\ref{sec:psrl}) eventually converges to the optimal policy, it will never select actions which are not part of the optimal policy for any~$\text{MDP}_{\theta}$, even if this action would immediately reveal the true parameters~$\theta_{\ast}$ to the learner. In order to reason about such information gathering actions, a recommender should explicitly consider the decision-theoretic value of information~\cite{Howard1966}. To do so, we follow~\cite{Chades2012} in modeling such a hidden-model MDP as a Mixed-Observability MDP (MOMDP).

The state space of a MOMDP model factors into a fully observable factor~$x\in X$ and a
partially observable factor~$y\in Y$, each with their own transition
functions,~$T_X(x'\mid x,y,a)$ and $T_Y(y' \mid x,y,a,x')$. An observation
function~$\Omega(o \mid a, y')$ exists to inform the decision maker about transitions
of the hidden factor. However, in addition to the observations, the decision maker also
conditions his policy~$\pi(t,x,o)$ on the observable factor~$x$. Given a parametric
MDP~$\langle \Theta, S, A, \bar{R}, \bar{T}, h \rangle$ over a \emph{finite} set of 
user types~$\Theta$, for example as generated in Section~\ref{sec:psrl-mdp}, we
derive an equivalent MOMDP~$\langle X, Y, A, O, T_X, T_Y, R, \Omega, h \rangle$ having
elements
\begin{equation}%
\begin{aligned}%
X &= S,\: Y = \Theta, &
R(s,\theta,a) &= R_\theta(s,a), &
T_X(s' \mid s, \theta, a) &= T_\theta(s' \mid s, a),
\\
O &= \{ o_{\textsc{null}} \}, & 
\Omega(o_{\textsc{null}} \mid a,\theta') &= 1, &
T_Y(\theta' \mid s,\theta,a,s') &= {\begin{cases} 1 & \text{if } \theta = \theta', \\ 0 & \text{otherwise}. \end{cases}}
\\
\end{aligned}%
\end{equation}%

The model uses the latent factor $Y$ to represent the agent's type, selecting the type-specific transition and reward functions based on its instantiation. Because a user's type does not change over the plan horizon, the model is a `stationary' Mixed-Observability MDP~\cite{Martin2017a}. The observation function~$O$ is uninformative, meaning that there is no direct way to infer a user's type. Intuitively, this means the recommender can only learn a user's type by observing state transitions.

This gives a recommender model for a single user~$i$, out of a total of $n$~users. To model the global capacities at different points of interest, we employ a consumption function~$C$ and limit vector~$L$ defined over $m$~POIs. The consumption of resource type~$r$ is defined using function~$C_{r} : S \times A \rightarrow \{ 0, 1 \}$, where 1 indicates that the user is present at~$r$. The limit function~$L_r$ gives POI~$r$'s peak capacity. The optimal (joint) recommender policy satisfies the constraints \emph{in expectation}, optimizing
\begin{equation}
\max_\pi 
	\mathbb{E} \bigl[ V^{\pi} \bigr] 
	\text{, subject to } 
		\mathbb{E} \bigl[ C^{\pi}_{r,t} \bigr] \leq L_r \quad \forall t, r.
\end{equation}
For multi-agent problems of reasonable size, directly optimizing this joint policy is infeasible. For such models Column Generation~(CG;~\cite{Gilmore1961}) has proven to be an effective algorithm~\cite{deNijs2017,Walraven2018,Yost2000}. Agent planning problems are decoupled by augmenting the optimality criterion of the (single-agent) planning problem with a Lagrangian term pricing the expected resource consumption cost $\mathbb{E}[C^{\pi_i}_{r,t}]$, i.e.,
\begin{equation}\label{eq:cg:plan}
\argmax_{\pi_i}
	{\Bigl( 
		\mathbb{E}[V^{\pi_i}] - 
		\sum_{t,r} \lambda_{t,r} \mathbb{E}[C^{\pi_i}_{r,t}] 
	\Bigr)}\quad\forall i.
\end{equation}
This routine is used to compute a new policy~$\pi_i$ to be added to a set~$Z_i$ of potential policies of agent~$i$. These sets form the search space of the CG LP, which optimizes the current best \emph{joint} mix of policies subject to constraints, by solving:
\begin{equation}\label{eq:cg:solve}
\begin{aligned}
\max_{x_{i,j}}\:\:& \sum_{i=1}^n \sum_{\pi_{i,j} \in Z_i} x_{i,j}\,\mathbb{E}[V^{\pi_{i,j}}], \\
\text{s.t.}\:\:	& \sum_{i=1}^n \sum_{\pi_{i,j} \in Z_i} x_{i,j}\,\mathbb{E}[C^{\pi_{i,j}}_{r,t}] \leq L_r
				& \forall r, t, \\
				& \sum_{\mathclap{\pi_{i,j} \in Z_i}} x_{i,j} = 1,\text{ and } x_{i,j} \geq 0 &\forall i, j.\\
\end{aligned}
\end{equation}
Solving this LP results in: 1) a probability distribution over policies, having agents follow a policy with $\Pr(\pi_i = \pi_{i,j})=x_{i,j}$, and 2) a new set of dual prices~$\lambda'_{t,r}$ to use in the subsequent iteration. This routine stops once $\lambda=\lambda'$, at which point a global optimum is found. 

\subsection{Bounded belief state space planning}
Unfortunately, in every iteration of column generation, we need to find $n$ optimal policies satisfying Equation~\eqref{eq:cg:plan}, which in itself has PSPACE complexity for MOMDPs~\cite{Papadimitriou1987}. Therefore, we propose a heuristic algorithm exploiting the structure of our problems: bounded belief state space planning (Alg.~\ref{alg:boundtree}).

To plan for partially observable MDP models it is convenient to reason over belief states~\citep{Kaelbling1998}. In our case, a belief state~$b$ records a probability distribution over the possible types $\Theta$, with $b(\theta)$ indicating how likely the agent is of type~$\theta$. Given a belief state~$b$, the action taken~$a$, and the observation received~$o$, the subsequent belief state~$b'(\theta)$ can be derived using application of Bayes' theorem. In principle, this belief-state model can be used to compute the optimal policy, but the exponential size of $B$ prohibits this. Therefore, approximation algorithms generally focus on a subset of the space~$B'$.
\begin{algorithm}[tb]
\begin{algorithmic}[1]
\STATE Given parametric MDP~$\langle \Theta, S, A, \bar{R}, \bar{T}, h \rangle$ and
approximate belief space~$B'$\label{algline:sample}
\STATE Plan $\pi_j^{\ast}$ for all $j$ \label{algline:planmdp}
\STATE Compute $V_{\theta_i,\pi_j^{\ast}}$ for all $i$, $j$ \label{algline:evalmdp}
\STATE Create policy $\pi[b]$
\FOR{time $t = h \to 1$}
  \FOR{belief point $b \in B'(t)$}
	\STATE $V[b] = -\infty$
    \FOR{action $a \in A$}
		\STATE $Q[b,a] = R(b,a)$
		\FOR{observed next state $s' \in S$}
			\STATE $b' = \text{updateBelief}(b, a, s')$
			\IF{$b' \in B'$}
				\STATE $Q[b,a] = Q[b,a] + \Pr(s' \mid b, a) \cdot V[b']$
			\ELSE \label{algline:mdpex1}
				\STATE $j = \argmax_j Q\bigl[b', \pi^{\ast}_j\bigr]$
					\label{algline:mdpex2}
				\STATE $\pi[b'] = \pi_j^{\ast}$\label{ch6:algline:minregretpi}
				\STATE $Q[b,a] = Q[b,a] + \Pr(s' \mid b, a) \cdot \bar{V}\bigl[ b' \bigr]$ \label{algline:mdpex3}
			\ENDIF
		\ENDFOR
		\IF{$Q[b,a] > V[b]$}
			\STATE $V[b] = Q[b,a]$
			\STATE $\pi[b] = a$\label{ch6:algline:bestaction}
		\ENDIF
    \ENDFOR
  \ENDFOR
\ENDFOR
\STATE \Return $\langle \pi, V[b] \rangle$
\label{ch6:algline:brreturn}
\end{algorithmic}
\caption{Bounded belief state space planning~\cite{DBLP:conf/atal/NijsTVWS18}.}
\label{alg:boundtree}
\end{algorithm}

When computing a policy~$\pi$ for a truncated belief space $B'$ we have to be careful that we compute unbiased consumption expectations~$\mathbb{E}[C_\pi]$, to guarantee feasibility of the Column Generation solution. This can be achieved if we know the exact expected consumption of the policy at each `missing' belief point \emph{not} in $B'$. For corners of the belief space, where $b(\theta_i) = 1$ (and $b(\theta_j) = 0$ for $i \neq j$), the fact that agent types are stationary ensures that the optimal continuation is the optimal policy for the $\text{MDP}_{\theta_i}$. If we use the same policy in a non-corner belief, policy~$\pi^{\ast}_i$ may instead be applied on a different $\text{MDP}_{\theta_j}$, with probability~$b(\theta_j)$. In general, the expected value of choosing policy~$\pi^{\ast}_i$ in belief point~$\langle t,s,b \rangle$ is
\begin{equation}
Q\bigl[\langle t,s,b \rangle, \pi^{\ast}_i\bigr] = \sum_{j = 1}^{|\Theta|}
	\Bigl( b(\theta_j) \cdot V^{\theta_j}_{\pi^{\ast}_i}[t,s] \Bigr).
\end{equation}
For belief points close to corner~$i$, policy~$\pi^{\ast}_i$ will be the optimal policy with high probability. If we take care to construct~$B'$ such that truncated points are close to corners, we can limit our search to the optimal policies of each type,
\begin{equation}\label{eqn:ch6:approxpol}
\bar{V}\bigl[\langle t,s,b \rangle\bigr] =
	\max_{\theta_i \in \Theta}
		Q\bigl[\langle t,s,b \rangle, \pi^{\ast}_i\bigr].
\end{equation}

When we apply policy~$\pi^{\ast}_i$ in a belief point that is not a corner, we incur regret proportional to the amount of value lost from getting the type wrong. Policy~$\pi_i^{\ast}$ applied to $\text{MDP}_{\theta_j}$ obtains expected value~$V^{\theta_j}_{\pi_i^{\ast}} \leq V^{\theta_j}_{\pi_j^{\ast}}$ by definition of optimality. Thus, the use of policy $\pi_i^{\ast}$ in belief point $b$ incurs a regret of
\begin{equation}\label{eq:regret}
\textsc{regret}(b) = \min_i \textsc{regret}(b, i) = \min_i \sum_{j = 1}^{|\Theta|}
	\biggl( b(\theta_j) \cdot
			\Bigl( V^{\theta_j}_{\pi_j^{\ast}} - V^{\theta_j}_{\pi_i^{\ast}} \Bigr)
	\biggr).
\end{equation}
This regret function can serve as a scoring rule for belief points worth considering in belief space $B'$. Let~$\mathrm{P}(b)$ stand for the probability of belief point~$b$, then we generate all subsequent belief points from initial belief~$b_0$ that meet a threshold (for hyper-parameters minimum probability~$p$ and shape~$\alpha$):
\begin{equation}\label{eq:boundedregret}
b \in B' \text{ if } \textsc{regret}(b) >
	\bigl( e^{-\alpha(\mathrm{P}(b)-p)} - e^{-\alpha(1-p)} \bigr) 
		\cdot \textsc{regret}(b_0).
\end{equation}

Algorithm~\ref{alg:boundtree} starts by computing the optimal MDP policy~$\pi_j^{\ast}$ for each type~$\theta_j$, followed by determining the exact expected values~$V^{\theta_i}_{\pi_j^{\ast}}$ of applying these policies to all different user types $\theta_i$. The remainder of the algorithm computes expected values at each belief point in regret-truncated space~$B'$, according to the typical dynamic programming recursion. However, in case of a missing point~$b'$, the best policy~$\pi_{j}^{\ast}$ is instead selected (line~\ref{algline:mdpex2}), and the expected value of using this MDP policy is computed according to the belief state. The resulting policy~$\pi$ thus consists of two stages: the maximally valued action stored in~$\pi[b]$ is selected, unless $b \notin B'$, at which point MDP policy~$\pi_j^{\ast}$ replaces~$\pi$ for the remaining steps.

\subsection{Empirical evaluation of scalability versus quality}
By bounding the exponential growth of the state space, Algorithm~\ref{alg:boundtree} trades off solution quality for scalability. To assess this trade-off, we perform an experiment on the POI recommendation problem introduced in Section~\ref{sec:acceptance}. We compare with the highly scalable PSRL on the one hand, and state-of-the-art mixed-observability MDP planner SARSOP~\cite{Kurniawata2008} on the other. We consider a problem consisting of 5~POIs, 3~user types, 50~users and PST depth~1. For this experiment we measure the quality of the computed policy as the mean over 1,000~simulations per instance, solving~$5$~instances per setting of the horizon. We consider two settings, the regular single recommendation case, and a dual recommendation case where the recommender is allowed to give an alternative to the main recommendation, which may provide more opportunities to gather information in each step.

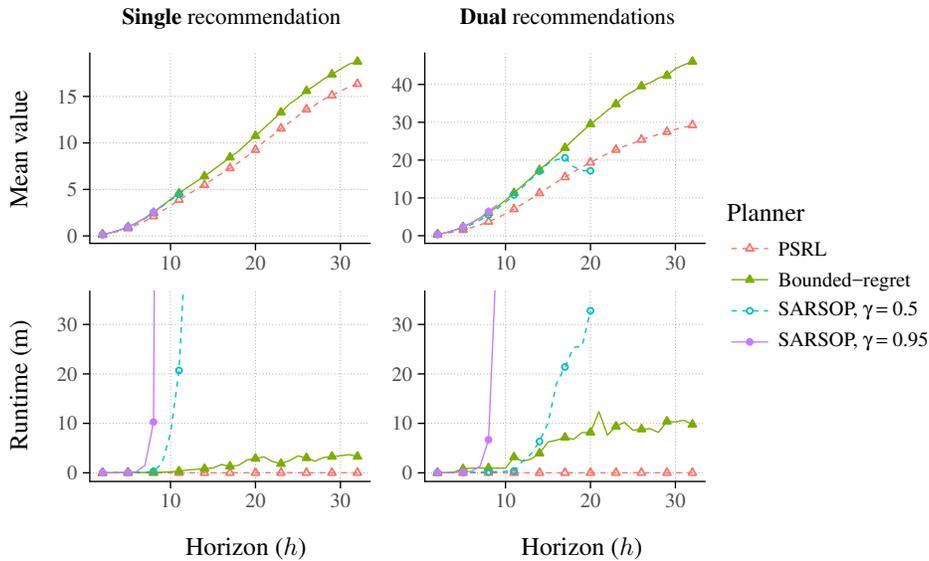
\begin{figure}[hb]
\centering
\begin{tikzpicture}[scale=0.96,transform shape,every node/.style={inner sep=0}]
\node at (0,0) {\includegraphics{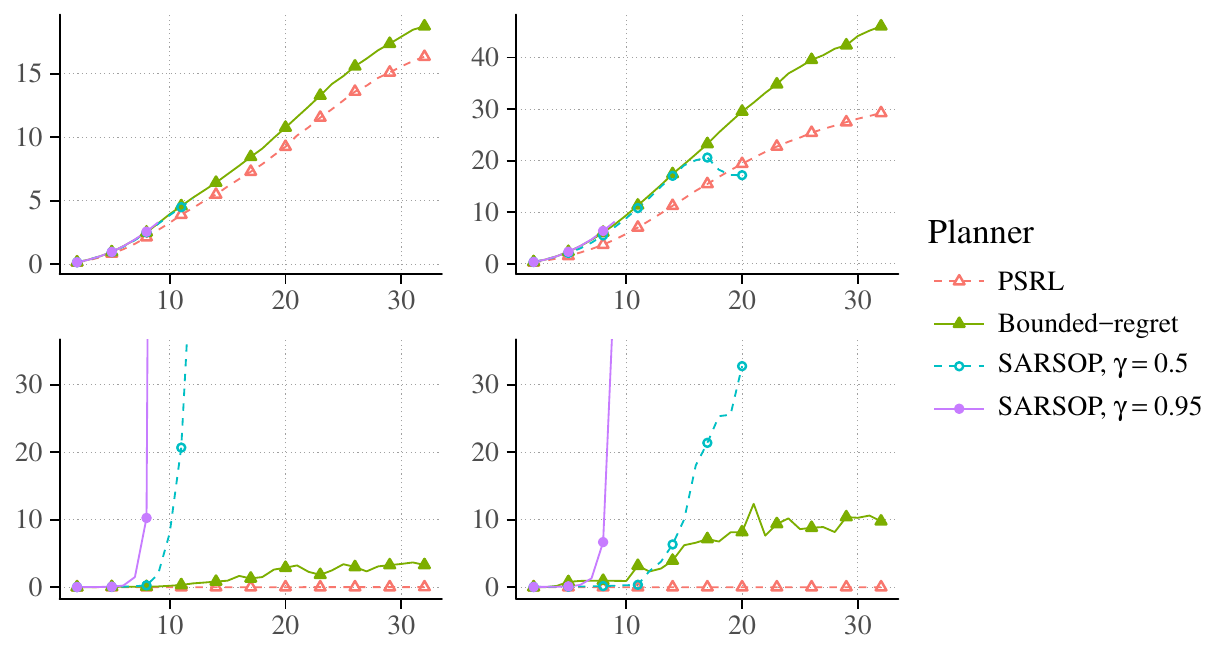}};

\node[anchor=base] at ( -3.55, 3.55)  {\small\textbf{Single} recommendation};
\node[anchor=base] at (  1.10, 3.55)  {\small\textbf{Dual} recommendations};

\node[anchor=base] at ( -3.35, -3.8)  {Horizon ($h$)};
\node[anchor=base] at (  1.30, -3.8)  {Horizon ($h$)};

\node[anchor=base,rotate=90] at ( -6.38,  1.82)  {Mean value};
\node[anchor=base,rotate=90] at ( -6.38, -1.45)  {Runtime (m)};
\end{tikzpicture}
\caption{Solution quality and planning time of the different sequential recommendation
planners, as a function of the horizon.}
\label{fig:performance}
\end{figure}
Figure~\ref{fig:performance} presents the results. The top row presents the observed mean reward, while the bottom row presents the required planning time in minutes. We observe that for our constrained finite-horizon problems, SARSOP quickly becomes intractable, even when the discount factor is set very low. However, by not optimizing for information value, PSRL obtains significantly lower lifetime value. Our algorithm finds policies which do maximize information value, while at the same time remaining tractable through its effective bounding condition on the state space growth. We note that its runtime stops increasing significantly beyond~$h=20$, as a result of the bounded growth of the state space.

\section{Large Action Spaces}
\label{sec:large-actions}
In many real-world recommendation systems the number of actions could be prohibitively large.  Netflix for example employs  a few thousands  of movie recommendations.   For SR systems the difficulty is even more severe, since the search space grows exponentially with the planning horizon.  In this section we show how to learn action embeddings for action generalization.  Most model-free reinforcement learning methods leverage state representations (embeddings) for generalization, but either ignore structure in the space of actions or assume the structure is provided \emph{a priori}. We show how a policy can be decomposed into a component that acts in a low-dimensional space of action representations and a component that transforms these representations into actual actions. These representations improve generalization over large, finite action sets by allowing the agent to infer the outcomes of actions similar to actions already taken. We provide an algorithm to both learn and use action representations and provide conditions for its convergence. The efficacy of the proposed method is demonstrated on large-scale real-world problems 
\cite{DBLP:conf/icml/ChandakTKJT19}.

\begin{figure}[h]
    \centering
	\includegraphics[scale=0.17]{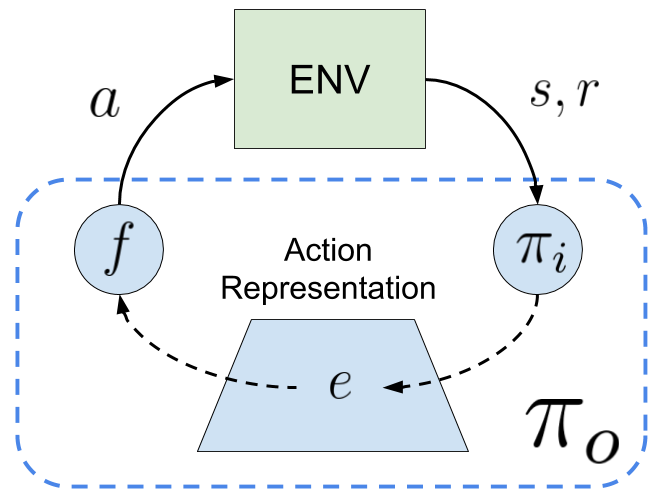}
	\quad\quad\quad\quad\quad
\includegraphics[width=0.45\textwidth]{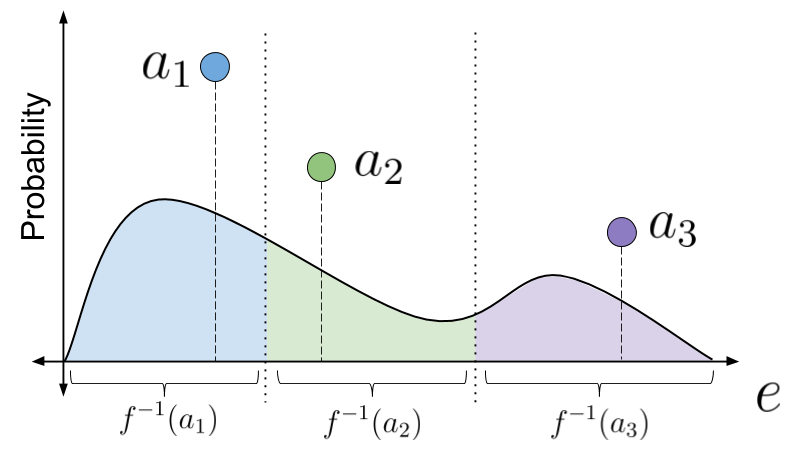}

	\caption{ 
	(Left) The structure of the proposed overall policy, $\pi_o$, consisting of $f$ and $\pi_i$, that learns action representations to generalize over large action sets. (Right) Illustration of the  probability induced for three actions by the probability density of $\pi_i(e|s)$ on a $1$-D embedding space. The $x$-axis represents the embedding, $e$, and the $y$-axis represents the probability. The colored regions represent the mapping $a=f(e)$, where each color is associated with a specific action.
	}
	\label{Fig:execution}
\end{figure}

    \subsection{Generalization over Actions}

The benefits of capturing the structure in the underlying state space of MDPs is a well understood and a widely used concept in RL.
State representations allow the policy to generalize across states.
Similarly, there often exists additional structure in the space of actions that can be leveraged.
We hypothesize that exploiting this structure can enable quick generalization across actions, thereby making learning with large action sets feasible. 
To bridge the gap, we introduce an action representation space, $\mathcal{E} \subseteq \mathbb{R}^{ d}$, and consider a factorized policy, $\pi_o$, parameterized by an embedding-to-action mapping function, $f \colon \mathcal{E} \to \mathcal{A}$, and an internal policy, $\pi_i \colon \mathcal{S} \times \mathcal{E} \to [0,1]$, such that the distribution of $A_t$ given $S_t$ is characterized by:
\begin{equation}
        E_t \sim \pi_i(\cdot |S_t), \hspace{2cm}  A_t = f(E_t). \label{eqn:decomposed-policy}
\end{equation}
Here, $\pi_i$ is used to sample $E_t \in \mathcal{E}$, and the function $f$ deterministically maps this representation to an action in the set $\mathcal{A}$.
Both these components together form an \textit{overall policy}, $\pi_o$.
Figure \ref{Fig:execution} (Right) illustrates the probability of each action under such a parameterization.
With a slight abuse of notation, we use $f^{-1}(a)$ as a one-to-many function that denotes the set of representations  that are mapped to the action $a$ by the function $f$, i.e., $f^{-1}(a) \coloneqq \{e\in\mathcal E:f(e)=a\}$. 
%

%
In the following sections we discuss the existence of an optimal policy $\pi_o^*$ and the learning procedure for $\pi_o$.
To elucidate the steps involved, we split it into four parts. 
First, we show that there exist $f$ and $\pi_i$ such that $\pi_o$ is an optimal policy.
Then we present the supervised learning process for the function $f$ when $\pi_i$ is fixed.
Next we give the policy gradient learning process for $\pi_i$ when $f$ is fixed.
Finally, we combine these methods to learn $f$ and $\pi_i$ simultaneously.

\subsection{Existence of $\pi_i$ and $f$ to Represent an Optimal Policy} 
   
%
In this section, we aim to establish a condition under which $\pi_o$ can represent an optimal policy. 
Consequently, we then define the optimal set of $\pi_o$ and $\pi_i$ using the proposed parameterization.    
To establish the main results we begin with the necessary assumptions.

%
The characteristics of the actions can be naturally associated with how they influence state transitions.
In order to learn a representation for actions that captures this structure, we consider a standard Markov property, often used for learning probabilistic graphical models \cite{ghahramani2001introduction}, and make the following assumption that the transition information can be sufficiently encoded to infer the action that was executed. 
\begin{ass}
\label{ass:A1}
Given an embedding $E_t$, $A_t$ is conditionally independent of $S_t$ and $S_{t+1}$:
{\small
$$
    P(A_t|S_t,S_{t+1}) =\!\!
    \int_{\mathcal{E}} \!\!\!P(A_t|E_t=e) P(E_t=e|S_t,S_{t+1})\,\mathrm{d}e.
$$
}
%
\end{ass}
%
%
%
\begin{ass}
\label{ass:A2}
Given the embedding $E_t$ the action, $A_t$ is deterministic and is represented by a function $f:\mathcal E \to \mathcal A$, i.e., $\exists a \text{ such that } P(A_t=a|E_t=e)=1$.
\end{ass}   
%

%
        %
       %
       
        We now establish a necessary condition under which our proposed policy can represent an optimal policy. 
        This condition will also be useful later when deriving learning rules.
    	\begin{lemma}  
    	    \label{lemma:bellman}
    	    Under Assumptions \eqref{ass:A1}--\eqref{ass:A2}, which defines a function $f$, for all $\pi$, there exists a $\pi_i$ such that
            \begin{align}
                v^\pi(s) = \sum_{a \in \mathcal{A}} \int_{f^{-1}(a)} \pi_i(e|s) q^\pi(s, a)\, \mathrm{d}e. \label{eqn:lemma-1}
            \end{align}
        \end{lemma}
         The proof is available in \cite{DBLP:conf/icml/ChandakTKJT19}.
         Following Lemma \eqref{lemma:bellman}, we use $\pi_i$ and $f$ to define the overall policy as
        \begin{align}
            \pi_o(a|s) &\coloneqq \int_{f^{-1}(a)}\pi_i(e|s)\,\mathrm{d}e.
            \label{eqn:optimal-policy}
        \end{align}

        \begin{thm} Under Assumptions \eqref{ass:A1}--\eqref{ass:A2}, which defines a function $f$,
        there exists an overall policy, $\pi_o$, such that $v^{\pi_o}=v^{\star}$.   \label{thm:optimal-overall-policy}
        \end{thm}

        \begin{proof}
        This follows directly from Lemma \ref{lemma:bellman}.
        Because the state and action sets are finite, the rewards are bounded, and $\gamma \in [0,1)$, there exists at least one optimal policy. 
        For any optimal policy $\pi^\star$, the corresponding state-value and  state-action-value functions are the unique $v^\star$ and $q^\star$, respectively.
       By Lemma \ref{lemma:bellman} there exist $f$ and $\pi_i$ such that 
       \begin{align}
           v^\star(s)&=  \sum_{a \in \mathcal{A}}  \int_{f^{-1}(a)}\pi_i(e|s) q^\star(s,a)\,\mathrm{d}e. \label{eqn:optimal-overall-policy}
       \end{align}
   Therefore, there exists $\pi_i$ and $f$, such that the resulting $\pi_o$ has the state-value function $v^{\pi_o}=v^{\star}$, and hence it represents an optimal policy.
        \end{proof}
        Note that Theorem \ref{thm:optimal-overall-policy} establishes existence of an optimal overall policy based on equivalence of the state-value function, but does \emph{not} ensure that all optimal policies can be represented by an overall policy. 
        Using \eqref{eqn:optimal-overall-policy}, we define $\Pi_o^\star \coloneqq \{\pi_o : v^{\pi_o}=v^\star\}$.
        Correspondingly, we define the set of \textit{optimal internal policies} as $\Pi_i^\star \coloneqq \{\pi_i : \exists \pi_o^\star \in \Pi_o^\star,\exists f, \pi_o^\star(a|s) = \int_{f^{-1}(a)}\pi_i(e|s)\,\mathrm{d}e \}$. 
        %
       %
       %
        %
        \subsection{Supervised Learning of $f$ for a Fixed $\pi_i$}    
        \label{section:learn-f}    
        %
        
        Theorem \ref{thm:optimal-overall-policy} shows that there exist $\pi_i$ and a function $f$, which helps in predicting the action responsible for the transition from $S_t$ to $S_{t+1}$, such  that the corresponding overall policy is optimal.  
        %
        %
        However, such a function, $f$, may not be known \emph{a priori}.
        In this section, we present a method to estimate $f$ using data collected from visits with the environment.

        By Assumptions \eqref{ass:A1}--\eqref{ass:A2}, $P(A_t|S_t,S_{t+1})$ can be written in terms of $f$ and $P(E_t|S_t,S_{t+1})$. 
        We propose searching for an estimator, $\hat f$, of $f$ and an estimator, $\hat g(E_t|S_t,S_{t+1})$, of $P(E_t|S_t,S_{t+1})$ such that a reconstruction of $P(A_t|S_t,S_{t+1})$ is accurate. 
        Let this estimate of $P(A_t|S_t,S_{t+1})$ based on $\hat f$ and $\hat g$ be
        {
        \begin{equation}
            \small
            \hat P(A_t|S_t,S_{t+1})  =  \int_\mathcal{E} \!\! \hat f (A_t|E_t\!=\!e) \hat g(E_t\!=\!e|S_t,S_{t+1})\,\mbox{d}e \label{eqn:action-rep-estimator}
        \end{equation}\textnormal{} 
        }
        %
        %
        One way to measure the difference between $P(A_t|S_t,S_{t+1})$ and $\hat P(A_t|S_t,S_{t+1})$ is using the expected (over states coming from the on-policy distribution) Kullback-Leibler (KL) divergence
        
        \begin{align}
             \text{KL}(P(A_t|S_t,S_{t+1}) || \hat P(A_t|S_t,S_{t+1}))=& -\mathbf{E} \left [\sum_{a \in \mathcal{A}} P(a|S_t,S_{t+1}) \ln \left ( \frac{\hat P(a|S_t,S_{t+1})}{P(a|S_t,S_{t+1}) } \right ) \right ]
            \\
            =& -\mathbf{E} \left [ \ln \left ( \frac{\hat P(A_t|S_t,S_{t+1})}{P(A_t|S_t,S_{t+1})} \right )  \right ]. \label{eqn:KL-sample}
        \end{align}
        %
        
        Since the observed transition tuples, $(S_t,A_t,S_{t+1})$, contain the action responsible for the given $S_t$ to $S_{t+1}$ transition,
        %
        an on-policy sample estimate of the KL-divergence 
        can be computed readily using \eqref{eqn:KL-sample}.
        We adopt the following loss function based on the KL divergence between $P(A_t|S_t,S_{t+1})$ and $\hat P(A_t|S_t,S_{t+1})$:
        \begin{align}
            \mathcal{L}(\hat f, \hat g) &= - \mathbf{E}\left [ \ln \left (\hat P(A_t|S_t,S_{t+1}) \right )\right ], 
            \label{Eqn:self-supervised-loss}
        \end{align}
        where the denominator in \eqref{eqn:KL-sample} is not included in \eqref{Eqn:self-supervised-loss} because it does not depend on $\hat f$ or $\hat g$. 
        If $\hat f$ and $\hat g$ are parameterized, their parameters can be learned by minimizing the loss function, $\mathcal{L}$, using a supervised learning procedure.
           	\begin{figure}[t]
    		\centering
    		\includegraphics[width=0.65\textwidth]{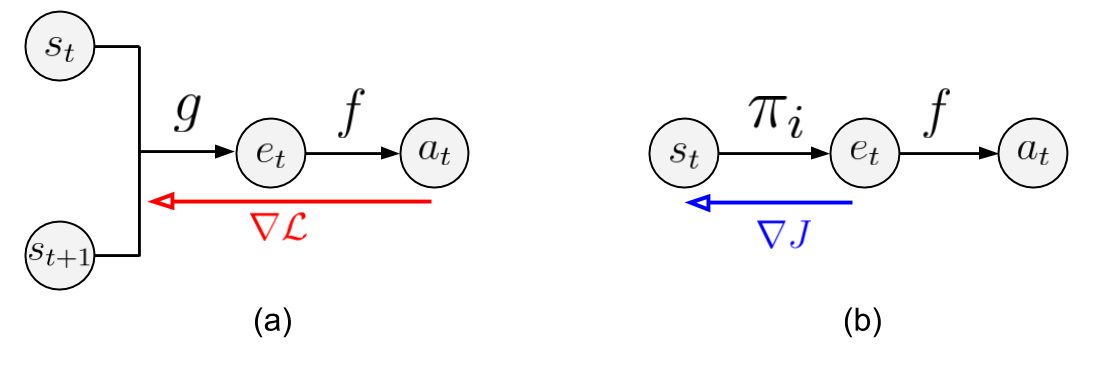}
    		\caption{%
    		(a) Given a state transition tuple, functions $g$ and $f$ are used to estimate the action taken. 
    		The red arrow denotes the gradients of the supervised loss \eqref{Eqn:self-supervised-loss} for learning the parameters of these functions.
    		%
    		(b) During execution, an internal policy, $\pi_i$, can be used to first select an action representation, $e$. 
    		The function $f$, obtained from previous learning procedure, then transforms this representation to an action.
    		The blue arrow represents the internal policy gradients \eqref{Eqn:internal_gradient} obtained using Lemma \ref{prop:local-policy-gradient} to update $\pi_i$.
    		}
    		\label{Fig:architecture-graph}
    	\end{figure}

        A computational graph for this model is shown in Figure \ref{Fig:architecture-graph}. 
        %
        %
        Note that, while $\hat f$ will be used for $f$ in an overall policy, $\hat g$ is only used to find $\hat f$, and will not serve an additional purpose.

        As this supervised learning process only requires estimating $P(A_t|S_t,S_{t+1})$, 
        it does not require (or depend on) the rewards. 
        This partially mitigates the problems due to sparse and stochastic rewards, since an alternative informative supervised signal is always available.
        This is advantageous for making the action representation component of the overall policy learn quickly and with low variance updates.

       \subsection{Learning $\pi_i$ For a Fixed $f$}
        \label{section:learn-internal-policy}
    	A common method for learning a policy parameterized with weights $\theta$ is to optimize the discounted start-state objective function, 
    	$ 
    	    J(\theta) := \sum_{s \in \mathcal{S}} d_0(s) v^\pi(s). 
    	$ 
    	For a policy with weights $\theta$, the expected performance of the policy can be improved by ascending the \emph{policy gradient}, $\frac{\partial J(\theta)}{\partial \theta}$. 

    	Let the state-value function associated with the internal policy, $\pi_i$, be $v^{\pi_i}(s) =  \mathbf{E}[\sum_{t=0}^{\infty}\gamma^tR_{t} |s, \pi_i, f]$, and the state-action value function $q^{\pi_i}(s,e) = \mathbf{E}[\sum_{t=0}^{\infty}\gamma^t R_{t}$ $ |s, e, \pi_i, f]$.
    	 We then define the performance function for $\pi_i$ as:
        \begin{align}
    	    J_i(\theta) := \sum_{s \in \mathcal{S}} d_0(s) v^{\pi_i}(s) \label{Eqn:internal-Performance-function}.
    	\end{align}
    	Viewing the embeddings as the action for the agent with policy $\pi_i$, the policy gradient theorem \cite{sutton2000policy}, states that the unbiased \cite{thomas2014bias} gradient of \eqref{Eqn:internal-Performance-function} is,
        \begin{align}
        	\frac{\partial J_i(\theta)}{\partial \theta} = \sum_{t=0}^{\infty}\mathbf{E}\left [  \gamma^t \int_\mathcal{E} q^{\pi_i}(S_t, e) \frac{\partial}{\partial \theta} \pi_i(e|S_t)  \, \mathrm{d}e\right ],
    	    \label{Eqn:internal_gradient}
    	\end{align}
    	where, the expectation is over states from $d^\pi$, as defined in \cite{sutton2000policy} (which is not a true distribution, since it is not normalized). 
    	The parameters of the internal policy can be learned by iteratively updating its parameters in the direction of $\partial J_i(\theta) /\partial \theta$.
    	Since there are no special constraints on the policy $\pi_i$, any policy gradient algorithm designed for continuous control, like DPG \cite{silver2014deterministic}, PPO \cite{schulman2017proximal}, NAC \cite{bhatnagar2009natural} etc., can be used out-of-the-box.

    	However, note that the performance function associated with the overall policy, $\pi_o$ (consisting of function $f$ and the internal policy parameterized with weights $\theta$), is:
    	\begin{align}
    	    J_o(\theta,f) = \sum_{s \in \mathcal{S}} d_0(s) v^{\pi_o}(s) \label{Eqn:Performance-function}.
    	\end{align}
    	The ultimate requirement is the improvement of this overall performance function, $J_o(\theta,f)$, and not just $J_i(\theta)$.
        So, how useful is it to update the internal policy, $\pi_i$, by following the gradient of its own performance function? The following lemma answers this question. 
    
        \begin{lemma}
         For all deterministic functions, $f$, which map each point, $e \in \mathbb{R}^{ d}$, in the representation space to an action, $a \in \mathcal{A}$, the expected updates to $\theta$ based on $\frac{\partial J_i(\theta)}{\partial \theta}$ are equivalent to updates based on $\frac{\partial J_o(\theta,f)}{\partial \theta}$. 
    	That is,
    	\begin{align*}
        	 \frac{\partial J_o(\theta,f)}{\partial \theta} = \frac{\partial J_i(\theta)}{\partial \theta}.
    	\end{align*}
    	\label{prop:local-policy-gradient}
    	\end{lemma}

    	The proof is available in \cite{DBLP:conf/icml/ChandakTKJT19}.
        The chosen parameterization for the policy has this special property, which allows $\pi_i$ to be learned using its internal policy gradient.
       	Since this gradient update does not require computing the value of any $\pi_o(a|s)$ explicitly,  
       	the potentially intractable computation of $f^{-1}$ in \eqref{eqn:optimal-policy} required for $\pi_o$ can be avoided.
       	Instead, $\partial J_i(\theta) / \partial \theta$ can be used directly to update the parameters of the internal policy while still optimizing the overall policy's performance, $J_o(\theta,f)$.
           
        \subsection{Learning $\pi_i$ and $f$ Simultaneously} 
        \label{section:learn-simultaneously}
        Since the supervised learning procedure for $f$ does not require rewards, a few initial trajectories can contain enough information to begin learning a useful action representation. 
        As more data becomes available it can be used for fine-tuning and improving the action representations.
        %
	\subsubsection{Algorithm}
		We call our algorithm \textbf{p}olicy \textbf{g}radients with \textbf{r}epresentations for \textbf{a}ctions (PG-RA). 
		%
		%
		%
		PG-RA first initializes the parameters in the action representation component by sampling a few trajectories using a random policy and using the supervised loss defined in \eqref{Eqn:self-supervised-loss}.
		If additional information is known about the actions, as assumed in prior work \cite{dulac2015deep}, it can also be considered when initializing the action representations. 
		Optionally, once these action representations are initialized, they can be kept fixed.

		In the Algorithm \ref{Alg:1}, Lines $2$-$9$ illustrate the online update procedure for all of the parameters involved. 
		Each time step in the episode is represented by $t$.
		For each step, an action representation is sampled and is then mapped to an action by $\hat f$. 
		Having executed this action in the environment, the observed reward is then used to update the internal policy, $\pi_i$, using \textit{any} policy gradient algorithm.
		Depending on the policy gradient algorithm, if a critic is used then semi-gradients of the TD-error are used to update the parameters of the critic.
		In other cases, like in REINFORCE \cite{williams1992simple} where there is no critic, this step can be ignored.
		The observed transition is then used in Line $9$ to update the parameters of $\hat f$ and $\hat g$ so as to minimize the supervised learning loss \eqref{Eqn:self-supervised-loss}. 
		In our experiments, Line $9$ uses a stochastic gradient update. 

	\IncMargin{1em}
	\begin{algorithm2e}[t]
		Initialize action representations \\
		\For {$episode = 0,1,2...$}{
			\For {$t = 0,1,2...$} {
			    Sample action embedding, $E_t$, from $\pi_i(\cdot|S_t) $ \\
			    $A_t = \hat f(E_t)$\\
			    Execute $A_t$ and observe $S_{t+1}, R_{t}$ \\
			    Update $\pi_i$ using \textit{any} policy gradient algorithm\\ 
	            Update critic (if any) to minimize TD error\\
	            Update $\hat f$ and $\hat g$ to minimize $\mathcal L$ defined in \eqref{Eqn:self-supervised-loss} 

			}
		}
		\caption{Policy Gradient with Representations for Action (PG-RA)}
		\label{Alg:1}  
	\end{algorithm2e}
	\DecMargin{1em}    	
 
    \subsubsection{PG-RA Convergence}
     If the action representations are held fixed while learning the internal policy, then as a consequence of Lemma \ref{prop:local-policy-gradient}, convergence of our algorithm directly follows from previous two-timescale results \cite{borkar1997actor,bhatnagar2009natural}.
 	Here we show that learning both $\pi_i$ and $f$ simultaneously using our PG-RA algorithm can also be shown to converge by using a three-timescale analysis.

    Similar to prior work \cite{bhatnagar2009natural,degris2012off,konda2000actor}, for analysis of the updates to the parameters, $\theta \in \mathbb{R}^{d_\theta}$, of the internal policy, $\pi_i$, we use a projection operator $\Gamma : \mathbb{R}^{d_\theta} \rightarrow \mathbb{R}^{d_\theta}$ that projects any $x \in \mathbb{R}^{d_\theta}$ to a compact set $\mathcal{C}\subset \mathbb R^{d_\theta}$.
 	We then define an associated  vector field operator, $\hat \Gamma$,
    that projects any gradients leading outside the compact region,  $\mathcal{C}$, back to $\mathcal{C}$.
 	Practically, however, we do not project the iterates to a constraint region as they are seen to remain bounded (without projection).
   Formally, we make the following assumptions,
   	\begin{ass}
	    \label{ass:differentiable}
	    For any state action-representation pair (s,e), internal policy, $\pi_i(e|s)$, is continuously differentiable in the parameter  $\theta$. 
	\end{ass}
	\begin{ass}
	\label{ass:projection}
	 The updates to the parameters, $\theta \in \mathbb{R}^{d_\theta}$, of the internal policy, $\pi_i$, includes a projection operator $\Gamma : \mathbb{R}^{d_\theta} \rightarrow \mathbb{R}^{d_\theta}$ that projects any $x \in \mathbb{R}^{d_\theta}$ to a compact set $\mathcal{C} = \{x|c_i(x) \leq 0, i=1,...,n\} \subset \mathbb{R}^{d_\theta}$, where $c_i(\cdot), i=1,...,n$ are real-valued, continuously differentiable functions on $\mathbb{R}^{d_\theta}$ that represents the constraints specifying the compact region. For each $x$ on the boundary of $\mathcal C$, the gradients of the active $c_i$ are considered to be linearly independent.  
	\end{ass}
	\begin{ass}
        \label{ass:param-bounded}
        The iterates $\omega_t$ and $\phi_t$ satisfy  $\underset{t}{\mathrm{sup}} \ (|| \omega_t||) < \infty$ and $\underset{t}{\mathrm{sup}} \ (|| \phi_t||) < \infty$.
	\end{ass}

   \begin{thm}
    	\label{thm:convergence}
    	  Under Assumptions \eqref{ass:A1}--\eqref{ass:param-bounded}, the internal policy parameters   $\theta_t$, converge to $\mathcal{\hat Z} = \left\{x \in \mathcal{C}|\hat \Gamma\left(\frac{\partial J_i(x)}{\partial \theta}\right)=0\right \}$ as $t \rightarrow \infty$, with probability one.
	    \end{thm}
%
  
%
%
	    \begin{proof} (Outline) We consider three learning rate sequences, such that the update recursion for the internal policy is on the slowest timescale, the critic's update recursion is on the fastest, and the action representation module's has an intermediate rate. 
	    With this construction, we leverage the three-timescale analysis technique \cite{borkar2009stochastic} and prove convergence.
	    The complete proof is available in \cite{DBLP:conf/icml/ChandakTKJT19}.
	    \end{proof}
 
\subsection{Experimental Analysis}
        We evaluate our proposed algorithms on the following domains.
        \paragraph{Maze: } As a proof-of-concept, we constructed a continuous-state maze environment where the state comprised of the coordinates of the agent's current location. 
        The agent has $n$ equally spaced actuators (each actuator moves the agent in the direction the actuator is pointing towards) around it, and it can choose whether each actuator should be on or off. 
        Therefore, the size of the action set is exponential in the number of actuators, that is $|\mathcal{A}| = 2^n$. 
        The net outcome of an action is the vectorial summation of the displacements associated with the selected actuators.
        The agent is rewarded with a small penalty for each time step, and a reward of $100$ is given upon reaching the goal position. 
        To make the problem more challenging, random noise was added to the action $10\%$ of the time and the maximum episode length was $150$ steps. 
        %
    	%
    	
    	This environment is a useful test bed as it requires solving a long horizon task in an MDP with a large action set and a single goal reward.
    	Further, we know the Cartesian representation for each of the actions, and can thereby use it to visualize the learned representation, as shown in Figure \ref{Fig:emb}.
    	%
    	%

    	\paragraph{Real-word recommender systems: } 
    	We consider two real-world applications of recommender systems that require decision making over \textit{multiple time steps}.

        First, a web-based video-tutorial platform, which has a recommendation engine that suggests a series of tutorial videos on various software.
    	    The aim is to meaningfully engage the users in learning how to use these software and convert novice users into experts in their respective areas of interest.
	    The tutorial suggestion at each time step is made from a large pool of available tutorial videos on several software.

        The second application is a professional multi-media editing software. 
	    Modern multimedia editing software often contain many tools that can be used to manipulate the media, and this wealth of options can be overwhelming for users. 
	    In this domain, an agent suggests which of the available tools the user may want to use next.
	    The objective is to increase user productivity and assist in achieving their end goal. 

        For both of these applications, an existing log of user's click stream data was used to create an $n$-gram based MDP model for user behavior \cite{shani2005mdp}.
        In the tutorial recommendation task, user activity for a three month period was observed. 
        Sequences of user visit were aggregated to obtain over $29$ million clicks.   
        Similarly, for a month long duration, sequential usage patterns of the tools in the multi-media editing software were collected to obtain a total of over $1.75$ billion user clicks.  
        Tutorials and tools that had less than $100$ clicks in total were discarded. 
        The remaining $1498$ tutorials and $1843$ tools for the web-based tutorial platform and the multi-media software, respectively, were used to create the action set for the MDP model.
        The MDP had continuous state-space, where each state consisted of the feature descriptors associated with each item (tutorial or tool) in the current $n$-gram. 
        Rewards were chosen based on a surrogate measure for difficulty level of tutorials and popularity of final outcomes of user visits in the multi-media editing software, respectively. 
        Since such data is sparse, only $5\%$ of the items had rewards associated with them, and the maximum reward for any item was $100$.

        Often the problem of recommendation is formulated as a contextual bandit or collaborative filtering problem, but as shown in \cite{TheocharousTG15} these approaches fail to capture the long term value of the prediction.
        Solving this problem for a longer time horizon with a large number of  actions (tutorials/tools) makes this real-life problem a useful and a challenging domain for RL algorithms.
\subsection{Results}
\subsubsection*{Visualizing the learned action representations }
 	\begin{figure*}[ht]
    		\centering
    		\includegraphics[ height=2.8cm, width=4cm]{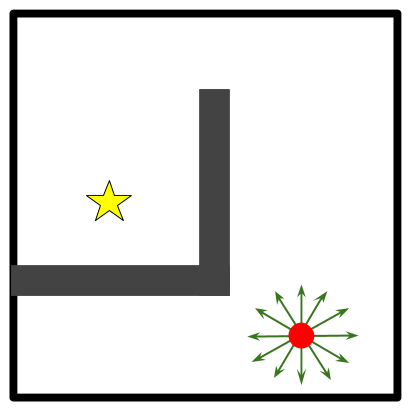}
    		\includegraphics[ height=3.2cm, width=4cm]{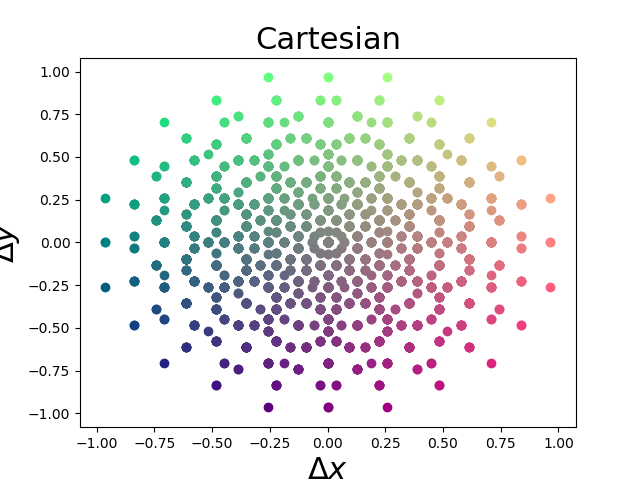}
    		\includegraphics[ height=3.2cm, width=4cm]{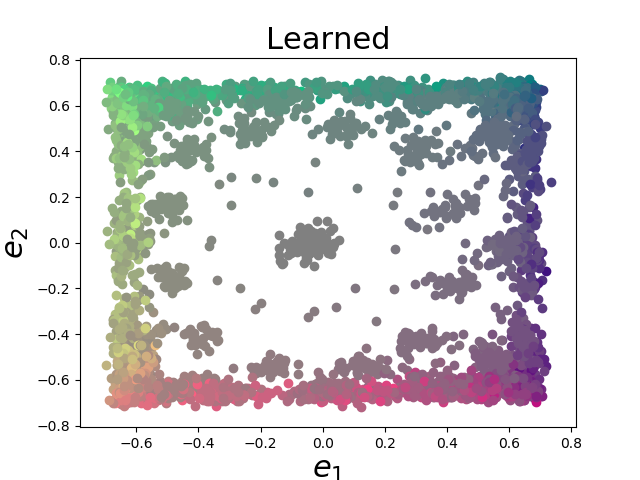}
    		\caption{ (a) The maze environment. 
    		The star denotes the goal state, the red dot corresponds to the agent and the arrows around it are the $12$  actuators.
    		Each action corresponds to a unique combination of these actuators.
    		Therefore, in total $2^{12}$ actions are possible.  
    		(b)  2-D representations for the displacements in the Cartesian co-ordinates caused  by each action, and (c) learned action embeddings.
    		In both (b) and (c), each action is colored based on the displacement ($\Delta x$, $\Delta y$) it produces. 
    		That is, with the color \lbrack R= $\Delta x$, G=$\Delta y$, B=$0.5$\rbrack, where $\Delta x$ and $\Delta y$ are normalized to $[0,1]$ before coloring. 
    		Cartesian actions are plotted on co-ordinates ($\Delta x$, $\Delta y$), and learned ones are on the coordinates in the embedding space.
    	    Smoother color transition of the learned representation is better as it corresponds to preservation of the \textit{relative} underlying structure.
    		The `squashing' of the learned embeddings is an artifact of a non-linearity applied to bound its range.
    		}
    		\label{Fig:emb}
    	\end{figure*}
To understand the internal working of our proposed algorithm, we present visualizations of the learned action representations on the maze domain.
%
%
A pictorial illustration of the environment is provided in Figure \ref{Fig:emb}. 
Here, the underlying structure in the set of actions is related to the displacements in the Cartesian coordinates.
This provides an intuitive base case against which we can compare our results.

In Figure \ref{Fig:emb}, we provide a comparison between the action representations learned using our algorithm and the underlying Cartesian representation of the actions.
It can be seen that the proposed method extracts useful structure in the action space.
Actions which correspond to settings where the actuators on the opposite side of the agent are selected result in relatively small displacements to the agent.
These are the ones in the center of plot.
In contrast, maximum displacement in any direction is caused by only selecting actuators facing in that particular direction.
Actions corresponding to those are at the edge of the representation space.
The smooth color transition indicates that not only the information about magnitude of displacement but the direction of displacement is also represented.
Therefore, the learned representations efficiently preserve the relative transition information among all the actions. %
To make exploration step tractable in the internal policy, $\pi_i$, we bound the  representation space along each dimension to the range [$-1,1$] using \textit{Tanh} non-linearity.
This results in `squashing' of these representations around the edge of this range.

\subsubsection*{Performance Improvement }

    	\begin{figure*}[t]
		\centering
		\includegraphics[width=0.29\textwidth]{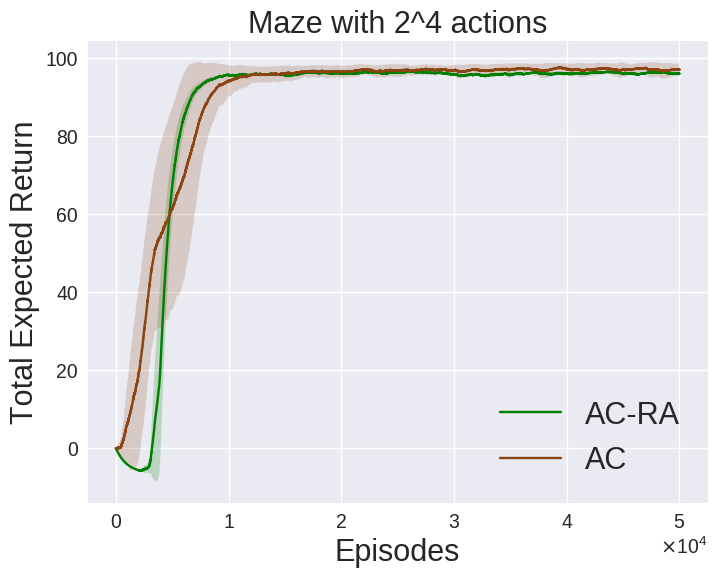} \hfill
		\includegraphics[width=0.29\textwidth]{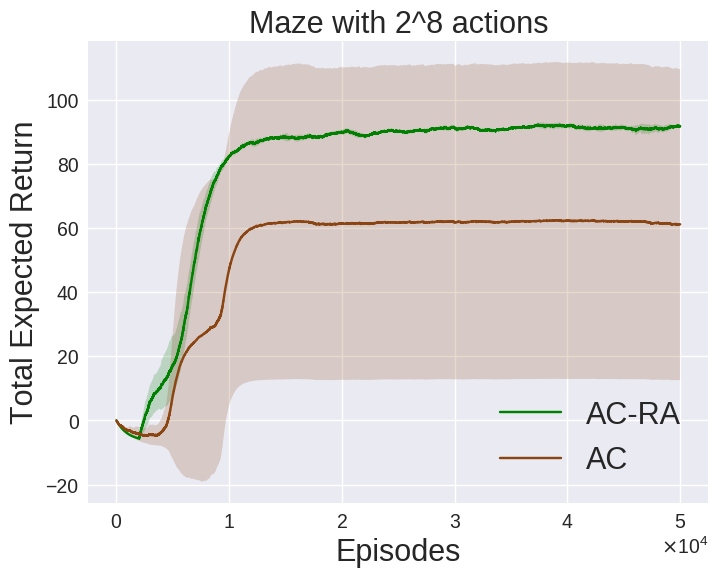} \hfill
		\includegraphics[width=0.29\textwidth]{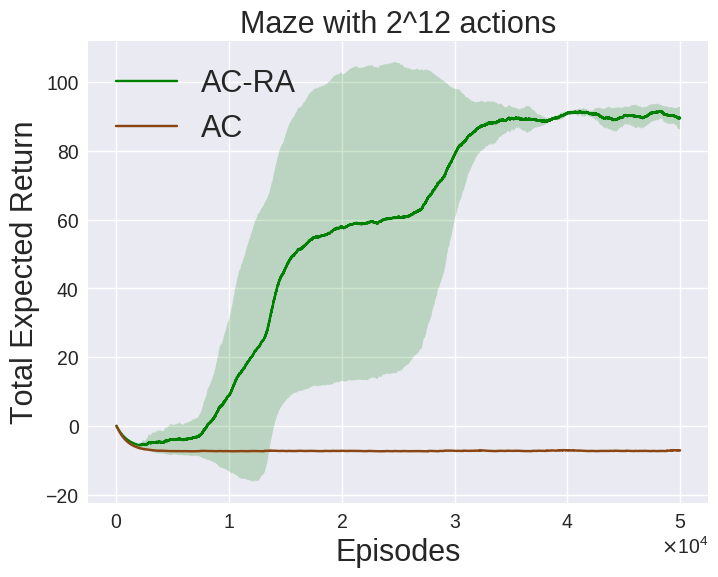} \\
		\includegraphics[width=0.29\textwidth]{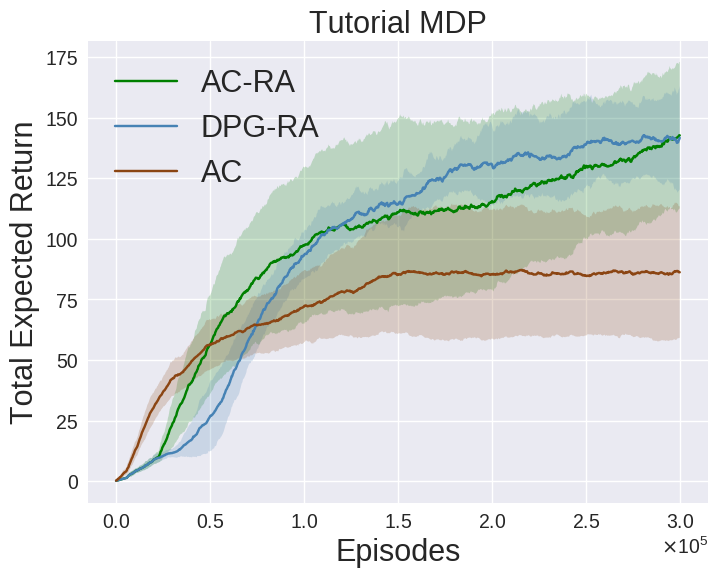}
		\hspace{20pt}
		\includegraphics[width=0.29\textwidth]{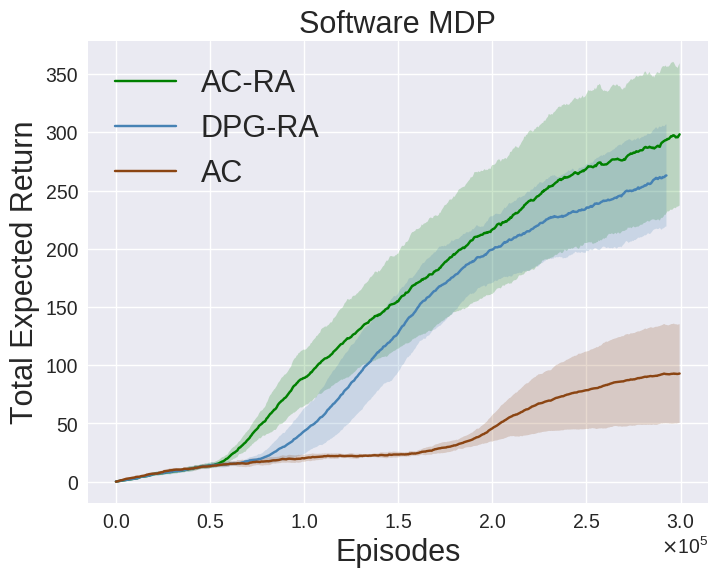}
		\caption{(Top) Results on the Maze domain with $2^4, 2^8,$ and $2^{12}$ actions respectively. (Bottom) Results on a) Tutorial MDP b) Software MDP. AC-RA and DPG-RA are the variants of PG-RA algorithm that uses actor-critic (AC) and DPG, respectively. The shaded regions correspond to one standard deviation and were obtained using $10$ trials.}
		\label{Fig:performance-plots}
	\end{figure*}
 
 %

The plots in Figure \ref{Fig:performance-plots} for the Maze domain show how the performance of standard actor-critic (AC) method deteriorates as the number of actions increases, even though the goal remains the same. 	
However, with the addition of an action representation module it is able to capture the underlying structure in the action space and consistently perform well across all settings.
Similarly, for both the tutorial and the software MDPs, standard AC methods fail to reason over longer time horizons under such an overwhelming number of actions, choosing mostly one-step actions that have high returns.  
In comparison, instances of our proposed algorithm are not only able to achieve significantly higher return, up to $2\times$ and $3\times$ in the respective tasks, but they do so much quicker.
These results reinforce our claim that learning action representations allow implicit generalization of feedback to other actions embedded in proximity to executed action.

Further, under the PG-RA algorithm, only a fraction of total parameters, the ones in the internal policy, are learned using the high variance policy gradient updates.
    	The other set of parameters associated with action representations are learned by a supervised learning procedure.
    	This reduces the variance of updates significantly, thereby making the PG-RA algorithms learn a better policy faster.
        This is evident from the plots in the Figure  \ref{Fig:performance-plots}.
These advantages allow the internal policy, $\pi_i$, to quickly approximate an optimal policy without succumbing to the curse of large actions sets. 

\section{Dynamic Actions}
\label{sec:dynaimic-actions}
Beside the large number of actions, in many real-world sequential decision making problems, the number of available actions (decisions) can vary over time. While problems like catastrophic forgetting, changing transition dynamics, changing rewards function, etc. have been well-studied in the lifelong learning literature, the setting where the action set changes remains unaddressed. In this section, we present an algorithm that autonomously adapts to an action set whose size changes over time. To tackle this open problem, we break it into two problems that can be solved iteratively: inferring the underlying, unknown, structure in the space of actions and optimizing a policy that leverages this structure. We demonstrate the efficiency of this approach on large-scale real-world lifelong learning problems
\cite{DBLP:journals/corr/abs-1906-01770}.

\subsection{Lifelong Markov Decision Process}
\label{sec:lmdp}
        MDPs, the standard formalization of decision making problems, are not flexible enough to encompass lifelong learning problems wherein the action set size changes over time. 
    	In this section we extend the standard MDP framework to model this setting. 

        In real-world problems where the set of possible actions changes, there is often underlying structure in the set of all possible actions (those that are available, and those that may become available).
        For example, tutorial videos can be described by feature vectors that encode their topic, difficulty, length, and other attributes;
        in robot control tasks, primitive locomotion actions like left, right, up, and down could be encoded by their change to the Cartesian coordinates of the robot, etc. 
        Critically, we will not assume that the agent knows this structure, merely that it exists. 
        If actions are viewed from this perspective, then the set of all possible actions (those that are available at one point in time, and those that might become available at any time in the future) can be viewed as a vector-space, $\mathcal E \subseteq \mathbb R^d$. 

\begin{figure}[t]
		\centering
		\includegraphics[width=0.7\textwidth]{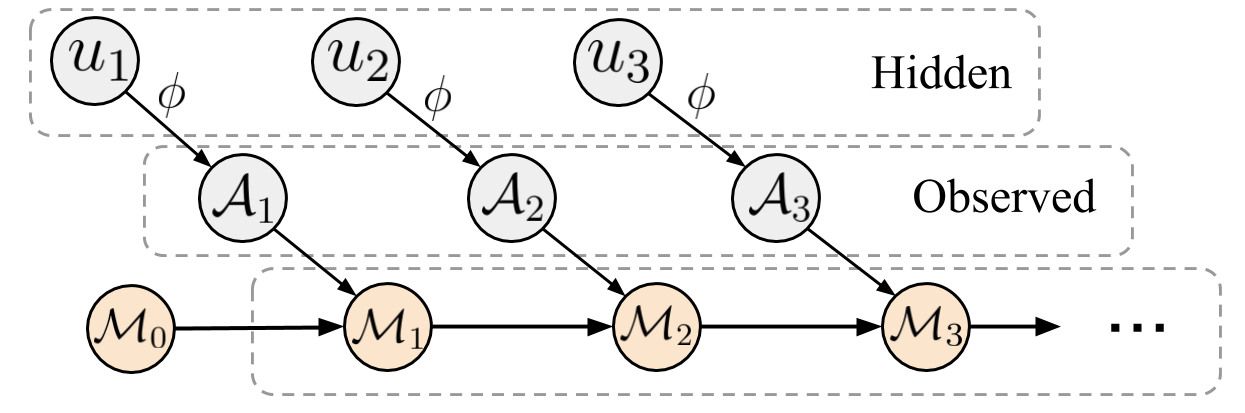}
		\caption{ Illustration of a \emph{lifelong MDP} where $\mathcal M_0$ is the base MDP. For every change $k$, $\mathcal M_K$ builds upon $\mathcal M_{k-1}$ by including the newly available set of actions $\mathcal A_k$. The internal structure in the space of actions is hidden and only a set of discrete actions is observed. 
		}
		\label{Fig:CL-MDP}
\end{figure}

        To formalize the lifelong MDP, we first introduce the necessary variables that govern when and how new actions are added.
        We denote the episode number using $\tau$.
        Let $I_\tau \in \{0, 1\}$ be a random variable that indicates whether a new set of actions are added or not at the start of episode $\tau$, and let frequency $\mathcal F: \mathbb N \rightarrow [0, 1]$ be the associated probability distribution over episode count, such that $\Pr(I_\tau=1) = \mathcal F(\tau)$. 
        Let $U_\tau \in 2^{\mathcal E}$ be the random variable corresponding to the set of actions that is added before the start of episode $\tau$. 
        When $I_\tau=1$, we assume that $U_\tau\neq\emptyset$, and when $I_\tau=0$, we assume that $U_\tau = \emptyset$. 
        Let $\mathcal D_\tau$ be the distribution of $U_\tau$ when $I_\tau=1$, i.e., $U_\tau \sim \mathcal D_\tau$ if $I_\tau=1$. 
        We use $\mathcal D$ to denote the set $\{\mathcal D_\tau\}$ consisting of these distributions.
        Such a formulation using $I_\tau$ and $\mathcal D_\tau$ provides a fine control of when and how new actions can be incorporated.
        This allows modeling a large class of problems where both the distribution over the type of incorporated actions as well intervals between successive changes might be irregular. 
        Often we will not require the exact episode number $\tau$ but instead require $k$, which denotes the number of times the action set is changed.

        Since we do not assume that the agent knows the structure associated with the action, we instead provide actions to the agent as a set of discrete entities, $\mathcal A_k$. 
        To this end, we define $\phi$ to be a map relating the underlying structure of the new actions to the observed set of discrete actions $\mathcal A_k$ for all $k$, i.e., if the set of actions added is $u_k$, then 
        $\mathcal A_k=\{  \phi(e_i) | e_i \in u_k\}$. 
        Naturally, for most problems of interest, neither the underlying structure $\mathcal E$, nor the set of distributions $\mathcal D$, nor the frequency of updates $\mathcal F$, nor the relation $\phi$ is known---the agent only has access to the observed set of discrete actions.

        We now define the \textit{lifelong Markov decision process} (L-MDP)  as $\mathscr{L} = (\mathcal M_0, \mathcal E, \mathcal D, \mathcal F)$, which extends a \textit{base} MDP $\mathcal M_0 = (\mathcal{S}, \mathcal{A},\mathcal{P},\mathcal{R}, \gamma, d_0)$. 
    	$\mathcal{S}$ is the set of all possible states that the agent can be in, called the state set.
    	$\mathcal A$ is the discrete set of actions available to the agent, and for $\mathcal M_0$ we define this set to be empty, i.e., $\mathcal A=\emptyset$. 
    	When the set of available actions changes and 
    	the agent observes a new set of discrete actions, $\mathcal A_k$, then $\mathcal M_{k-1}$ transitions to $\mathcal M_k$, such that  $\mathcal A$ in $\mathcal M_k$ is the set union of $\mathcal A$ in $\mathcal M_{k-1}$ and $\mathcal A_k$. 
    	Apart from the available actions, other aspects of the L-MDP remain the same throughout.
    	An illustration of the framework is provided in Figure \ref{Fig:CL-MDP}.
    	We use $S_t \in \mathcal S$, $A_t \in \mathcal A$, and $R_t \in \mathbb R$ as random variables for denoting the state, action and reward at time $t \in \{0,1,\dotsc\}$ within each episode.
    	%
        %
    	The first state, $S_0$, comes from an initial distribution, $d_0$, and the reward function $\mathcal R$ is defined to be only dependent on the state such that $\mathcal R(s)=\mathbf{E}[R_t|S_t=s]$ for all $s \in \mathcal S$. 
    	We assume that $R_t \in [-R_\text{max},R_\text{max}]$ for some finite $R_\text{max}$. 
    	The reward discounting parameter is given by $\gamma \in [0,1)$. 
    	$\mathcal{P}$ is the state transition function, such that for all $s,a,s',t$, the function $ \mathcal{P}(s,a,s')$ denotes the transition probability $ P(s'| s, e)$, where $a = \phi(e)$.\footnote{For notational ease, (a) we overload symbol $P$ for representing both probability mass and density; 
    	(b) we assume that the state set is finite, however, our primary results extend to MDPs with continuous states.
    	} 

    	In the most general case, new actions could be completely arbitrary and have no relation to the ones seen before.
    	In such cases, there is very little hope of lifelong learning by leveraging past experience.
    	To make the problem more feasible, we resort to a notion of \textit{smoothness} between actions.
    	Formally, we assume that transition probabilities in an L-MDP are $\rho-$Lipschitz in the structure of actions, i.e., $\exists \rho > 0$ such that 
    	\begin{equation} 
    	 \forall s, s', e_i, e_j \hspace{5pt} \lVert P(s'| s,e_i) - P(s'| s,e_j) \rVert_1 \leq \rho \lVert e_i - e_j\rVert_1. 
    	\label{eqn:lipschitz}
    	\end{equation}
    	For any given MDP $\mathcal{M}_k$ in $\mathscr L$, an agent's goal is to find a policy, $\pi_k$, that maximizes the expected sum of discounted future rewards.
    	For any policy $\pi_k$, the corresponding state value function is $v^{\pi_k}(s) = \mathbf{E}[\sum_{t=0}^{\infty}\gamma^t R_{t} |s, \pi_k]$.

 \subsection{Blessing of Changing Action Sets}
 \label{sec:blessing}
Finding an optimal policy when the set of possible actions is large 
is 
difficult due to the curse of dimensionality.
In the L-MDP setting this problem might appear to be exacerbated, as an agent must additionally adapt to the changing levels of possible performance as new actions become available. 
This raises the natural question: \textit{as new actions become available, how much does the performance of an optimal policy change? }
%
If it fluctuates significantly, can a lifelong learning agent succeed by continuously adapting its policy, or is it better to learn from scratch with every change to the action set? 

To answer this question, 
consider an optimal policy, $\pi^*_k$, for MDP $\mathcal M_k$, i.e., an optimal policy when considering only policies that use actions that are available during the $k^\text{th}$ episode. 
%
%
We now 
quantify how sub-optimal $\pi^*_k$ is relative to the performance of a hypothetical policy, $\mu^*$, that acts optimally given access to all possible actions. 
\begin{thm}
\label{thm:1}
In an L-MDP, let $\epsilon_k$ denote the maximum distance in the underlying structure of the closest pair of available actions, i.e.,  
$\epsilon_k \coloneqq \underset{a_i \in \mathcal A}{\text{sup}} ~ \underset{a_j \in \mathcal A}{\text{inf}}  \lVert e_i - e_j \rVert_1$, 
then
\begin{align}
    v^{\mu^*}(s_0) -v^{\pi^*_k}(s_0)   &\leq \frac{\gamma  \rho \epsilon_k}{(1 - \gamma)^2}  R_{\text{max}}.  \label{eqn:thm1} 
\end{align}
\end{thm}
%
%
The proof is available in \cite{DBLP:journals/corr/abs-1906-01770}. With a bound on the maximum possible sub-optimality, Theorem \ref{thm:1} presents an important connection between achievable performances, the nature of underlying structure in the action space, and a property of available actions in any given $\mathcal M_k$.
Using this, we can make the following conclusion.
\begin{cor}
\label{cor:1}
Let $\mathcal Y\subseteq \mathcal E$ be the smallest closed set such that,  $P(U_k \subseteq 2^\mathcal Y)=1$. We refer to $\mathcal Y$ as the element-wise-support of $U_k$. If $\,\,\forall k$, the element-wise-support of $U_k$ in an L-MDP is $\mathcal E$, 
then as $k \rightarrow \infty$ the sub-optimality vanishes. That is,
$$\lim_{k \rightarrow \infty} v^{\mu^*}(s_0) -v^{\pi^*_k}(s_0)  \rightarrow 0. $$
\end{cor}
%
%
Through Corollary \ref{cor:1}, we can now establish that the change in optimal performance will eventually converge to zero as new actions are repeatedly added. 
An intuitive way to observe this result would be to notice that every new action that becomes available indirectly provides more information about the underlying, unknown, structure of $\mathcal E$.
However, in the limit, as the size of the available action set increases, the information provided by each each new action vanishes and thus performance saturates.

Certainly, in practice, we can never have $k \rightarrow \infty$, but this result is still advantageous.
Even when the underlying structure $\mathcal E$, the set of distributions $\mathcal D$, the change frequency $\mathcal F$, and the mapping relation $\phi$ are all \textit{unknown}, it establishes the fact that the difference between the best performances in \textit{successive changes} will remain bounded and will not fluctuate arbitrarily.
%
%
This opens up new possibilities for developing algorithms that do 
not need to start from scratch after new actions are added, but rather can build upon their past experiences using updates to their existing policies that efficiently leverage estimates of the structure of $\mathcal E$ to adapt to new actions.  
%
%
%
\subsection{Learning with Changing Action Sets}
\label{sec:learning}
Theorem \ref{thm:1}  characterizes what \textit{can be} achieved in principle, however, it does not specify \textit{how} to achieve it---how to find $\pi_k^*$ efficiently.
Using any parameterized policy, $\pi$, which acts directly in the space of observed actions, suffers from one key practical drawback in the L-MDP setting.
That is, the parameterization is deeply coupled with the number of actions that are available. 
That is, not only is the meaning of each parameter coupled with the number of actions, but often the number of parameters that the policy has is dependent on the number of possible actions. 
This makes it unclear how the policy should be adapted when additional actions become available. 
A trivial solution would be to ignore the newly available actions and continue only using the previously available actions.
However, this is clearly myopic, and will prevent the agent from achieving the better long term returns that might be possible using the new actions.

To address this parameterization-problem, instead of having the policy, $\pi$, act directly in the observed action space, $\mathcal A$, we propose an approach wherein the agent reasons about the underlying structure of the problem in a way that makes its policy parameterization invariant to the number of actions that are available. 
To do so, we split the policy parameterization into two components.
The first component corresponds to the state conditional policy responsible for making the decisions, $\beta : \mathcal S \times \hat{\mathcal E} \rightarrow [0, 1]$, where $\hat {\mathcal E} \in \mathbb{R}^d$.
The second component corresponds to $\hat \phi : \hat{ \mathcal E} \times \mathcal A \rightarrow [0,1]$, an estimator of the relation $\phi$, which is used to map the output of $\beta$ to an action in the set of available actions.
That is, an $E_t \in \hat{\mathcal E}$ is sampled from $\beta(S_t, \cdot)$ and then $ \hat \phi(E_t)$ is used to obtain the action $A_t$. 
Together, $\beta$ and $\hat \phi$ form a complete policy, and $\hat {\mathcal E}$ corresponds to the inferred structure in action space.

One of the prime benefits of estimating $\phi$ with $\hat \phi$ 
is that it makes the parameterization of $\beta$ invariant to the cardinality of the action set---changing the number of available actions does not require changing the number of parameters of $\beta$.
%
%
Instead, only the parameterization of $\hat \phi$, the estimator of the underlying structure in action space, must 
be modified when new actions become available. 
We show next that the update to the parameters of $\hat \phi$ can be performed using \emph{supervised learning} methods that are independent of the reward signal and thus typically more efficient than RL methods. 
%
%
%

 %

%
%
%
 
While our proposed 
parameterization of the policy using both $\beta$ and $\hat \phi$ has the advantages described above, the performance of $\beta$ is now constrained by the quality of $\hat \phi$, as in the end $\hat \phi$ is responsible for selecting an action from $\mathcal A$.
Ideally we want $\hat \phi$ to be such that it lets $\beta$  be both: (a) invariant to the cardinality of the action set for practical reasons and (b) as expressive 
as a policy, $\pi$, explicitly parameterized for the currently available actions. 
Similar trade-offs have been considered in the context of learning optimal state-embeddings for representing sub-goals in hierarchical RL \citep{nachum2018near}.  
For our lifelong learning setting, we build upon their method to efficiently estimate $\hat \phi$ in a way that provides bounded sub-optimality. 
%
Specifically, we make use of an additional \textit{inverse dynamics} function, $\varphi$, that takes as input two states, $s$ and $s'$, and produces as output a prediction of which $e \in \mathcal E$ caused the transition from $s$ to $s'$. 
Since the agent does not know $\phi$, when it observes a transition from $s$ to $s'$ via action $a$, it does \emph{not} know which $e$ caused this transition. 
So, we cannot train $\varphi$ to make good predictions using the actual action, $e$, that caused the transition. 
Instead, we use $\hat \phi$ to transform the prediction of $\varphi$ from $e \in \mathcal E$ to $a \in \mathcal A$, and train both $\varphi$ and $\hat \phi$ so that this process accurately predicts which action, $a$, caused the transition from $s$ to $s'$. 
Moreover, rather than viewing $\varphi$ as a deterministic function mapping states $s$ and $s'$ to predictions $e$, we define $\varphi$ to be a \textit{distribution} over $\mathcal E$ given two states, $s$ and $s'$. 
%

For any given $\mathcal M_k$ in L-MDP $\mathscr L$, let $\beta_k$ and $\hat \phi_k$ denote the two components of the overall policy and let $\pi_k^{**}$ denote the best overall policy that can be represented using some fixed $\hat \phi_k$. 
The following theorem bounds the sub-optimality of $\pi_k^{**}$.
\begin{thm}
\label{thm:2} For an L-MDP $\mathcal M_k$,
If there exists a $\varphi : S \times S \times \hat{\mathcal E} \rightarrow [0, 1] $ and $\hat \phi_k : \hat {\mathcal E} \times \mathcal A \rightarrow [0, 1]$ such that 
\begin{align}
    \footnotesize
    \sup_{s \in \mathcal S, a \in \mathcal A} \text{KL}\Big(P(S_{t+1}|S_t=s, A_t=a) \Vert P(S_{t+1}|S_t=s, A_t=\hat A) \Big) &\leq \delta_k^2/2, \label{eqn:lemma1} 
\end{align}
where $\hat A 
\sim \hat \phi_k(\cdot|\hat E)$ and $\hat E \sim \varphi(\cdot | S_t, S_{t+1})$, then 
\begin{align*}
 v^{\mu^*}(s_0) -v^{\pi_k^{**}}(s_0)  &\leq \frac{\gamma \left( \rho \epsilon_k + \delta_k \right)}{(1 - \gamma)^2}  R_{\text{max}}.
\end{align*}
\end{thm}
%
%

See \cite{DBLP:journals/corr/abs-1906-01770} for the proof. By quantifying the impact $\hat \phi$ has on the sub-optimality of achievable performance, Theorem \ref{thm:2} provides the necessary constraints for estimating $\hat \phi$. 
%
%
At a high level, Equation \eqref{eqn:lemma1} ensures $\hat \phi$ to be such that it can be used to generate an action corresponding to any $s$ to $s'$ transition.
%
%
This allows $\beta$ to leverage $\hat \phi$ and choose the required action that induces the state transition needed for maximizing performance.
Thereby, following \eqref{eqn:lemma1}, sub-optimality would be minimized if $\hat \phi$ and $\varphi$ are optimized to reduce the supremum of KL divergence over all $s$ and $a$. 
In practice, however, the agent does not have access to all possible states, rather it has access to a limited set of samples collected from interactions with the environment.
Therefore, instead of the supremum, we propose minimizing the average over all $s$ and $a$ from a set of observed transitions,
\begin{align}
    \!\!\! \mathcal L(\hat \phi, \varphi) \!\!&\coloneqq \!\! \sum_{s \in \mathcal S}\! \sum_{a \in \mathcal A_k} P(s, a) \textit{KL}\left(P(s'|s, a) \Vert P(s'|s, \hat a) \right) . \label{eqn:exp_kl}
\end{align}
Equation \eqref{eqn:exp_kl}  suggests that $\mathcal L(\hat \phi, \varphi)$ would be minimized when $\hat a$ equals $a$, but using \eqref{eqn:exp_kl}  directly in the current form is inefficient as it requires computing KL over all probable $s' \in \mathcal S$ for a given $s$ and $a$.
To make it practical, we make use of the following property.
\begin{prop}
\label{prop:lb}
For some constant C, 
$- \mathcal L(\hat \phi, \varphi)$ is lower bounded by
\begin{align}
     \sum_{s \in \mathcal S} \sum_{a \in \mathcal A_k}\sum_{s' \in \mathcal S} P(s, a, s') \Bigg (  \mathbf{E}\left[\log \hat \phi(\hat a|\hat e) \middle | \hat e \sim \varphi(\cdot|s,s')  \right]
     - \text{KL}\Big(\varphi(\hat e|s, s') \Big \Vert P(\hat e|s, s')  \Big) \Bigg) + C. ~~~~~\label{eqn:beta_vae}
\end{align}
\end{prop}
As minimizing $\mathcal L(\hat \phi, \varphi)$ is equivalent to maximizing $- \mathcal L(\hat \phi, \varphi)$, we consider maximizing the lower bound obtained from Property \ref{prop:lb}.
In this form, it is now practical to optimize \eqref{eqn:beta_vae} just by using the observed $(s, a, s')$ samples. 
As this form is similar to the objective for variational auto-encoder, inner expectation can be efficiently optimized using the reparameterization trick \citep{kingma2013auto}.
$P(\hat e | s, s')$ is the prior on $\hat e$, and we treat it as a hyper-parameter that allows the KL to be computed in closed form.

Importantly, note that this optimization procedure only requires individual transitions, $s,a,s'$, 
and is independent of the reward signal. 
Hence, at its core, it is a supervised learning procedure. 
This means that learning good parameters for $\hat \phi$ tends to require far fewer samples than optimizing $\beta$ (which is an RL problem). 
This is beneficial for our approach because $\hat \phi$, the component of the policy where new parameters need to be added when new actions become available, 
can be updated efficiently.
As both $\beta$ and $\varphi$ are invariant to action cardinality, they do not require new parameters when new actions become available. 
Additional implementation level details are available in Appendix F. 

\subsection{Algorithm}
\label{sec:algo}
 \begin{figure}[t]
        \centering
		\includegraphics[width=0.5\textwidth]{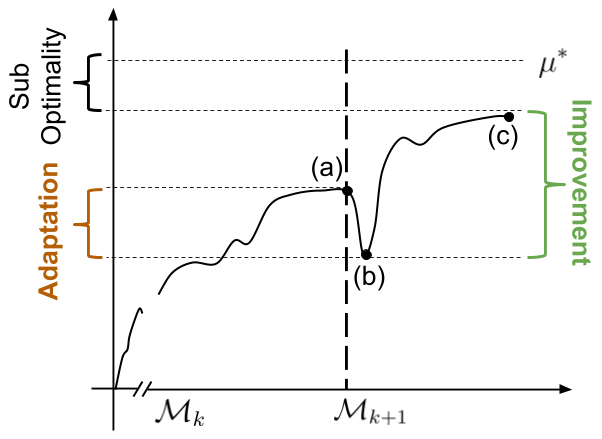}
		\caption{An illustration of a typical performance curve for a lifelong learning agent.
		The point $(a)$ corresponds to the performance of the current policy in $\mathcal M_k$.
		The point $(b)$ corresponds to the performance drop resulting as a consequence of adding new actions.
		We call the phase between (a) and (b) as the adaptation phase, which aims at minimizing this drop when adapting to new set of actions.
		The point $(c)$ corresponds to the improved performance in $\mathcal M_{k+1}$ by optimizing the policy to leverage the new set of available actions. 
		$\mu^*$ represents the best performance of the hypothetical policy which has access to the entire structure in the action space.
		}
		\label{Fig:LAICA}
\end{figure}

 When a new set of actions, $\mathcal A_{k+1}$, becomes available, the agent should leverage the existing knowledge and quickly adapt to the new action set.
 Therefore, during every change in $\mathcal M_k$, the ongoing best components of the policy, $\beta_{k-1}^*$ and $\phi_{k-1}^*$, in $\mathcal M_{k-1}$ are carried over, i.e., $\beta_k \coloneqq \beta_{k-1}^*$ and $\hat \phi_k \coloneqq \hat \phi_{k-1}^*$.
 For lifelong learning, the following property illustrates a way to organize the learning procedure so as to minimize the sub-optimality in each $\mathcal M_k$, for all $k$.  
 \begin{prop} (Lifelong Adaptation and Improvement)\label{prop:LAICA}
 In an L-MDP,  let $\Delta$ denote the difference of performance between $v^{\mu^*}$ and the best achievable using our policy parameterization, then the overall sub-optimality can be expressed as, 
  \begin{align}
    v^{\mu^*}(s_0) - v_{_{\mathcal M_1}}^{\beta_1 \hat \phi_1}(s_0) =&  \underbrace{\sum_{k=1}^{\infty}\left(v_{_{\mathcal M_k}}^{\beta_{k} \hat \phi_k^*}(s_0) - v_{_{\mathcal M_{k}}}^{\beta_{k}  \hat\phi_{k}}(s_0) \right)}_{\text{Adaptation}} + \underbrace{\sum_{k=1}^{\infty}\left(v_{_{\mathcal M_k}}^{\beta_k^*  \hat\phi_k^*}(s_0) - v_{_{\mathcal M_k}}^{\beta_k  \hat\phi_k^*}(s_0) \right)}_{\text{Policy Improvement}} + \Delta, 
 \label{eqn:adapt_improve}
  \end{align} 
  where $\mathcal M_k$ is used in the subscript to emphasize the respective MDP in $\mathscr L$. 
 \end{prop}
%
%

Property \ref{prop:LAICA} illustrates a way to understand the impact of $\beta$ and $\hat \phi$ by splitting the learning process into an adaptation phase and a policy improvement phase.
These two iterative phases are the crux of our algorithm for solving an L-MDP $\mathscr L$.
Based on this principle, we call our algorithm LAICA: \textit{lifelong adaptation and improvement for changing actions}.
%

Whenever new actions become available, adaptation is prone to cause a performance drop as the agent has no information about when to use the new actions, and so its initial uses of the new actions may be at inappropriate times. 
%
%
%
Following Property \ref{prop:lb}, we update $\hat \phi$ so as to efficiently infer the underlying structure and minimize this drop.
That is, for every $\mathcal M_k$, $ \hat\phi_k$ is first adapted to $ \hat \phi_k^*$ in the adaptation phase by adding more parameters for the new set of actions and then optimizing \eqref{eqn:beta_vae}.
After that, $ \hat \phi_k^*$ is fixed and $\beta_k$ is improved towards $\beta_k^*$ in the policy improvement phase, by updating the parameters of $\beta_k$ using the policy gradient theorem \citep{sutton2000policy}.
These two procedures are performed sequentially whenever $\mathcal M_{k-1}$ transitions to $\mathcal M_k$, for all $k$, in an L-MDP $\mathscr L$.
An illustration of the procedure is presented in Figure \ref{Fig:LAICA}.
%


%
%
%
A step-by-step pseudo-code for the LAICA algorithm is available in Algorithm 1.
The crux of the algorithm is based on the iterative adapt and improve procedure obtained from Property \ref{prop:LAICA}.

We begin by initializing the parameters for $\beta_{0}^*, \hat \phi_{0}^*$ and $\varphi_{0}^*$.
In Lines $3$ to $5$, for every change in the set of available actions, instead of re-initializing from scratch, the previous best estimates for $\beta, \hat\phi$ and $\varphi$ are carried forward to build upon existing knowledge.
As $\beta$ and $\varphi$ are invariant to the cardinality of the available set of actions, no new parameters are required for them.
In Line $6$ we add new parameters in the function $\hat \phi$ to deal with the new set of available actions.

To minimize the adaptation drop, we make use of Property \ref{prop:lb}.
Let $\mathcal L^{\text{lb}}$ denote the lower bound for $\mathcal L$, such that,
$$ \mathcal L^{\text{lb}}(\hat \phi, \varphi) \coloneqq \mathbf{E}\left[\log \hat \phi(\hat A_t|\hat E_t) \middle | \varphi(\hat E_t|S_t,S_{t+1})  \right] - \lambda \text{KL}\left(\varphi(\hat E_t|S_t, S_{t+1}) \middle \Vert P(\hat E_t|S_t, S_{t+1})  \right).$$
Note that following the literature on variational auto-encoders, we have generalized \eqref{eqn:beta_vae} to use a Lagrangian $\lambda$ to weight the importance of KL divergence penalty \citep{higgins2017beta}.
When $\lambda = 1$, it degenrates to \eqref{eqn:beta_vae}.
We set the prior $P(\hat e|s, s')$ to be an isotropic normal distribution, which also allows KL to be computed in closed form  \citep{kingma2013auto}. 
From Line $7$ to $11$ in the Algorithm 1, random actions from the available set of actions are executed and their corresponding transitions are collected in a buffer.
Samples from this buffer are then used to maximize the lower bound objective $\mathcal L^{\text{lb}}$ and adapt the parameters of $\hat \phi$ and $\varphi$.
The optimized $\hat \phi^*$ is then kept fixed during policy improvement.

Lines $16$--$22$ correspond to the standard policy gradient approach for improving the performance of a policy.
In our case, the policy $\beta$ first outputs a vector $\hat e$ which gets mapped by $\hat \phi^*$ to an action.
The observed transition is then used to compute the policy gradient \citep{sutton2000policy} for updating the parameters of $\beta$ towards $\beta^*$.  
%
%
	%
If a critic is used for computing the policy gradients, then it is also subsequently updated by minimizing the TD error \citep{Sutton1998}.
This iterative process of adaption and policy improvement continues for every change in the action set size.
%
%
%
		\IncMargin{1em}
	\begin{algorithm2e}
		\label{Alg:laica} 
		\caption{Lifelong Adaptation and Improvement for Changing Actions (LAICA)}
	    \textbf{Initialize} $\beta_{0}^*, \hat \phi_{0}^*, \varphi_{0}^*$.
	    \\
        \For {\text{change} $k = 1,2...$} 
        {
            
            $\beta_k \leftarrow \beta_{k-1}^*$  \tikzmark{top}
            \\
            $\varphi_k \leftarrow \varphi_{k-1}^*$ 
            \\
            $\hat \phi_k \leftarrow \hat \phi_{k-1}^*$ 
            \\
            Add parameters in $\hat \phi_k$ for new actions ~~~~~~~~\tikzmark{right} \tikzmark{bottom}
            \\
            
            \vspace{10pt}
    		Buffer $\mathbb{B} = \{\}$  \tikzmark{top2}
    		\\
           \For {$episode = 0,1,2...$}
            {
    			\For {$t = 0,1,2...$}
    			{
			    Execute random $a_t$ and observe $s_{t+1}$ \\
			    Add transition to $\mathbb{B}$ 
			    }
            }
           \For {$iteration = 0,1,2...$}
            {
                Sample batch $b  \sim \mathbb{B}$
                \\
                Update $\hat \phi_k$ and $\varphi_k$ by maximizing $\mathcal L^{\text{lb}}(\hat \phi_k, \varphi_k)$ for $b$ ~~~~~~~~~~\tikzmark{right2}
                
            }    \tikzmark{bottom2}
			 \\
            \vspace{8pt}
    		\tikzmark{top3}
            \For {$episode = 0,1,2...$}{
			\For {$t = 0,1,2...$} {
			    Sample $\hat e_t \sim \beta_k(\cdot|s_t) $ \\
			    Map $\hat e_t$ to an action $a_t$ using $ \hat \phi_k^*(e)$ \\
			    Execute $a_t$ and observe $s_{t+1}, r_{t}$ \\
			    Update $\beta_k$ using any policy gradient algorithm ~~~~~~~~~~~~\tikzmark{right3}\\
			    Update critic by minimizing TD error. \tikzmark{bottom3}
			}
		}
	} 
	\nonl 
	\AddNote{top}{bottom}{right}{Reuse past knowledge.}
	\AddNote{top2}{bottom2}{right2}{Adapt \\ $\hat \phi_k$ to $\hat \phi_k^*$  }
	\AddNote{top3}{bottom3}{right3}{Improve \\ $\beta_k$ to $\beta_k^*$}
	\end{algorithm2e}
	\DecMargin{1em}   
 
\subsection{Empirical Analysis}

\label{sec:empirical}
        In this section, we aim to empirically compare the following methods, 
        \begin{itemize}
            \item Baseline(1): The policy is re-initialised  and the agent learns from scratch after every change.
            \item Baseline(2): New parameters corresponding to new actions are added/stacked to the existing policy (and previously learned parameters are carried forward as-is).
            \item LAICA(1): The proposed approach that leverages the structure in the action space. To act in continuous space of inferred structure, we use DPG \citep{silver2014deterministic} to optimize $\beta$.
            \item LAICA(2): A variant of LAICA which uses an actor-critic \citep{Sutton1998} to optimize $\beta$.
        \end{itemize}
        
    	To demonstrate the effectiveness of our proposed method(s) on lifelong learning problems, we consider a maze environment and two domains corresponding to real-world applications, all with a large set of changing actions. 
		For each of these domains, the total number of actions  were randomly split into five equal sets.
		Initially, the agent only had the actions available in the first set and after every change the next set of actions was made available additionally.
        In the following paragraphs we briefly outline the domains.
        

    	\paragraph{Case Study: Real-World Recommender Systems. } 
    	We consider the following two real-world applications of large-scale recommender systems that require decision making over multiple time steps and where the number of possible decisions varies over the lifetime of the system.
    	%
    	%
    	\begin{itemize}
    	    \item A web-based video-tutorial platform, that has a recommendation engine to suggest a series of tutorial videos. The aim is to meaningfully engage the users in a learning activity.
    	    %
    	    %
    	    In total, $1498$ tutorials were considered for recommendation.
            \item A professional multi-media editing software, where sequences of tools inside the software need to be recommended. The aim is to increase user productivity and assist users in quickly achieving their end goal. In total, $1843$ tools were considered for recommendation.  
    	\end{itemize}
    	
        For both of these applications, an existing log of user's click stream data was used to create an $n$-gram based MDP model for user behavior \citep{shani2005mdp}.
        Sequences of user interaction were aggregated to obtain over $29$ million clicks and $1.75$ billion user clicks for the tutorial recommendation and the tool recommendation task, respectively.    
        The MDP had continuous state-space, where each state consisted of the feature descriptors associated with each item (tutorial or tool) in the current $n$-gram. 
        %
        \subsection{Results}
                    \begin{figure*}[t]
    		\centering
    		\includegraphics[width=0.8\textwidth]{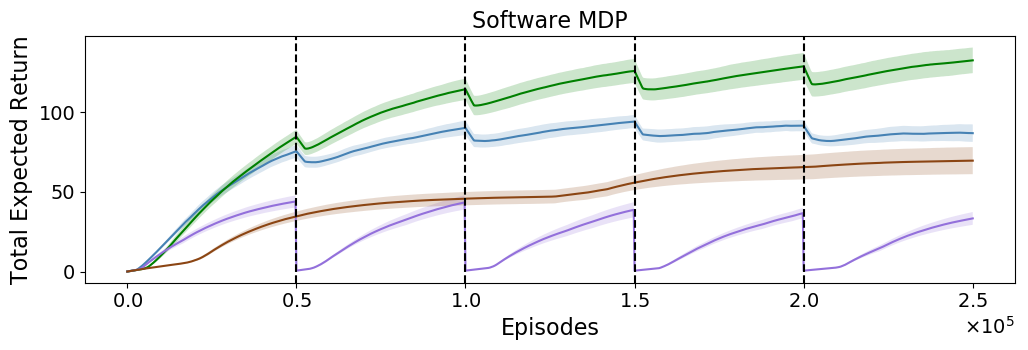}\\
    		\includegraphics[width=0.8\textwidth]{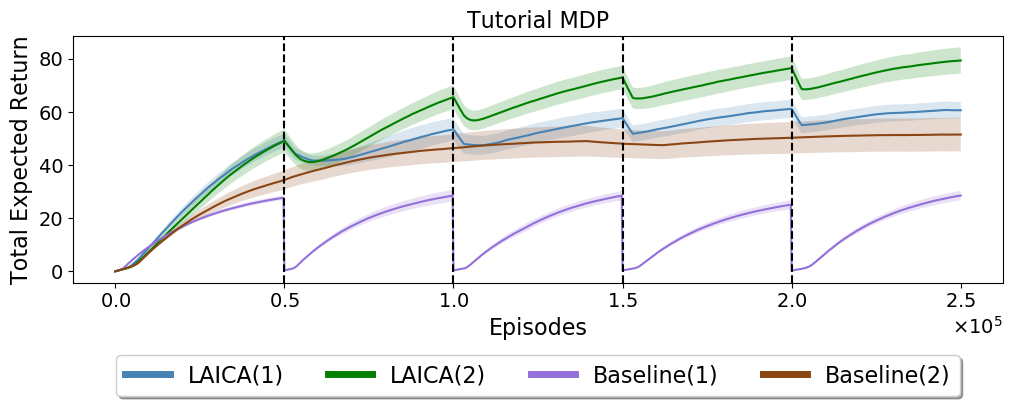}
    		\caption{Lifelong learning experiments with a changing set of actions in the recommender system domains. 
    		The learning curves correspond to the running mean of the best performing setting for each of the algorithms.  
    		The shaded regions correspond to standard error obtained using $10$ trials.
    		Vertical dotted bars indicate when the set of actions was changed.
    		}
    		\label{Fig:CL-experiments2}
		\end{figure*}
    	The plots in Figure \ref{Fig:CL-experiments2} present the evaluations on the domains considered.
    	The advantage of LAICA over Baseline(1) can be attributed to its policy parameterization.
    	The decision making component of the policy, $\beta$, being invariant to the action cardinality can be readily leveraged after every change without having to be re-initialized. 
    	This demonstrates that efficiently re-using past knowledge can improve data efficiency over the approach that learns from scratch every time.

	    Compared to Baseline(2), which also does not start from scratch and reuses existing policy, we notice that the variants of LAICA algorithm still perform favorably.
    	As evident from the plots in Figure \ref{Fig:CL-experiments2}, while Baseline(2) does a good job of preserving the existing policy, it fails to efficiently capture the benefit of new actions.

    	While the policy parameters in both LAICA and Baseline(2) are improved using policy gradients,  the superior performance of LAICA can be attributed to the adaptation procedure incorporated in LAICA which aims at efficiently inferring the underlying structure in the space of actions.
    	Overall LAICA(2) performs almost twice as well as both the baselines on all of the tasks considered.

    	Note that even before the first addition of the new set of actions, the proposed method performs better than the baselines.
    	This can be attributed to the fact that the proposed method efficiently leverages the underlying structure in the action set and thus learns faster.
    	Similar observations have been made previously \citep{dulac2015deep,he2015deep,bajpai2018transfer,naturalGuy}.

\section{Cognitive Bias Aware}
\label{sec:bias-aware}
While SRs are beginning to find their way in academia and industry, other recommendation technologies such as collaborative filtering \cite{Netflix} and contextual bandits \cite{Li10} have been the mainstream methods.  There are hundreds of papers every year in top-tier machine learning conferences advancing the state of the art in collaborative filtering and bandit technologies.   Nonetheless, all these systems do not truly understand people but rather naively optimize expected utility metrics such as click through rate. People,  do not perceive expected values in a rational fashion \cite{Tversky92}.  In fact, people have evolved with various cognitive biases to simplify information processing. Cognitive biases are systematic patterns of deviation from norm or rationality in judgment.   They are rules of thumb that help people make sense of the world and reach decisions with relative speed. Some of these biases include the perception of risk, collective effects and long-term decision making. In this final section we argue that the next generation of  recommendation systems needs to incorporate human cognitive biases.  We would like to build recommendation and personalization systems that are aware of these biases.  We would like to do it in a fashion that is win-win for both the marketer and the consumer.  Such technology is not studied in academia and does not exist in the industry \cite{DBLP:conf/um/TheocharousHMS19}.

It has becoming a well-established result that humans do not reason by maximizing their expected utility, and yet most work in artificial intelligence and machine learning continues to be based on idealized mathematical models of decisions \cite{Machina87,Ellsberg61,VonNeumann44}. There are many studies that show that given two choices that are logically equivalent, people will prefer one to another based on how the information is presented to them, even if the choices made violate the principle of maximizing expected utility \cite{Allais53}. The Allais paradox is a classic example that demonstrates this type of inconsistency in the choices people make when presented with two different gambling strategies. It was found that people often preferred a gambling strategy with a lower expected utility but with more certain positive gains over a strategy where expected utility is higher at the cost of more uncertainty. Furthermore, there are a number of biases that have been explored in the context of marketing and eCommerce that influence the manner in which consumers make purchasing decisions. An example of such a bias is the \emph{decoy effect} which refers to the phenomenon when consumers flip their preference for one item over another when presented with a third item with certain characteristics. Furthermore, consider how fatigue and novelty bias can be incorporated into a movie or book recommendation system. While a recommendation system may have identified a user's preference for a particular genre of movies or books, novelty bias, which is a well studied phenomenon in behavioral psychology, suggests that novelty in options might actually yield uplift in system results if modeled appropriately.

\subsection{Cognitive Biases}
Next we list different types of biases that have a direct impact on decision making and that we envision will be explicitly modelled in future personalization systems.

\subsubsection{Loss, Risk and Ambiguity}
Biases related to loss, risk and ambiguity are some of the most well studied from a modeling perspective and are relevant to any eCommerce recommendation.  These biases could be incorporated by modeling the degree of certainty in experiencing satisfaction from the purchase of a product, or alternatively modeling the risk of regret associated with the purchase. In \cite{Tversky92} specific curves describing the degree of loss aversion are derived for gambling examples.  From a behavioral science perspective these biases can be described as follows: 

\begin{itemize}
    \item \textbf{Loss (Regret) Aversion}: Motivated by the tendency to avoid the possibility of experiencing regret after making a choice with a negative outcome; loss aversion refers to the asymmetry between the affinity and the aversion to the same amount of gain and loss respectively. In other words, this bias refers to the phenomenon whereby a person has a higher preference towards not losing something to winning that same thing. I.e losing \$$100$ for example results in a greater loss of satisfaction than the gain in satisfaction that is caused by winning \$$100$.
    
    \item \textbf{Risk Aversion}: This refers to the tendency to prefer a certain positive outcome over a risky one despite the chance of it being more profitable (in expectation) than the certain one.
    
     \item  \textbf{Ambiguity Aversion}: This phenomenon in decision making refers to a general preference for known risks rather than unknown ones. The difference between risk and ambiguity aversion is subtle. Ambiguity aversion refers to the aversion to not having information about the degree of risk involved. \cite{Baron94,Ellsberg61}\\
\end{itemize}    

\subsubsection{Collective Effects}
Another set of biases we believe are important to model relate to how the value of recommended items is likely to be perceived in the context of other items or recommendations.  These comparative differences are important when recommendations are considered in sequence or when collectively surfaced such as multiple ads on the same web page.  These include:
    
\begin{itemize}    
    \item \textbf{Contrast Effect}: An object is perceived differently when shown in isolation than when shown in contrast to other objects.  For example a \$$5$ object might seem inexpensive next to a \$$11$ object but expensive next to an object priced at \$$1$. \cite{Plous93}
 
    \item \textbf{Decoy Effect}: This effect is common in consumer and marketing scenarios whereby a consumer's preference between two items reverses when they are presented with a third option that is clearly inferior to one of the two items and only partially inferior to the second. The Decoy effect can be viewed as a specific instance or a special case of the contrast effect described above.
    
    \item \textbf{Distinction Bias}: This refers to the situation where two items are viewed as more distinct (from each other) when viewed together than when viewed separately.
    
    \item \textbf{Framing effect}: This refers to the effect on decision making of the manner in which something is presented to a user. I.e. the same information presented in different ways can lead to different decision outcomes. \\
\end{itemize}
    
\subsubsection{Long Term Decision Making}    
Finally, we consider biases that might arise in different parts of the decision making process over time. These biases will all play a role in optimizing future recommendation systems for long term user value.  Some examples of these include:     \\

\begin{itemize}

    \item \textbf{Choice Supportive Bias}: Also referred to as post-purchase rationalization, this bias refers to the tendency to remember past choices as justifiable and in some cases retroactively ascribe a positive sentiment to them. \cite{Mather00}

    \item \textbf{Anchoring (Confirmation) Bias}: paying more attention to information that supports one's opinions while ignoring or underscoring another set of information. This type of bias includes resorting to preconceived opinions when encountering something new. 
    
    \item \textbf{Hyperbolic Discounting Effect}: This bias refers to the tendency to prefer immediate payoffs than those that occur later in time. As an example related to our application of recommendation systems, considering a consumer's preference to receive an object sooner rather than at a later time period (for instance due to shipping time) can have an direct impact on item preference.
    
    \item \textbf{Decision Fatigue}: Although not an explicit cognitive bias, decision fatigue is a phenomenon worth exploring as it refers to the manner in which decision making deviates from the expected norm as a result of fatigue induced by long periods of decision making.
    
    \item \textbf{Selective Perception}: This refers to the tendency for expectation to bias ones perception of certain phenomena. \cite {Chandler11} \\
\end{itemize}

\subsection{The Future of Personalization}

In this section, we have argued that the next-generation of  personalization systems should be designed by explicitly modeling people's cognitive biases.  Such a redesign will require developing algorithms that have a more realistic model of how humans perceive the world and will be able to exploit human deviations from perfect rationality. 

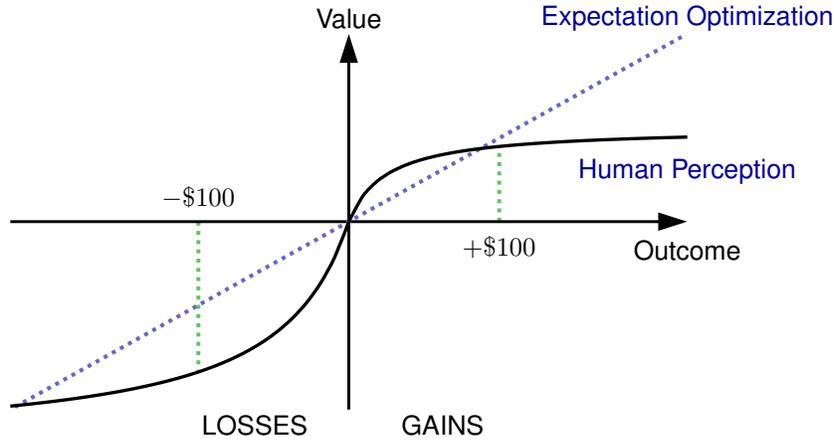
\begin{figure}[h]
\centering
\quad\begin{tikzpicture}[every node/.style={inner sep=0}]

\newcommand\myX{2.0}
\draw [line width=1.5,dotted,green!60!black!60!white] (-\myX,0) -- (-\myX, {3*(-\myX/(1+\myX))});
\draw [line width=1.5,dotted,green!60!black!60!white] (\myX,0) -- (\myX, {1.5*(\myX/(1+\myX))});
\node [anchor=south,above=5] at (-\myX,0) {$-\$100$};
\node [anchor=north,below=5] at (\myX,0) {$+ \$100$};

\draw [line width=1.1,-{Triangle[length=1.1em]}] (-4.5,0) -- (4.5,0);
\draw [line width=1.1,-{Triangle[length=1.1em]}] (0,-2.5) -- (0,2.5);

\draw [line width=1.5,dotted,blue!60!black!60!white,shorten <=2,shorten >=2] (-4.5,-2.5) -- (4.5,2.5);

\draw [line width=1.25,scale=0.5,domain=0:9,smooth,variable=\x] plot ({\x},{2.5*(\x / (1 + abs(\x)))});
\draw [line width=1.25,domain=-4.5:0,smooth,variable=\x] plot ({\x},{3*(\x / (1 + abs(\x)))});

\node [anchor=south,above=2] at (0,2.5) {\textsf{Value}};
\node [anchor=west,below=7] at (4.5,0) {\textsf{Outcome}};
\node [anchor=north east,below=2] at (-1.25,-2.5) {\textsf{LOSSES}};
\node [anchor=north west,below=2] at ( 1.25,-2.5) {\textsf{GAINS}};
\node [anchor=south,blue!60!black,above=1] at (4.5,2.5) {\textsf{Expectation Optimization}};
\node [anchor=south,blue!60!black] at (4.5,0.5) {\textsf{Human Perception}};

\end{tikzpicture}
  \caption{A piece-wise non-linear model of modelling human perception of gains and losses.  Algorithms that optimize expectation view the gain of \$100 and the loss of \$100 as equal, but humans do not.}
  \label{fig:prospect_value}
\end{figure}

Recent academic papers have shown that it is possible to model human biases by incorporating behavioral science frameworks, such as prospect theory~\cite{Kahneman79,Tversky92} into reinforcement learning algorithms \cite{Csaba16}. Traditional work in reinforcement learning is based on maximizing long-term expected utility. Prospect theory will require redesigning reinforcement learning models to reflect human nonlinear processing of probability estimates. These models incorporate the cognitive bias of loss aversion using a theory~\cite{Kahneman79,Tversky92} that models perceived loss as asymmetric to gain, as illustrated in Figure 1.  This captures a human perspective where potential gain of \$100 is less preferable to avoiding a potential loss of \$100.  Algorithms that simply maximize expectation would treat these two outcomes equally.  

We envision that future personalization algorithms will incorporate such models for a wide variety of human cognitive biases, adjusting the steepness of the curves and the value of the inflection points accordingly for each person based on their overall character, their immediate context and the history of their interactions with the system.  Unlike contextual bandits that cannot differentiate between a `visit' and a `user', the next generation of personalization systems will consider a sequence of recent events and interactions with the system when making a recommendation, and thus be able to incorporate surprise and novelty into the recommendation sequence according to the user's modelled profile.  We also argue that the cognitive bias model will be useful in deciding how best to present the recommendation to the user, for example a person with a high familiarity bias might only want to invest in a stock that she knows, and be less likely to  keep a more diverse portfolio.  A recommendation system that had an explicit model of this bias could present the diversified portfolio in a way that makes it appear more familiar, for example mentioning the user's friends who had similar diversified portfolios.  

We expect that  these next-generation personalization algorithms may initially require a high density of data, but that this dependence on data may be ameliorated as we move beyond modelling solely based on click streams, and exploiting other data sources available from sensor rich environments.  In particular we envision curated experiences such as visits to theme parks, cruise ships or hotel chains with loyalty programs.  In such environments rich data is available, activity preferences, shop purchases, facility utilization and user queries could all contribute to help train a sufficiently effective model.  
  
Designing an SR system that understands how people actually reason  has a huge potential to retain users and far more effectively market products than a system that mathematically optimizes some convenient but non-human-like  optimization function.  Understanding how people actually make decisions will help us match them with products that will truly make them happy and keep them engaged in using the system for the long term.

\section{Summary and Conclusions}
In this paper we demonstrated through various real world use-cases the advantages of strategic recommendation systems and their implementation using RL.  To make RL practical, we solved many fundamental challenges, such as evaluating policies offline with high confidence, safe deployment, non-stationarity, building systems from passive data that do not contain past recommendations, resource constraint optimization in multi-user systems, and scaling to large and dynamic actions spaces.  Finally we presented ideas for the what we believe will be the next generation of SRs that would truly understand people by incorporating human cognitive biases.

\bibliographystyle{spmpsci}      
\bibliography{biblio}   

\end{document}